%% file: paper_arxiv.tex
\documentclass[twoside]{article}

\usepackage[accepted]{aistats2025}

\usepackage{xcolor}
\definecolor{mydarkblue}{rgb}{0,0.08,0.45}
\definecolor{mydarkgreen}{rgb}{0.2,0.75,0.15}
\usepackage[colorlinks, citecolor=mydarkgreen, linkcolor=blue, urlcolor=mydarkgreen]{hyperref}

\usepackage{amsmath}
\usepackage{amssymb}
\usepackage{mathtools}
\usepackage{url}
\usepackage{algorithm}
\usepackage{algorithmic}
\usepackage{amsthm}
\usepackage[authoryear]{natbib}
\usepackage{tikz}
\usepackage{graphicx}
\usepackage{booktabs}
\usepackage{cite}
\usepackage{subcaption}
\usetikzlibrary{positioning, shapes.geometric, patterns, arrows.meta}
\definecolor{blush}{rgb}{0.87, 0.36, 0.51}

\renewcommand{\d}{\textnormal{d}}

\newcommand\norm[1]{\left\lVert#1\right\rVert}

\renewcommand{\div}{\textnormal{div}}
\newcommand{\Id}{\textnormal{Id}}

\newcommand{\mmds}[2]{\textnormal{MMD}^2 \left(#1,#2\right)}
\newcommand{\R}{\mathbb{R}}
\newcommand{\bE}{\mathbb{E}}

\newcommand{\cF}{\mathcal{F}}
\newcommand{\cG}{\mathcal{G}}
\newcommand{\cL}{\mathcal{L}}

\newcommand{\cP}{\mathcal{P}_2}
\newcommand{\fS}{\mathfrak{S}}
\newcommand{\N}{\mathbb{N}}
\newcommand{\X}{\mathbf{X}}
\newcommand{\Z}{\mathbf{Z}}
\newcommand{\Y}{\mathbf{Y}}
\newcommand{\D}{\mathrm{D}}

\newcommand{\gW}{\nabla_W}

\DeclareMathOperator*{\argmin}{arg\,min}

\newtheorem{theoremcounter}{GlobalCounter}

\newtheorem{theorem}[theoremcounter]{Theorem}
\newtheorem{appendixtheorem}{Theorem}[section]

\newtheorem{appendixlemma}[appendixtheorem]{Lemma}

\newtheorem{corollary}[theoremcounter]{Corollary}

\newtheorem{proposition}[theoremcounter]{Proposition}
\newtheorem{appendixproposition}[appendixtheorem]{Proposition}

\newtheorem{appendixassumption}[appendixtheorem]{Assumption}

\theoremstyle{definition}

\newtheorem{definition}[theoremcounter]{Definition}
\newtheorem{appendixdefinition}[appendixtheorem]{Definition}

\newtheorem{appendixremark}[appendixtheorem]{Remark}

\usepackage[textwidth=0cm]{todonotes}

\definecolor{blush}{rgb}{0.87, 0.36, 0.51}

\begin{document}

%
\runningtitle{DDEQs: Distributional Deep Equilibrium Models through Wasserstein Gradient Flows}

%

\twocolumn[

\aistatstitle{DDEQs: Distributional Deep Equilibrium Models \\through Wasserstein Gradient Flows}

\aistatsauthor{ Jonathan Geuter \And Cl\'ement Bonet \And  Anna Korba \And David Alvarez-Melis }

\aistatsaddress{ Harvard University\And  CREST, ENSAE, IP Paris \And CREST, ENSAE, IP Paris \And Harvard University } ]

\begin{abstract}
    Deep Equilibrium Models (DEQs) are a class of implicit neural networks that solve for a fixed point of a neural network in their forward pass. Traditionally, DEQs take sequences as inputs, but have since been applied to a variety of data. In this work, we present Distributional Deep Equilibrium Models (DDEQs), extending DEQs to discrete measure inputs, such as sets or point clouds. We provide a theoretically grounded framework for DDEQs. Leveraging Wasserstein gradient flows, we show how the forward pass of the DEQ can be adapted to find fixed points of discrete measures under permutation-invariance, and derive adequate network architectures for DDEQs. In experiments, we show that they can compete with state-of-the-art models in tasks such as point cloud classification and point cloud completion, while being significantly more parameter-efficient.
\end{abstract}

\section{INTRODUCTION}

\looseness=-1 Implicit neural networks such as Deep Equilibrium Models (DEQs) \citep{deq} or Neural ODEs \citep{neuralodes} have recently emerged as a promising class of networks which, instead of explicitly computing outputs of a neural network, find equilibria of certain dynamics implicitly. They simulate ``infinitely deep'' networks with adaptive depth, while maintaining a constant memory footprint thanks to their implicit nature.

DEQs solve for the equilibrium point of a model in their forward pass, and in the backward pass, Jacobian-based linear fixed-point equations \citep{deq} or \textit{Phantom Gradients} \citep{phantomgradient} are used to approximate the gradient, without the need to backpropagate through a large number of stacked layers, which would be prohibitively expensive. Originally, DEQs were introduced for sequence-to-sequence tasks such as language modelling \citep{deq}, since the implicit nature of the network mandates the outputs to be of the same shape as the inputs. They have since been applied to tasks such as computer vision \citep{mdeqs} leveraging multi-scale network outputs, and exhibit competitive performance with state-of-the-art explicit models.

\input{tikz/fig_intro}

\looseness=-1 In this paper, we apply DEQs to data that comes in the shape of discrete probability measures seen as a set of particles.
Various types of data, such as sets, graphs, or point clouds, can be viewed through this distributional lens. It arises in areas such as autonomous driving, robotics, or augmented reality \citep{objectdetectionautonomous, robotics}, and often stems from sensors scanning their environments with discrete measurements. Various architectures for distributional data have been proposed, ranging from deep networks with simple feature transformations \citep{deepsets, qi2017pointnet, pointnet++} to attention based models \citep{settransformer, pointtransformer}. In this work, we present DDEQs, short for Distributional Deep Equilibrium Models, which process distributional data and feature very flexible architectures that can be applied to a variety of tasks. The main challenge in using DEQs on measure inputs lies in the fixed point solver: traditional solvers such as fixed point iteration, Newton's method, or Anderson acceleration \citep{deq} cannot produce fixed points under permutation invariance of the particles. 
We propose to minimize a discrepancy between probability distributions to find fixed points in the forward pass. To this end, we rely on Wasserstein gradient flows (WGFs) \citep{ambrosio2008gradient, santambrogio2017euclidean}, which are gradient flows of a functional in the space of probability measures.

 While many permutation-equivariant measure-to-measure architectures exist \citep{deepsets, qi2017pointnet, settransformer}, DDEQs are, to our knowledge, the first class of measure-to-measure neural networks whose forward pass is agnostic to permutations, which is an important step towards properly treating inputs as measures.

\looseness=-1 We provide a theoretically grounded framework for DDEQs, proving that they enjoy the same backward pass as traditional DEQs, deriving the Wasserstein gradient of the forward pass, and proving a convergence result for the average gradient of the optimization loss. In experiments on point cloud classification and point cloud completion, we show that DDEQs can compete with state-of-the-art architectures, and believe they can open the door to a new class of neural networks for point cloud processing.

\section{RELATED WORK}

\looseness=-1 \textbf{Deep Equilibrium Models.} 
Since their introduction \citep{deq}, DEQs have been extended in several key directions. To improve representational capabilities, multiscale DEQs \citep{mdeqs} introduce hidden states at different resolutions. Jacobian regularization \citep{bai2021stabilizingequilibriummodelsjacobian} and Lyapunov-stable DEQs were shown to stabilize training. One of the main drawbacks of DEQs, their slow training, is addressed by \textit{Phantom Gradients} \citep{phantomgradient}, an approximation of the true gradient direction, which we also utilize in training DDEQs. Existence guarantees for fixed points in DEQs, as well as convergence guarantees, are difficult to obtain in theory however, and typically involve imposing restrictions on the model weights or activations \citep{winston2021monotoneoperatorequilibriumnetworks,ghaoui2020implicitdeeplearning, gabor2024positiveconcavedeepequilibrium, ling2024deepequilibriummodelsequivalent}. These works also study the uniqueness of the DEQ fixed point under the assumption that the model takes the form $f(z,x)=\phi(Wz+Ux+b)$ for weights $W$, $U$, and $b$. This setting only covers single-layer affine networks with fixed-size inputs; in particular, it does not include attention-based networks like transformers as we are using in this work. This goes to show that studying questions of uniqueness and existence of fixed points for DEQs is challenging, but establishing results for attention-based networks is beyond the scope of this work.

\textbf{Bilevel Optimization.} DEQ training can be seen as a particular instance of a \textit{bilevel optimization} problem  \citep{bileveloptimization}, where the goal is to minimize an outer objective
\begin{equation*}
    \mathcal{L}(\theta)=f(\theta,y^*_\theta)
\end{equation*}
over parameters $\theta\in\R^w$, with the inner problem
\begin{equation*}
    y^*_\theta=\argmin_{y\in\R^p}\ g(\theta,y).
\end{equation*}
Typically, $\theta$ are the parameters of a neural network, and the inner problem corresponds to an empirical loss. The gradient of the outer problem w.r.t. the parameters $\theta$ can be computed implicitly with the inverse function theorem \citep{deq,ye2024enhancinghypergradientsestimationstudy}, whereas the inner loop is often solved iteratively with root solving algorithms \citep{deq}. However, errors in the inner loop can propagate to the outer loop, leading to inexact gradients \citep{ye2024enhancinghypergradientsestimationstudy}. Various approaches to reduce the inner loop error have been proposed in the literature, such as sparse fixed point corrections \citep{bai2022deepequilibriumopticalflow}, warm starting \citep{bai2022deepequilibriumopticalflow, geuter2025universalneuraloptimaltransport,thornton2023rethinkinginitializationsinkhornalgorithm, amos2023metaoptimaltransport}, and preconditioning or reparametrizing the inner problem \citep{ye2024enhancinghypergradientsestimationstudy}. Other studies have investigated the role of the number of iterations for the inner loop at train versus inference time \citep{ramzi2023testliketrainimplicit}.
Recently, it was also proposed to jointly solve the inner and outer problem in a single loop for some bilevel optimization problems \citep{dagreou2022framework,implicit}. We leave studying the effect of these strategies on DDEQs for future research.

\looseness=-1 \textbf{Point Cloud Networks.}
Neural architectures acting on point cloud inputs have seen significant advancements in recent years, and been applied to a multitude of tasks. For classification, PointNet \citep{qi2017pointnet} was groundbreaking in directly processing unordered point clouds using shared MLPs and a max-pooling operation, while PointNet++ \citep{pointnet++} improved upon this by capturing local structures using hierarchical feature learning. Set Transformers \citep{settransformer} introduced attention mechanisms to improve permutation invariance, and Point Transformers \citep{pointtransformer} further leveraged attention for capturing both local and global features, achieving state-of-the-art performance. Notably, the attention layer in Point Transformers uses vector attention instead of scalar attention (i.e. where attention scores are vector-valued), and features positional encodings which learn distances between input particles. While our architecture also builds on attention networks (see Section \ref{sec:architecture}), we use scalar attention without positional encodings, as we did not find the attention network from Point Transformers to work particularly well for DDEQs.

For point cloud completion \citep{pointcloudcompletion}, several approaches have emerged. AtlasNet \citep{groueix2018atlasnetpapiermacheapproachlearning} introduced a method that learns to map 2D parametrizations to 3D surfaces to reconstruct complete shapes, while PCN \citep{yuan2019pcnpointcompletionnetwork} features a coarse-to-fine reconstruction strategy. PMP-Net \citep{wen2021pmpnetpointcloudcompletion} learns paths between particles in the input and the target, and moves input particles along these paths. However, none of these methods exactly preserve the input point cloud as part of the output. Instead, they predict a new point set that approximates the complete shape. 
Preserving inputs can be important in applications where exact input fidelity is required. We show that DDEQs are architecturally well-suited to predict point cloud completions that exactly preserve inputs.

\looseness=-1 \textbf{Wasserstein Gradient Flows.} WGFs offer an elegant way to minimize functionals with respect to probability distributions, making them an increasingly popular choice for various applications ranging  from generative modeling \citep{fan2022variational, choi2024scalable} and sampling \citep{wibisono2018sampling,lambert2022variational,huix2024theoretical} to reinforcement learning \citep{ziesche2024wasserstein} and optimizing machine learning algorithms such as gradient boosting \citep{matsubara2024wasserstein}. With suitable discretization of the flows, various objectives for minimization have been studied, including the Kullback-Leibler divergence \citep{wibisono2018sampling}, the Maximum Mean Discrepancy (MMD) \citep{mmdflow, altekruger2023neural, rieszkernels} and their variants \citep{glaser2021kale,chen2024regularized, neumayer2024wasserstein, chazal2024statistical}, Sliced-Wasserstein distances \citep{liutkus2019sliced, chao2023nonparametric, bonet2024sliced} or Sinkhorn divergences \citep{carlier2024displacement,zhu2024neural}.

\section{DDEQs}
\textbf{Notation.} Denote by $\cP(\mathbb{R}^d)$ the space of probability measures on $\mathbb{R}^d$ with finite second moments; for $\mu\in\cP(\mathbb{R}^p)$ and a function $F:\mathbb{R}^p\to\mathbb{R}^p$, denote by $F_\#\mu\in\cP(\mathbb{R}^p)$ the push-forward measure of $\mu$ by $F$, and for $\cG:\cP(\R^p)\to\R$ denote by $\nabla_W\cG(\mu):\R^d\to\R^p$ the Wasserstein gradient of $\cG$ at $\mu$, if it exists. $\frac{\delta\cG}{\delta\mu}:\R^d\to\R$ denotes the first variation of $\cG$. For any $\mu\in\cP(\R^d)$, we denote by $L^2(\mu)$ the space of functions $f:\R^d\to \R^d$ such that $\|f\|_{L^2(\mu)}^2 := \int \|f\|^2\mathrm{d}\mu <\infty$ and by $\mathrm{Id}\in L^2(\mu)$ the identity map.
Background on measure theory, optimal transport, and Wasserstein gradient flows, as well as all missing proofs, can be found in Appendix~\ref{sec:proofs}.

\subsection{Deep Equilibrium Models}

Given an input data $x\in \R^d$, a DEQ initializes a latent variable $z$ and aims to find  a fixed point $z^*\in \R^p$ of the function $F_\theta(\cdot,x)$, where $F_\theta$ is a neural network parameterized by weights $\theta\in\mathbb{R}^w$. 
Adding the latent variable $z$ instead of directly operating on $x$ has several advantages, such as parametrizing a unique fixed point function $F_\theta(\cdot,x)$ for each input $x$ instead of using the same function $F_\theta(\cdot)$ for all inputs, which improves the richness of the fixed points, and the dimension of the latent $z$ can be chosen arbitrarily, increasing the flexibility and complexity of the model.
One way to find fixed points would be via function iterations:
\begin{equation*}
    z^{l+1}=F_\theta(z^l,x),\quad l=0,1,...,L-1,
\end{equation*}
for a large $L$. This sequence could be backpropagated through to find gradient directions; however, this would be prohibitively expensive. Hence, root solvers such as Newton's or Anderson's method are used to find roots of the function $g_\theta(z)=F_\theta(z,x)-z$.
Then, the gradient for the backward pass can be computed through the inverse function theorem \citep{blondel2024elements}. Suppose the fixed point $z^*$ is passed through a non-parametric function $h$, and a loss $l$ between $h(z^*)$ and the target $y$ is computed as $l(h(z^*),y)$. Then the following formula holds \citep{deq}:

\begin{equation}\label{eq:deq-backprop}
        \frac{\d \ell}{\d \theta}=\frac{\partial \ell(h(z_\theta^*))}{\partial h}\frac{\partial h(z_\theta^*)}{\partial z_\theta^*}\left(I-\frac{\partial F_\theta(z_\theta^*,x)}{\partial z_\theta^*}\right)^{-1}\frac{\partial F_\theta(z_\theta^*,x)}{\partial \theta}.
\end{equation}
Importantly, this is \textit{independent} of how the fixed point was attained, hence any black-box root solver can be used. However, due to the inverse Jacobian, (\ref{eq:deq-backprop}) is costly to compute precisely, but efficient approximation methods have been developed \citep{deq, phantomgradient}.

\subsection{DDEQs as Bilevel Optimization}

We will now turn to distributional inputs and latent variables.   Let $\rho\in\cP(\mathbb{R}^d)$ be the input, $\mu\in{\cP}(\mathbb{R}^p)$ be the latent variable, and $F_\theta:{\cP}(\mathbb{R}^p)\times{\cP}(\mathbb{R}^d)\to{\cP}(\mathbb{R}^p)$ be a function parameterized by some $\theta\in\mathbb{R}^w$. Let $\mathcal{Y}$ be the target space (which can be Euclidean or itself be a probability measure space), i.e., our dataset consists of samples $(\rho,y)\in\cP(\mathbb{R}^d)\times\mathcal{Y}$. Depending on the task, we might further have to process a fixed point $\mu^*_{\theta,\rho}$ of $F_\theta(\cdot,\rho)$, hence we let $h_\theta: \cP(\mathbb{R}^p)\to\mathcal{Y}$, which maps latent fixed points $\mu^*_{\theta,\rho}$ to predictions such as class labels, also be parametrized by $\theta$. Furthermore, let $\ell:\mathcal{Y}\times\mathcal{Y}\to\mathbb{R}$ be a loss function and $\mathrm{D}:\cP(\mathbb{R}^p)\times\cP(\mathbb{R}^p)\to\mathbb{R}$ be a distance or divergence between probability distributions. We can write, for any input data $\rho \in \cP(\R^d)$ and $\theta\in \R^w$,
\begin{equation}\label{eq:inner_obj}
    \mu^*_{\theta,\rho} = \argmin_{\mu \in \cP(\R^p)}\ \mathrm{D}\big(\mu,F_{\theta}(\mu,\rho)\big):=\cG_{\theta,\rho}(\mu),
\end{equation}
since the optimizers of the functional $\cG_{\theta,\rho}$ are precisely the fixed points of $F_\theta(\cdot,\rho)$. Note that existence of fixed points in DEQs is difficult to prove, and typically involves imposing restrictions on the model weights \citep{winston2021monotoneoperatorequilibriumnetworks, gabor2024positiveconcavedeepequilibrium}. We leave deriving such guarantees for DDEQs for future work.

A natural optimization scheme for minimizing the loss
\begin{equation}\label{eq:outerloop}
    \mathcal{L}(\theta) := \mathbb{E}_{\rho,y\sim P_{\text{data}}}\left[\ell\big(h_\theta (\mu^*_{\theta,\rho}),y\big)\right]
\end{equation}
with respect to $\theta\in \R^w$ is gradient descent, i.e. for all $k\ge 0$, $\theta^{k+1} = \theta^k - \gamma \nabla \mathcal{L}(\theta^k)$.
However, since
\begin{equation*}
    \nabla \cL(\theta) =\bE_{\rho, y\sim P_{\text{data}}}\left[ \nabla_\theta \ell\big(h_\theta(\mu^*_{\theta,\rho}),y\big)\right], 
\end{equation*}
we need to solve the following two problems: \emph{i)} find a fixed point $\mu^*_{\theta,\rho}$, given $\theta$ and $\rho$ (\textit{inner loop}) and \emph{ii)} compute $\nabla_\theta \ell\big(h_\theta(\mu^*_{\theta,\rho}\big), y)$, given $\mu^*_{\theta,\rho}$ and $y$. 
This will enable us to solve (\ref{eq:outerloop}) (\textit{outer loop}).
Hence, minimizing $\cL$ is a typical instance of a bilevel optimization problem. 
In section \ref{sec:inner}, we will see how to solve the inner problem \eqref{eq:inner_obj}, and in section \ref{sec:outer}, we will investigate computing the gradient $\nabla_\theta \ell(h_\theta(\mu^*_{\theta,\rho}), y)$ and solving the outer problem.

\subsection{Inner Optimization}\label{sec:inner}

In this section, we focus on the inner problem (\ref{eq:inner_obj}). Let $\cG_{\theta, \rho}(\mu) :=\mathrm{D}\big(\mu, F_{\theta}(\mu,\rho)\big)$ for an input $\rho \in \cP(\R^d)$ and parameters $\theta\in\R^w$ (we will sometimes drop the dependency on $\theta$ and $\rho$ for ease of reading).
In the following, we assume for simplicity that we can write $F_\theta(\mu,\rho)={F_\theta}_\#\mu$, i.e., as some push-forward (dropping the dependency on a fixed $\rho$ for ease of notation)\footnote{This is also a slight abuse of notation, as $F_\theta$ in ${F_\theta}_\#\mu$ is now a map from $\R^p$ to $\R^p$.}. Note that this is not possible in general, as the push-forward itself might depend on $\mu$ (i.e., $F_\theta(\mu,\rho)={F_\theta}(\mu)_\#\mu$, as used in \citep{furuya2024transformers,castin2024smooth}). However, it is a reasonable assumption in our setting (for more details, see Appendix~\ref{sec:pushforward}).
Since (\ref{eq:inner_obj}) is an objective over probability measures, the natural dynamic to minimize it is following its Wasserstein gradient flow $t\mapsto \mu_t$, which solves the following continuity equation:
\begin{align}\label{eq:continuity}
    \partial_t\mu_t=&\div\big(\mu_t\nabla_W\cG_{\theta,\rho}(\mu_t)\big),
\end{align}
where $\nabla_W \cG_{\theta,\rho}(\mu)$ denotes the Wasserstein gradient of $\cG_{\theta,\rho}$ at $\mu$. Its existence, of course, depends on the nature of $\cG$, i.e., the choice of $\mathrm{D}$. In practice, we use $\mathrm{D}=\frac{1}{2}\text{MMD}^2$, the squared maximum mean discrepancy \citep{gretton2012kernel}, which is defined as 
\begin{equation*}
    \mmds{\mu}{\nu} = \iint k(x,y)\ \mathrm{d}(\mu-\nu)(x) \mathrm{d}(\mu-\nu)(y),
\end{equation*}
for a symmetric positive definite kernel $k:\R^p\times\R^p\to \R$. Hence, our inner optimization loss takes the form
\begin{equation*}
    \cG_{\theta, \rho}(\mu)=\frac{1}{2}\mathrm{MMD}^2\big(\mu, F_\theta(\mu,\rho)\big).
\end{equation*}
The MMD is comparably fast to compute, with a time complexity of $\mathcal{O}(n^2)$ where $n$ is the number of particles (when $\mu,\nu$ are discrete and both supported on $n$ particles), whereas e.g. the time complexity of the Sinkhorn algorithm \citep{cuturi2013sinkhorndistanceslightspeedcomputation} is $\mathcal{O}(n^2\log(n)/\epsilon^2)$ \citep{sinkhornconvergence}, where $\epsilon$ is the regularization parameter and typically fairly small.
Furthermore, the MMD is well defined between empirical distributions with different support (unlike, for example, the KL Divergence), and a Wasserstein gradient of the MMD to a fixed target measure exists and can be evaluated in closed form in quadratic time \citep{mmdflow} for smooth kernels.
We will now show that this result also extends to the setting where the target is a push-forward of the source measure.
To this end, we first derive the Wasserstein gradient of a functional of the form $\mu\mapsto \cF(T_\#\mu)$ for a $\mu$-almost everywhere (a.e.) differentiable push-forward operator $T:\R^p\to \R^p$.
\begin{proposition}
    Let $\mu\in\cP(\R^p)$, $\cF:\cP(\R^p)\to\R$, $T:\R^p\to\R^p\in L^2(\mu)$ a $\mu$-a.e. differentiable map and define $\tilde{\cF}(\mu):=\cF(T_\#\mu)$. Assume $\sup_x\ \|\nabla T(x)\|_{\mathrm{op}} < +\infty$.
    If the Wasserstein gradient of $\cF$ at $T_\#\mu$ exists, then the Wasserstein gradient of $\tilde{\cF}$ at $\mu$ also exists, and it holds:
    \begin{equation*}
        \gW\tilde{\cF}(\mu) = \nabla T\big(\gW\cF(T_\#\mu)\circ T\big) = \nabla \left(\frac{\delta\cF}{\delta\mu}(T_\#\mu)\circ T\right).
    \end{equation*}
\end{proposition}
This allows us to derive the Wasserstein gradient for our MMD objective from the chain rule.
\begin{corollary}
    Let $\cG(\mu):=\frac{1}{2}\mmds{\mu}{T_\#\mu}$. Define the witness function $f_\mu$ as
    \begin{equation*}
        f_\mu:=\int k(\cdot,y)\d\mu(y)-\int k(\cdot,y)\d (T_\#\mu)(y),
    \end{equation*}
    where $k:\R^p\times\R^p\to\R$ is the kernel of the MMD, which we assume smooth. Then the Wasserstein gradient of $\cG$ at $\mu$ exists, and is equal to
    \begin{equation}\label{eq:wassersteingradient}
        \nabla_W\cG(\mu) = \nabla f_\mu-\nabla(f_\mu\circ T).
    \end{equation}
\end{corollary}

Note that without convexity assumptions, we cannot hope to find global optimizers of (\ref{eq:inner_obj}), as
with most optimization problems. Yet, following (minus) the Wasserstein gradient (\ref{eq:wassersteingradient})
lets us find a local minimizer.
In practice, we approximate \eqref{eq:continuity} by discretizing it using the forward Euler scheme, which is commonly referred to as the Wasserstein gradient descent \citep{bonet2024mirror}, for a step-size $\eta>0$:
\begin{equation*}
    \mu_{k+1} = \big(\mathrm{Id} - \eta \gW \cG_{\theta, \rho}(\mu_k)\big)_\#\mu_k,
\end{equation*} 
for all $k\ge 0$, where $\mathrm{Id}$ is the identity map on $L^2(\mu_k)$.

\subsection{Outer Optimization}\label{sec:outer}
In this section, we derive a formula to compute the derivative of the loss $\ell:=\ell\big(h_\theta (\mu^*_{\theta,\rho}),y\big)$ w.r.t. $\theta$, hence of the outer loss $\mathcal{L}$ given in \eqref{eq:outerloop}.

\begin{theorem}\label{thm:gradient}\textbf{(Informal.)}
    Let $\mathcal{Y}=\R^n$ or $\mathcal{Y}=\cP(\R^y)$, $\rho\in\cP(\mathbb{R}^d)$ and $y\in\mathcal{Y}$ be fixed, and $\mu^*_\theta\in\cP(\mathbb{R}^p)$ be a fixed point of $F(\cdot,\theta):=F_\theta(\cdot,\rho)$. 
    Assume that $\left(\Id-D_{\mu^*_{\theta,\rho}}F(\mu^*_{\theta,\rho},\theta)\right)^{-1} \circ D_\theta F(\mu^*_{\theta,\rho},\theta):\R^p\to\R^w$ exists, where $\Id$ is the identity map in $\R^p$. Then under suitable differentiability assumptions it holds:
    \begin{multline*}
        \frac{\d \ell}{\d \theta} = D_h \ell(h(\mu^*_{\theta,\rho},\theta),y)\circ\Big[D_\theta h (\mu^*_{\theta,\rho},\theta) + \\ D_{\mu^*_{\theta,\rho}}h(\mu^*_{\theta,\rho},\theta)\circ\left(\Id-D_{\mu^*_{\theta,\rho}}F(\mu^*_{\theta,\rho},\theta)\right)^{-1} \circ D_\theta F(\mu^*_{\theta,\rho},\theta)    \Big].
    \end{multline*}
\end{theorem}
\looseness=-1 The formal version and proof of Theorem \ref{thm:gradient} is deferred to Appendix \ref{sec:proof_th_implicit_grad} along with the definitions of the differential of $F$ relative to $\mu$, which we define following \citep{lessel2020differentiablemapswassersteinspaces}. Similar to Theorem 1 in \citep{deq}, Theorem \ref{thm:gradient} shows that the gradient of $\ell$ w.r.t. $\theta$ can be computed implicitly, without backpropagating through the inner loop. But as is commonly done, we do not use this formula directly in the backward pass, but instead use the phantom gradient and autodifferentiation.

By analyzing the following continuous bilevel problem, similarly to \citep{implicit},
\begin{equation*}
    \begin{cases}
        \forall \rho \in \cP(\R^d),\ \partial \mu_{t,\rho} = \div\big(\mu_{t,\rho} \gW \cG_{\theta_t,\rho}(\mu_{t,\rho})\big) \\
        \d\theta_t = -\varepsilon_t \nabla\mathcal{L}(\theta_t)\d t,
    \end{cases}    
\end{equation*}
where $\epsilon_t>0$ corresponds to the ratio of learning rates between
the inner and the outer problems, 
we prove a convergence result for the average of the outer optimization gradients in $\mathcal{O}(\log(T)^2/\sqrt{T})$ in Appendix \ref{appendix:loss_cv}, which holds under suitable regularity assumptions on $\mathrm{D}$ and $F_\theta$ and a Polyak-Łojasiewicz inequality. However, we note that these might not always hold for the specific choice of $\mathrm{D}$ (i.e., an MMD) and $F_\theta$ in our case; we refer the reader to the appendix for a more detailed discussion.

\subsection{Training Algorithm}

Our training procedure is described in Algorithm \ref{alg:setdeq}. It takes as input the number of iterations $K\in\N$ for the outer loop, the number of iterations $L\in\N$ of the inner loop, a (data) batch size $B\in\N$, inner loop learning rates $\eta_l$ for $l\in [L]$, and outer loop learning rates $\gamma_k$ for $k\in [K]$.
In practice, our latents $\mu$ will be discrete measures and have a finite number of particles, where the number of particles $J_\mu$ is itself a hyperparameter\footnote{For classification, this could be a fixed number, while for point cloud completion, it could be an estimate of the number of particles in the complete point cloud. More details can be found in section \ref{sec:exps}.}. We denote these particles by $\mu_{\theta^k,i}(j)$ for $j=1,\dots,J_{\mu_{\theta^k,i}}$, where $\mu_{\theta^k,i}$ is the latent corresponding to an input sample $\rho_i$.
The initialization of $\mu_{\theta^k,i}$ as well as the dimension $p$ of particles in the latents $\mu_{\theta^k,i}$ are chosen depending on the task, see Section \ref{sec:exps}.
To speed up computations, we pad all samples in the batch with zeros (to align the number of particles) and compute the inner loop over $l$ for all $B$ inputs in parallel by leveraging adequate masking.
We use the Riesz kernel
 $   k(x,y) = -\norm{x-y}$
for the MMD, as it has been shown to have better convergence than Gaussian kernels \citep{rieszkernels, hagemann2024posteriorsamplingbasedgradient}. Note this is not a positive definite kernel (neither a smooth one), but the resulting distance (the energy distance)  is equivalent to an MMD \citep{sejdinovic2013equivalence}.
Further details on the implementation can be found in Appendix~\ref{sec:training_details}.

\begin{algorithm}[h]
\caption{DDEQ Training Procedure}
\label{alg:setdeq}
\begin{algorithmic}[1]
    \STATE Initialize $\theta^0\in\R^w$
    \STATE Input $ K,L,B\in\mathbb{N}$, $\eta^l>0$, $l=0,...,L$; $\gamma_k>0$, $k=0, ..., K-1$; $p\in\N$
    \FOR{$k\in\{0, ..., K-1\}$}
        \STATE sample $(\rho_1,y_1),...,(\rho_B,y_B)\sim P_{\text{data}}$
        \STATE Init $\mu_{\theta^k,i}$, $i=1,...,B\quad$ (e.g. $\mu_{\theta^k,i}(j)\sim\mathcal{N}(0,I_p)$ for $j=1,...,J_{\mu_{\theta^k,i}}$)
        \FOR{$i=1$,...,$B$ (in parallel)}
            \FOR{$l=0$,...,$L-1$}
            \STATE $
                \mu_{\theta^k,i}(j)\leftarrow \mu_{\theta^k,i}(j) \newline \text{\qquad \qquad \quad } - \eta_l \nabla_W\cG_{\theta^k,\rho_i}(\mu_{\theta^k,i})(\mu_{\theta^k,i}(j)) 
            \newline $ for all particles $\mu_{\theta^k,i}(j)$, 
            \ENDFOR
        \ENDFOR
        \STATE $\theta^{k+1}\leftarrow \theta^k-\gamma_k \frac{1}{B}\sum_{i=1}^B\left[\nabla_\theta \ell(h_\theta(\mu_{\theta^k,i}),y_i)\right]$
    \ENDFOR
\end{algorithmic}
\end{algorithm}

\subsection{Architecture}\label{sec:architecture}

\input{tikz/ddeq_fig}

In this section, we will describe our architecture and show that it fulfills certain criteria desirable for point cloud processing. We want to emphasize, however, that DDEQs are largely architecture-independent, as long as the architecture is suitable for processing point clouds, and more carefully designed architectures could improve their performance. We leave the design of such architectures for future research.

Recall that $F=F_\theta:
\cP(\mathbb{R}^p)\times \cP(\mathbb{R}^d)\to\cP(\mathbb{R}^p)$, $
(\mu,\rho)
\mapsto
F(\mu,\rho)$ (in this section, we remove dependency on $\theta$ for simplicity). Computationally, we will represent empirical measures as matrices by stacking the particles along the first dimension. That is, we represent $\rho = \sum_{i=1}^M \frac{1}{M}\delta_{\mathbf{x}_i} \in {\cP}(\mathbb{R}^d)$ as a matrix $\mathbf{X} \in \mathbb{R}^{M\times d}$, and $\mu = \sum_{i=1}^N \frac{1}{N}\delta_{\mathbf{z}_i} \in {\cP}(\mathbb{R}^p)$ as a matrix $\mathbf{Z} \in \mathbb{R}^{N\times p}$. Note that $M$ depends on the data sample $\rho$ and $N$ depends on the latent $\mu$ here; in particular, the number of rows in the matrices varies across samples. However, all architectures we discuss in this section allow for variable numbers of particles in their inputs.

We will now derive a property which we will call the \textit{EI Property} that induces a useful inductive bias to the network in our setting.
A common approach to design neural networks that act on matrices encoding data without inherent ordering (e.g., sets or empirical distributions), and that make predictions such as classifications, is to make them \textit{invariant} under row permutations. This corresponds to the fact that the prediction should be independent under the \textit{ordering} of the points in the point cloud. Similarly, in the setting where the output is another empirical measure of the same size, making the network \textit{equivariant} under permutations of the rows of the input (i.e., if the input is permuted in a certain way, the output is permuted in the same way) is a common approach. Unlike in most well-known architectures acting on set inputs, such as Deep Sets \citep{deepsets}, PointNet \citep{qi2017pointnet}, or Set Transformer \citep{settransformer}, we are now dealing with \textit{two} variables $\X$ and $\Z$ instead of one, which makes the analysis a bit more delicate. Since our network outputs measures that have the same number of particles as the latent $\Z$, it is natural to consider networks $F:\mathbb{R}^{N\times p}\times\mathbb{R}^{M\times d}\to\mathbb{R}^{N\times p}$, $(\Z,\X)\mapsto F(\Z,\X)$ (abusing notation a bit, as we now let $F$ act on matrices) which are row-\textit{equivariant} under $\Z$ and row-\textit{invariant} under $\X$. While row-invariance in $\Z$ would be closer to treating $\Z$ as a set, this is difficult to accomplish with standard network architectures. However, we get a type of row-invariance in $\Z$ ``for free'' by using a Wasserstein gradient flow forward solver which finds fixed points up to permutations, as opposed to classical forward solvers.

\begin{definition}[EI Property]
    Let $F:\mathbb{R}^{N\times p}\times\mathbb{R}^{M\times d}\to\mathbb{R}^{N\times p}$, $(\Z,\X)\mapsto F(\Z,\X)$. The function $F$ is said to be permutation-\textbf{E}quivariant in $\Z$ and permutation-\textbf{I}nvariant in $\X$, if for any permutations $\sigma_N\in \fS_N$ and $\sigma_M\in \fS_M$ of the rows of $\Z$ and $\X$, resp., it holds that
    \begin{equation*}
        F\big(\sigma_N(\Z),\sigma_M(\X)\big)=\sigma_N\big(F(\Z,\X)\big).
    \end{equation*}
    In that case, we say that $F$ fulfills the \textit{\textbf{EI} Property}.
\end{definition}

By ensuring that our architecture fulfills the EI property, we induce an inductive bias particularly well-suited for point cloud data. In the following, we show that a wide range of network building blocks fulfills this property.

\begin{proposition}\label{prop:bilinear}
    Let $F^{\textnormal{bil}}:\mathbb{R}^{N\times p}\times\mathbb{R}^{M\times d}\to\mathbb{R}^{N\times p}$, $(\Z,\X)\mapsto F^{\textnormal{bil}}(\Z,\X)$. Then $F^{\textnormal{bil}}$ is a bilinear map that fulfills the EI property
    if and only if $F^{\textnormal{bil}}$ is of the form
    \begin{equation*}
        F^{\textnormal{bil}}(\Z,\X)=\Z\alpha \X^\top 1_M+1_N 1_N^\top \Z\beta \X^\top 1_M
    \end{equation*}
    for some tensors $\alpha,\beta\in\mathbb{R}^{p\times p\times d}$, where $1_N,1_M$ denote the vectors with ones in $\mathbb{R}^N$ and $\mathbb{R}^M$.
\end{proposition}

EI also extends to self- and cross-attention, as well as layers that process each particle separately, such as Layer Normalization \citep{ba2016layernormalization}.

\begin{proposition}
    Denote by MultiHead(tgt, src) multi-head attention \citep{transformer} between a target and source sequence. Then MultiHead$(\Z, \Z)$ and MultiHead$(\Z, \X)$ both fulfill the EI property.
\end{proposition}

We can now define our network architecture for $F_\theta$, which also fulfills the EI property, being a composition of functions that do. The core building blocks are two self-encoder networks which compute the self-attention on $\X$ and $\Z$, resp., which are followed by a cross-encoder network computing the cross-attention between the target $\Z$ and the source $\X$. All encoders are standard multi-head attention encoders (see Appendix \ref{sec:appendixarchitecture}). An overview of the complete network architecture can be seen in Figure \ref{fig:arch}.
Here, feedforward networks are two-layer linear networks with one ReLU activation, which are applied on a per-particle basis. Note how modules in the architecture have different input and output heights, which correspond to the particle dimension. The network features the dimension of $\X$ (typically 2 or 3), the encoder dimension which is the same as the dimension of $\Z$, and the dimension of the bilinear layer which is typically much smaller than the encoder dimension.
Details on the implementation can be found in section \ref{sec:implementation}. Depending on the task, we pre- and post-process the in- and outputs to the DDEQ in different ways.

\input{tikz/fig_classif_ddeq}

\textbf{Classification Pipeline.} For classification, we initialize $\Z\sim \mathcal{N}(0,I)$. We set $h_\theta(\Z^*)=\text{Lin}_{h_\theta}(\text{MaxPool}(\Z^*))$, where $\text{Lin}_{h_\theta}$ denotes a single linear layer, i.e. we first apply max-pooling along the particle dimension and then a single linear layer mapping from the latent dimension to the number of classes. The loss is the cross entropy between the predicted and target labels. An overview is found in Figure \ref{fig:pipeline-classification}.

\textbf{Completion Pipeline.} In point cloud completion \citep{pointcloudcompletion}, the task is to predict a target point cloud $\Y$ given a partial input point cloud $\X$. Hence, we could initialize $\Z$ to be equal to $\X$, add a certain number of ``free'' particles, and then let the network learn to move the free particles to the right locations, while keeping the particles that were initialized at $\X$ fixed. However, this would mean that particles in $\Z$ are of the same dimension as $\X$. Instead we allow for a higher hidden dimension of the particles to increase expressiveness of the network. We first add a number of free particles\footnote{as the target is not known, the number of free particles is empirically chosen such that the total number of particles in $\Z$ matches the number of particles in the target on average; for more details, refer to section \ref{sec:exps}.} to $\X$ and call this $\Tilde{\X}$. Then we pad $\Tilde{\X}$ with zeros to match the latent dimension; call this $\tilde{\Z}$. We then feed $\tilde{\Z}$ through an invertible neural network $q$ which takes the following form:
\begin{align*}
    \Z=q(\tilde{\Z})&=[\tilde{\Z}_1, \tilde{\Z}_2 e^{\phi(\tilde{\Z}_1)}+\psi(\tilde{\Z}_1)]\\
    q^{-1}(\Z)&=[\Z_1, (\Z_2-\psi(\Z_1))e^{-\phi(\Z_1)}],
\end{align*}
where $\phi$ and $\psi$ are two-layer neural networks, and $\tilde{\Z}=[\tilde{\Z}_1,\tilde{\Z}_2]$ is a partition of $\X$ into two equal-sized chunks along the particle dimension. In this sense, $q^{-1}$ corresponds to $h_\theta$.
The invertible network is a parametrized coupling layer \citep{nice} which is typically used in normalizing flows \citep{normalizingflows} and invertible by definition; further details can be found in the appendix. In the DDEQ forward pass, we keep the particles in $\Z$ corresponding to inputs $\X$ fixed. Once a fixed point $\Z^*$ is found, we apply $q^{-1}$ on $\Z^*$. This ensures the output still contains the input particles $\X$. The final loss is the squared MMD (with the same Riesz kernel as before) between the prediction and the target point cloud. Crucially, at no point in training is the model given \textit{any information} about the class label. A schematic overview can be seen in Figure \ref{fig:pipeline-completion}.
We highlight two key differences between our approach and common point cloud completion algorithms from the literature, such as AtlasNet \citep{groueix2018atlasnetpapiermacheapproachlearning}, PCN \citep{yuan2019pcnpointcompletionnetwork}, or PMP-Net \citep{wen2021pmpnetpointcloudcompletion}: firstly, we recover the input point cloud \textit{exactly} as part of the output, and secondly, we choose the number of target particles \textit{adaptively} based on the input. Both of these properties can be crucial in faithfully recovering completed point clouds.

\section{EXPERIMENTS}\label{sec:exps}

\looseness=-1 We evaluate DDEQs on point cloud classification and point cloud completion. Note that tasks such as point cloud segmentation \citep{qi2017pointnet} that require identifying output particles with input particles are not suitable for our architecture, because the WGF solver inherently treats inputs as measures and removes any ``ordering'' from the particles.

More extensive experimental results, including ablation studies, a computational complexity analysis, and results on the nature of the fixed points can be found in Appendix~\ref{sec:add_expes}. The code is available at: \url{https://github.com/j-geuter/DDEQs}.

\subsection{Datasets}
We train our model on two different datasets: \href{https://www.kaggle.com/datasets/cristiangarcia/pointcloudmnist2d}{MNIST point clouds}, a 2D dataset which features point clouds of digits with up to 350 particles per sample, a train split with $60,000$ samples, and a test split with $10,000$ samples, which we will call \textit{MNIST-pc}; and on \href{https://www.kaggle.com/datasets/balraj98/modelnet40-princeton-3d-object-dataset/data}{ModelNet40}, a 3D dataset which features point clouds of objects of 40 different categories with up to nearly 200.000 particles per sample, a train split with $9,843$ samples, and a test split with $2,468$ samples. Since samples with large numbers of particles are highly redundant and too memory intense to handle, we employ voxel based downscaling, such that each sample contains no more than 800 particles, and dub the resulting dataset \textit{ModelNet40-s}.
We normalize all samples from both datasets to have zero mean and unit variance.
For point cloud completion, we select two particles per (normalized) sample at random, and remove all particles in a radius $r$ around them to create the partial point clouds, where $r=0.6$ for MNIST and $r=0.8$ for ModelNet. We do not impose any restrictions on the distance between the two selected particles. For ModelNet40, while we use all classes for point cloud classification, we select a subset of eight classes (airplane, bathtub, bowl, car, chair, cone, toilet, vase) for point cloud completion. We call the resulting datasets \textit{MNIST-pc-partial} and \textit{ModelNet40-s-partial}. 

\subsection{Implementation Details}\label{sec:implementation}
\sloppy
\looseness=-1 We set the dimension of $\Z$ and the hidden dimension of the encoders to 128. The latent dimension of the bilinear model is 16. The cross-encoder has three layers, while the self-encoders have a single layer\footnote{For classification on MNIST-pc, we reduce the number of layers in the cross-encoder from three to one, which reduces the number of parameters to 776k, since the performance gains of more layers are marginal.}.
This results in a total number of 1.17M parameters for the DDEQ. The linear classifier for MNIST-pc has 1.3k parameters, the one for ModelNet40-s 5k; the invertible coupling layer for point cloud completion has 16.6k. For classification tasks, we set the number of particles in $\Z$ to 10. Note that in theory, for classification, even the edge case of just a single particle in $\Z$ works, which would effectively reduce the inner loop to standard gradient descent. We found that increasing this number to 10 slightly improves performance, but further increasing it was detrimental. This shows that a DDEQ forward pass with a discretized MMD flow can be helpful even in the classification setting. An ablation study of this hyperparameter can be found in Appendix~\ref{sec:ab_study_z}. For point cloud completion, we set the number of free particles to be 27.5\% of the number of particles in the input, which equals the average amount of particles removed to create the partial input point clouds.

We train the model for 5 epochs on both MNIST datasets, and for 20 epochs on ModelNet40-s for classification and 100 epochs for completion (since we select just eight classes for point cloud completion, the dataset contains only 2868 samples), both with a batch size of 64. To batch together samples with varying numbers of particles, we adequately pad samples with zeros and mask out gradients. For the inner loop, we use SGD with a learning rate of 5 and 200 iterations.
The outer loop uses an Adam optimizer with initial learning rate 0.001, and a scheduler which reduces the learning rate by 90\% after 40\% resp. 80\% of the total epochs. On MNIST-pc, we also reduce the number of layers in the cross-encoder from 3 to 1, as additional layers only provide marginal performance gains.
We implement DDEQs with the \verb|torchdeq| library \citep{torchdeq}, which utilizes phantom gradients \citep{phantomgradient} for the backward pass.

\textbf{Point Cloud Classification.}
We compare against PointNet \citep{qi2017pointnet} and Point Transformer (PT) \citep{pointtransformer}, both of which we train from scratch on our datasets with the same number of epochs and with the same learning rate scheduler as DDEQs. Other hyperparameters, such as the optimizers, are taken from the original papers.
The top-1 accuracies on both datasets are reported in Table \ref{tab:classificationaccuracies}.
We see that DDEQs achieve results almost identical to PointNet, and only marginally worse than PT, while being significantly more parameter efficient.

\begin{table}[t]
    \centering
    \caption{Accuracies on Point Cloud Classification}
    \vspace{0.5em}
    \begin{tabular}{lcc}
    \toprule
     Models (size) & \textbf{\small{MNIST-pc}} & \textbf{\small{ModelNet40-s}} \\
    \midrule
    PointNet (1.6M)& 97.5 & 77.3 \\
    PT (3.5M)& 98.6 & 79.2 \\
    \midrule
    DDEQ (776k/1.2M)& 98.1 & 78.2\\
    \bottomrule
    \end{tabular}
    \label{tab:classificationaccuracies}
    \vspace{-1em}
\end{table}

\begin{table}[h]
    \centering
    \caption{Comparison of W2 distances for DDEQ and PCN under noisy training data for MNIST-pc-partial.}
    \begin{tabular}{lcccccccc}
        \toprule
        \textbf{Noisy Particles} & 0\% & 5\% & 20\% & 80\%  \\
        \midrule
        \textbf{W2 (DDEQ)} & 0.33 & 0.17 & 0.19 & 0.33 \\
        \textbf{W2 (PCN)}  & 0.49 & 0.51 & 0.49 & 0.53 \\
        \bottomrule
    \end{tabular}
    \label{tab:w2_comparison}
\end{table}

\begin{figure}
    \centering
    \includegraphics[width=.8\linewidth]{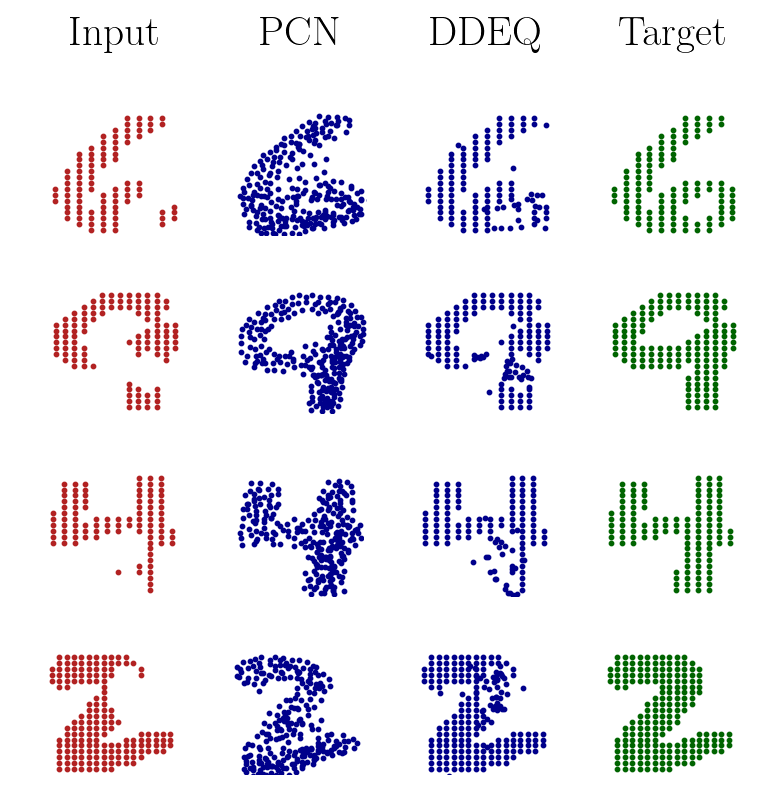}
    \caption{MNIST-pc-partial samples.}
    \label{fig:completed}
    \vspace{-1em}
\end{figure}

\textbf{Point Cloud Completion.}
We compare DDEQs against PCN \citep{yuan2019pcnpointcompletionnetwork} with a grid size of one and 256 points per sample on MNIST-pc-partial and 768 on ModelNet40-s-partial, which yielded the best results amongst a range of hyperparameters. Again, we trained this baseline on our dataset from scratch for the same number of epochs as DDEQs\footnote{We also tried comparing against PMP-Net \citep{wen2021pmpnetpointcloudcompletion} and PMP-Net++ \citep{wen2022pmpnetpointcloudcompletion} using the official \href{https://github.com/diviswen/PMP-Net/tree/main}{GitHub repo}, but could not get it to produce predictions that differ from the input point cloud on our dataset, thus we are not including it.}. While DDEQs allow for input batches to be padded with zeros, most existing frameworks require all inputs to have the same number of points, hence we filled input batches for PCN with randomly sampled points of the input point clouds which improves performance.
Sample results can be seen in Figures \ref{fig:completed} and \ref{fig:completed2}; additional samples and results can be found in Appendix \ref{sec:additionalmodelnet}.
As we can see from the Figures, the DDEQ tends to be good at moving free particles to the right locations, but is having trouble spreading them out evenly in those locations. PCN, on the other hand, spreads out particles very evenly, but tends to produce outputs that are too diffuse, and clearly do not faithfully recover the partial input.
We observed that DDEQs benefit from adding Gaussian noise from $\mathcal{N}(0,1)$ to a small portion of the input points, as can be seen in Table \ref{tab:w2_comparison}, where we quantitatively compare DDEQs and PCN in terms of the Wasserstein-2 loss over varying noising schedules. To make the comparison more fair, we compute the Wasserstein-2 distance only over the free particles for DDEQs, whereas we compute it over all particles for PCN\footnote{Arguably, this comparison slightly favours PCN, as the loss for PCN is also averaged over particles that ``correspond'' to input particles in some way, but it is more fair than averaging over all particles for DDEQs, as DDEQs retain the input point cloud by design.}. We hypothesize this makes training more robust, and consequently we trained DDEQs on $5\%$ noisy samples for point cloud completion.

\begin{figure}
    \centering
    \includegraphics[width=.8\linewidth]{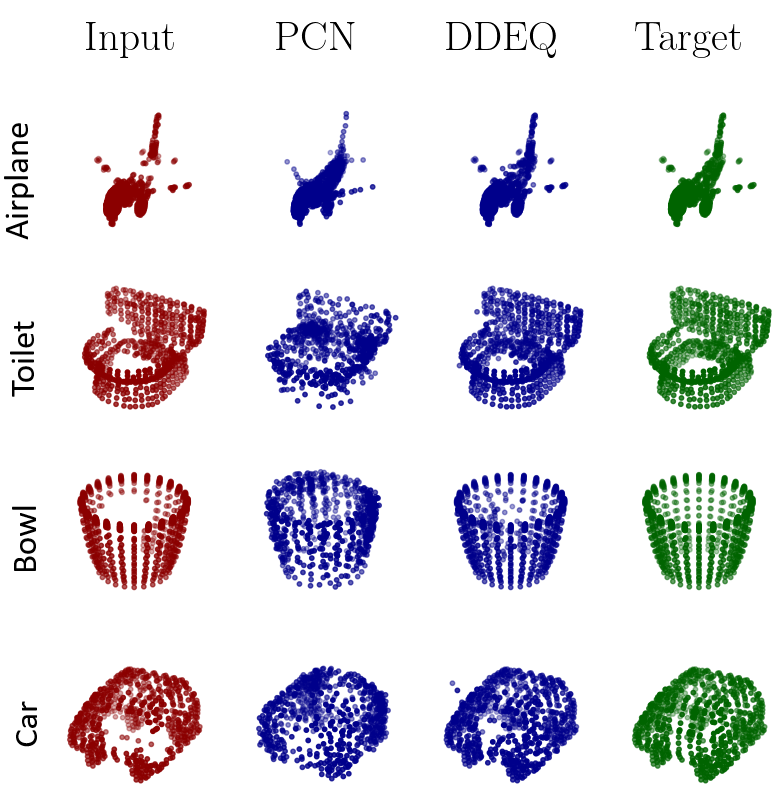}
    \caption{ModelNet40-s-partial samples.}
    \label{fig:completed2}
    \vspace{-1em}
\end{figure}

\section{CONCLUSION}
In this paper, we presented DDEQs, a type of Deep Equilibrium Model which features discrete probability measure inputs and outputs,
and provide a rigorous theoretical framework of the model. We derived a suitable forward pass, which we implement using a discretized Wasserstein gradient flow on the squared MMD between a latent sample and its push-forward under the neural network. We carefully analyzed the forward pass, and proposed a suitable network architecture based on multi-head attention encoders. Through experiments on point cloud classification and point cloud completion, we showed the versatility of the proposed approach, and how to apply it end-to-end on these tasks. Our model achieves competitive results on challenging tasks, such as point cloud classification on ModelNet40, while being significantly more parameter efficient. However, a more extensive empirical analysis would be instructive, and we believe that through more carefully designed architectures and hyperparameters, DDEQs can successfully be applied to a wide range of applications. One of the main drawbacks is the slow convergence of the forward pass, as is typical for MMD flows. Future research includes applying DDEQs to other tasks that require measure-to-measure learning, and to derive methods to speed up the forward pass.

\section*{ACKNOWLEDGEMENTS}

CB acknowledges the support of ANR PEPR PDE-AI. JG and DAM acknowledge support from the Kempner Institute, the Aramont Fellowship Fund, and the FAS Dean’s Competitive Fund for Promising Scholarship.

\bibliographystyle{plainnat}
\bibliography{references}

\appendix
\include{supplement_arxiv}
\end{document}

%% file: tikz/fig_intro.tex
\begin{figure}[t]
\centering

\newcommand{\sh}{1}
\newcommand{\mh}{2}
\newcommand{\lh}{3}
\newcommand{\meh}{0.7}
\newcommand{\rh}{-6}
\definecolor{myyellow}{rgb}{1.0, 1.0, 0.6}
\definecolor{mygreen}{rgb}{.0, 1.0, .0}
\definecolor{encodercolor}{rgb}{.7, .6, .4}
\definecolor{meanpoolcolor}{rgb}{1.3,1,.4}

\newcommand{\trapezoid}[5][]{ 
\begin{scope}[#1]
  \draw[fill=blue!70!green!30, draw=black, line width=1pt] 
    (0, -#2/2 + #4) --           
    (1, -#3/2 + #5) --          
    (1, #3/2 - #5) --         
    (0, #2/2 - #4) --          
    cycle;              
  \end{scope}
}

\newcommand{\invertible}[6][]{ 
\begin{scope}[#1]
  \draw[fill=lime!40, draw=black, line width=1pt] 
    (0, -#2/2 + #4) --           
    (1, -#3/2 + #5) --          
    (1, #3/2 - #5) --         
    (0, #2/2 - #4) --          
    cycle;              
    \node at (0.5, -0.1) {#6};
  \end{scope}
}

\newcommand{\inputbox}[3][]{
\begin{scope}[#1]
    \draw[fill=olive!10, line width=1pt, rounded corners=2.5pt]
    (0, -#3/2) rectangle (1, #3/2);
    \node at (0.5, 0) {#2};
\end{scope}
}

\newcommand{\tinyinputbox}[3][]{
\begin{scope}[#1]
    \draw[fill=red!20, line width=1pt, rounded corners=2.5pt]
    (0, -#3/2) rectangle (0.2, #3/2);
    \node at (0.8, 0) {#2};
\end{scope}
}

\newcommand{\meanpool}[1][]{
\begin{scope}[#1]
    \draw[fill=pink!30,  line width=1pt, rounded corners=2.5pt]
    (0, -\lh/2) rectangle (0.5, \lh/2);
\end{scope}
}

\newcommand{\ddeq}[1][]{
\begin{scope}[#1]
    \draw[fill=teal!30,  line width=1pt, rounded corners=2.5pt]
    (0, -\lh/2) rectangle (2,\lh/2);
    \node at (1, 0) {\small{\textbf{DDEQ}}};
\end{scope}
}

\newcommand{\yellowbox}[2][]{
\begin{scope}[#1]
    \draw[fill=myyellow!80, line width=1pt, rounded corners=2.5pt]
    (0, -#2/2) rectangle (0.5, #2/2);
\end{scope}
}

\begin{tikzpicture}[scale=0.5]
    \sffamily
    \tikzset{
        myarrow/.style={
            -{Latex[length=2.5mm]},  
            thick,                  
            line width=1pt
        }
    }

    \inputbox[shift={(2, \rh-0.5)}]{$\Z$}{\lh}
    \inputbox[shift={(-2, \rh-0.5)}]{$\tilde{\X}$}{\sh}
    \draw[myarrow] (1,\rh-0.5)--(2,\rh-0.5);
    \draw[myarrow] (3,\rh-0.5)--(4,\rh-0.5);
    \draw[-, thick, line width=1pt] (-1,\rh-3.5)--(5,\rh-3.5);
    \draw[myarrow] (5,\rh-3.5)--(5,\rh-2);
    \draw[myarrow] (-1.5,\rh-3.5)--(-1.5,\rh-1);
    \draw[myarrow] (-1,\rh-0.5) -- (0,\rh-0.5);
    \inputbox[shift={(-2,\rh-3.5)}]{$\X$}{\sh} (x);
    \invertible[shift={(0,\rh-0.5)}]{1}{\lh}{0}{0.1}{$\mathbf{q}$}
    \ddeq[shift={(4,\rh-0.5)}];
    \draw[myarrow] (6,\rh-0.5)--(7,\rh-0.5);
    \inputbox[shift={(7,\rh-0.5)}]{$\Z^*$}{\lh}
    \invertible[shift={(7+2,\rh-0.5)}]{\lh}{1}{0.1}{0}{$\mathbf{q^{\scalebox{0.5}{-1}}}$};
    \inputbox[shift={(9.75+2, \rh-0.5)}]{$\Y^*$}{\sh}
    \draw[myarrow] (8,\rh-0.5)--(9,\rh-0.5);
    \draw[myarrow] (10,\rh-0.5) -- (11.75, \rh-0.5);
    
    \node at (-0.9, -12) [circle, draw, thick, line width=1pt, minimum size=1.1cm, path picture={
    \node at (path picture bounding box.center) {
            \includegraphics[width=1cm]{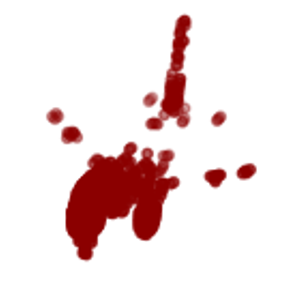}
        };
        \clip (path picture bounding box.center) circle(0.6cm); 
    }] (c1) {};
    \draw[dashed, -{Latex[length=1.5mm]}] (-1.5,\rh-4) to[out=250, in=120] ([shift=(120:1.1cm)]c1);

    \node at (3.25, -12) [circle, draw, thick, line width=1pt, minimum size=1.1cm, path picture={
    \node at (path picture bounding box.center) {
            \includegraphics[width=1cm]{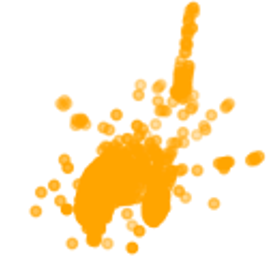}
        };
        \clip (path picture bounding box.center) circle(0.6cm); 
    }] (c2) {};
    \draw[dashed, -{Latex[length=1.5mm]}] (-1, \rh-0.7) to[out=320, in=120] ([shift=(120:1.1cm)]c2);

    \node at (7.5, -12) [circle, draw, thick, line width=1pt, minimum size=1.1cm, path picture={
    \node at (path picture bounding box.center) {
            \includegraphics[width=1cm]{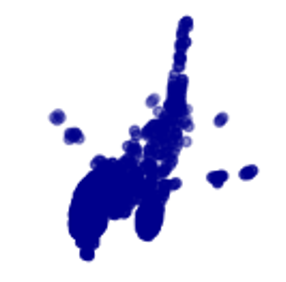}
        };
        \clip (path picture bounding box.center) circle(0.6cm); 
    }] (c3) {};
    \draw[dashed, -{Latex[length=1.5mm]}] (9.75+2.5, \rh-1) to[out=250, in=60] ([shift=(60:1.1cm)]c3);

    \node at (11.75, -12) [circle, draw, thick, line width=1pt, minimum size=1.1cm, path picture={
    \node at (path picture bounding box.center) {
            \includegraphics[width=1cm]{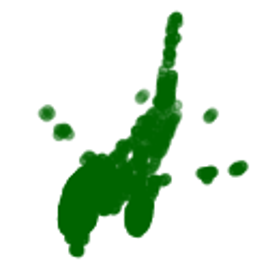}
        };
        \clip (path picture bounding box.center) circle(0.6cm); 
    }] (c4) {};
    \draw[dashed, -{Latex[length=1.5mm]}] (c3.east) to[out=0, in=180] (c4.west);
    \draw[-{Latex[length=1.5mm]}] (7.8+1.1,-12)--(7.5+1.1,-12);

\end{tikzpicture}
\caption{
\looseness=-1 Point cloud completion with DDEQs for an airplane. We add random particles to the partial input point cloud $\X$ (red) to create the DDEQ input $\tilde{\X}$ (orange), which we upscale by an invertible layer $q$. The DDEQ outputs a prediction $\Y^*$ (blue), which is compared against the target (green) with a MMD loss.
}
\label{fig:pipeline-completion}
\vspace{-1em}
\end{figure}

%% file: tikz/ddeq_fig.tex
\begin{figure*}[t]
\centering

\newcommand{\sh}{1}
\newcommand{\mh}{2}
\newcommand{\lh}{3}
\newcommand{\meh}{0.7}

\definecolor{myyellow}{rgb}{1.0, 1.0, 0.6}
\definecolor{mygreen}{rgb}{.0, 1.0, .0}
\definecolor{encodercolor}{rgb}{.7, .6, .4}

\newcommand{\trapezoid}[5][]{ 
\begin{scope}[#1]
  \draw[fill=blue!70!green!30, draw=black, line width=1.5pt] 
    (0, -#2/2 + #4) --           
    (1.5, -#3/2 + #5) --          
    (1.5, #3/2 - #5) --         
    (0, #2/2 - #4) --          
    cycle;              
  \end{scope}
}

\newcommand{\inputbox}[3][]{
\begin{scope}[#1]
    \draw[fill=olive!10, line width=1.5pt, rounded corners=2.5pt]
    (0, -#3/2) rectangle (1, #3/2);
    \node at (0.5, 0) {#2};
\end{scope}
}

\newcommand{\miniencoder}[1][]{
\begin{scope}[#1]
    \draw[thick, fill=gray!20, line width=1pt, rounded corners=2.5pt] (-0.15, -\meh -0.15) rectangle (1.55, \meh + 0.15);
    \draw[fill=orange!40, line width=1pt, rounded corners=1.5pt]
    (0, -\meh) rectangle (0.3, \meh);
    \draw[fill=myyellow!80, line width=1pt, rounded corners=1.5pt]
    (0.4, -\meh) rectangle (0.55, \meh);
    \draw[fill=blue!70!green!30, line width=1pt, rounded corners=1.5pt]
    (0.85, -\meh) rectangle (1.15, \meh);
    \draw[fill=myyellow!80, line width=1pt, rounded corners=1.5pt]
    (1.25, -\meh) rectangle (1.4, \meh);
\end{scope}
}

\newcommand{\encoder}[3][]{
\begin{scope}[#1]
    \draw[fill=mygreen!20,  line width=1.5pt, rounded corners=2.5pt]
    (0, -\lh/2) rectangle (2, \lh/2);
    \node at (1, 1.1) {#2};
    \node at (1, 0.7) {#3};
    \miniencoder[shift={(0.3, -0.5)}]
\end{scope}
}

\newcommand{\bilinear}[2][]{
\begin{scope}[#1]
    \draw[fill=violet!15,  line width=1.5pt, rounded corners=2.5pt]
    (0, -\mh/2) rectangle (3, \mh/2);
    \node at (1.5, 0) {#2};
\end{scope}
}

\newcommand{\yellowbox}[2][]{
\begin{scope}[#1]
    \draw[fill=myyellow!80, line width=1.5pt, rounded corners=2.5pt]
    (0, -#2/2) rectangle (0.5, #2/2);
\end{scope}
}

\begin{tikzpicture}[scale=0.65]
    \sffamily
    \tikzset{
        myarrow/.style={
            -{Latex[length=2.5mm]},  
            thick,                  
            line width=1.5pt
        }
    }
    \inputbox[shift={(0, 0)}]{$\Z$}{\lh}
    \draw[-, thick, line width=1.5pt] (0.5, 1.5) -- (0.5, 2);
    \draw[-, thick, line width=1.5pt] (0.5, 2) -- (13.35, 2);
    \draw[myarrow] (13.35, 2) -- (13.35, 1.5);
    \draw[-, thick, line width=1.5pt] (4.45, 1) -- (4.45, 1.5);
    \draw[-, thick, line width=1.5pt] (4.45, 1.5) -- (9.65, 1.5);
    \draw[myarrow] (9.65, 1.5) -- (9.65, 1);

    \trapezoid[shift={(2.5, 0)}]{\lh}{\mh}{0.05}{0.0}
    \yellowbox[shift={(4.2, 0)}]{\mh}
    \draw[myarrow] (1,0) -- (2.5, 0);
    \draw[-, thick, line width=1.5pt] (4,0) -- (4.2, 0);
    \bilinear[shift={(6.2,0)}]{Bilinear};
    \draw[myarrow] (4.7, 0.5) -- (6.2, 0.5);
    \yellowbox[shift={(9.4, 0)}]{\mh};
    \draw[-, thick, line width=1.5pt] (9.2, 0) -- (9.4, 0);
    \trapezoid[shift={(11.4, 0)}]{\mh}{\lh}{0}{0.05}
    \yellowbox[shift={(13.1, 0)}]{\lh}
    \draw[-, thick, line width=1.5pt] (12.9, 0) -- (13.1, 0);
    \draw[myarrow] (9.9, 0) -- (11.4, 0);
    
    \encoder[shift={(15.1, 0)}]{Self}{Encoder};
    \draw[myarrow] (13.6, 0) -- (15.1, 0);
    \encoder[shift={(18.6, 0)}]{Cross}{Encoder};
    \draw[myarrow] (17.1, 0.5) -- (18.6, 0.5);
    \trapezoid[shift={(22.1, 0)}, rounded corners=2.5pt]{\lh}{\lh}{0}{0};
    \draw[myarrow] (20.6, 0) -- (22.1, 0);
    \yellowbox[shift={(23.8, 0)}]{\lh};
    \draw[-, thick, line width=1.5pt] (23.6, 0) -- (23.8, 0);
    \draw[myarrow] (24.3, 0) -- (25, 0);
    \inputbox[shift={(25,0)}]{$\Z^*$}{\lh}

    \inputbox[shift={(0, -4)}]{$\X$}{\sh}
    \draw[myarrow] (1, -4) -- (2.5, -4);
    \trapezoid[shift={(2.5, -4)}]{\sh}{\mh}{0}{0.05}
    \yellowbox[shift={(4.2, -4)}]{\mh}
    \draw[-, thick, line width=1.5pt] (4, -4) -- (4.2, -4);
    \draw[myarrow] (4.7, -4) -- (11.4, -4);
    \trapezoid[shift={(11.4, -4)}]{\mh}{\lh}{0}{0.05}
    \yellowbox[shift={(13.1, -4)}]{\lh}
    \draw[-, thick, line width=1.5pt] (12.9, -4) -- (13.1, -4);
    \draw[-, thick, line width=1.5pt] (5.45, -4) -- (5.45, -0.5);
    \draw[myarrow] (5.45, -0.5) -- (6.2, -0.5);
    \encoder[shift={(15.1, -4)}]{Self}{Encoder};
    \draw[myarrow] (13.6, -4) -- (15.1, -4);
    \draw[-, thick, line width=1.5pt] (17.1, -4) -- (17.85, -4);
    \draw[-, thick, line width=1.5pt] (17.85, -4) -- (17.85, -0.5);
    \draw[myarrow] (17.85, -0.5) -- (18.6, -0.5);
    \draw[line width=1.5pt, thick, fill=gray!5, rounded corners=1.5pt] (18.6, -5.5) rectangle (26, -2);
    \trapezoid[shift={(19, -3)}, rounded corners=2.5pt]{\sh}{\sh}{0}{0}
    \node at (23.5, -3) {Feedforward network};
    \yellowbox[shift={(19.5, -4.5)}]{\sh}
    \node at (22.5, -4.5) {Layer Norm};

\end{tikzpicture}
\caption{DDEQ Network architecture. There are two residual connection to the second and third layer norm in the top row. All encoders are standard multi-head attention network (see appendix for more details).}
\label{fig:arch}
\vspace{-0.5em}
\end{figure*}

%% file: tikz/fig_classif_ddeq.tex
\begin{figure}
\centering

\newcommand{\sh}{1}
\newcommand{\mh}{2}
\newcommand{\lh}{3}
\newcommand{\meh}{0.7}
\newcommand{\rh}{-6}
\definecolor{myyellow}{rgb}{1.0, 1.0, 0.6}
\definecolor{mygreen}{rgb}{.0, 1.0, .0}
\definecolor{encodercolor}{rgb}{.7, .6, .4}
\definecolor{meanpoolcolor}{rgb}{1.3,1,.4}

\newcommand{\trapezoid}[5][]{ 
\begin{scope}[#1]
  \draw[fill=blue!70!green!30, draw=black, line width=1pt] 
    (0, -#2/2 + #4) --           
    (1, -#3/2 + #5) --          
    (1, #3/2 - #5) --         
    (0, #2/2 - #4) --          
    cycle;              
  \end{scope}
}

\newcommand{\invertible}[6][]{ 
\begin{scope}[#1]
  \draw[fill=green!20, draw=black, line width=1pt] 
    (0, -#2/2 + #4) --           
    (1, -#3/2 + #5) --          
    (1, #3/2 - #5) --         
    (0, #2/2 - #4) --          
    cycle;              
    \node at (0.5, -0.1) {#6};
  \end{scope}
}

\newcommand{\inputbox}[3][]{
\begin{scope}[#1]
    \draw[fill=olive!10, line width=1pt, rounded corners=2.5pt]
    (0, -#3/2) rectangle (1, #3/2);
    \node at (0.5, 0) {#2};
\end{scope}
}

\newcommand{\tinyinputbox}[3][]{
\begin{scope}[#1]
    \draw[fill=olive!10, line width=1pt, rounded corners=2.5pt]
    (0, -#3/2) rectangle (0.2, #3/2);
    \node at (0.8, 0) {#2};
\end{scope}
}

\newcommand{\meanpool}[1][]{
\begin{scope}[#1]
    \draw[fill=purple!30,  line width=1pt, rounded corners=2.5pt]
    (0, -\lh/2) rectangle (0.5, \lh/2);
\end{scope}
}

\newcommand{\ddeq}[1][]{
\begin{scope}[#1]
    \draw[fill=teal!30,  line width=1pt, rounded corners=2.5pt]
    (0, -\lh/2) rectangle (2,\lh/2);
    \node at (1, 0) {\small{\textbf{DDEQ}}};
\end{scope}
}

\newcommand{\yellowbox}[2][]{
\begin{scope}[#1]
    \draw[fill=myyellow!80, line width=1pt, rounded corners=2.5pt]
    (0, -#2/2) rectangle (0.5, #2/2);
\end{scope}
}

\begin{tikzpicture}[scale=0.6]
    \sffamily
    \tikzset{
        myarrow/.style={
            -{Latex[length=2.5mm]},  
            thick,                  
            line width=1pt
        }
    }
    \inputbox[shift={(2, 0)}]{$\Z$}{\lh}
    \draw[myarrow] (3,0)--(4,0);
    \draw[-, thick, line width=1pt] (3,-3)--(5,-3);
    \draw[myarrow] (5, -3) -- (5,-1.5);
    \inputbox[shift={(2,-3)}]{$\X$}{\sh}
    \ddeq[shift={(4,-0)}];
    \draw[myarrow] (6,-0)--(7,0);
    \inputbox[shift={(7,-0)}]{$\Z^*$}{\lh}
    \draw[myarrow] ((8,-0)--(9,0);
    \meanpool[shift={(7+2,0)}];
    \draw[-, thick, line width=1pt] (7.5+2,0)--(7.75+2,0);
    \trapezoid[shift={(7.75+2,0)}]{\lh}{1}{0.1}{0};
    \inputbox[shift={(9.75+2, 0)}]{$\Y^*$}{\sh}
    \draw[myarrow] (8.75+2,-0)--(9.75+2,-0);
\end{tikzpicture}
\caption{The point cloud classification pipeline, where $\Z$ is initialized independently of $\X$, and the DDEQ is followed by a max pool (purple) and linear (blue) layer.}
\label{fig:pipeline-classification}

\vspace{-1em}
\end{figure}

%% file: supplement_arxiv.tex
\onecolumn
\aistatstitle{Supplementary Materials}

\section{BACKGROUND}
In this section, we provide background on some of the mathematical concepts used in the paper. In Section \ref{sec:measuretheory}, we recall two definitions from measure theory used throughout the paper; in Section \ref{sec:ot}, we recall some of the basic definitions from OT (see \citep{santambrogio2015optimal, villani2021topics} for more details). Section \ref{sec:distances} defines the Wasserstein distance and MMD between probability measures, and in Section \ref{sec:wgfs} we define Wasserstein gradient flows (see \citep{villani2021topics,chewi2024statistical,lee2019riemannianmanifolds, lee2003smoothmanifolds} for more details).

\subsection{Measure Theory}\label{sec:measuretheory}
All functions and sets used throughout the paper are considered to be Borel measurable, and all measures are considered to be Borel measures, without explicitly mentioning it; e.g., when we say \textit{``measurable''}, we mean Borel measurable.
\begin{appendixdefinition}[Space of Measures with bounded second Moments]
    Denote by $\mathcal{P}(\R^d)$ the space of probability measures over $\R^d$ with respect to the Borel-$\sigma$-algebra over $\R^d$. Then the space of probability measures of bounded second moments, denoted by $\cP(\R^d)\subset\mathcal{P}(\R^d)$, contains all measures $\mu\in\mathcal{P}(\R^d)$ such that
\begin{equation*}
    \int \|x\|^2 \d\mu(x)<\infty.
\end{equation*}
\end{appendixdefinition}

\begin{appendixdefinition}[Pushforward]
    Let $\mu\in\cP(\R^d)$, and $F:\R^d\to\R^d$ a (measurable) function. Then the pushforward of $\mu$ under $F$, denoted by $F_\#\mu$, is a probability measure defined by
    \begin{equation*}
        F_\#\mu(B)=\mu\big(F^{-1}(B)\big)
    \end{equation*}
    for measurable $B\subset\R^d$.
\end{appendixdefinition}

\subsection{Optimal Transport}\label{sec:ot}
In this section, we define the optimal transport problem on $\R^d$.
Let $\mu,\nu\in\cP(\R^d)$.

\begin{appendixdefinition}[Coupling]
    The set of couplings between $\mu$ and $\nu$ is
    \begin{equation*}\Pi(\mu,\nu) = \{\gamma\in \cP(\R^d\times \R^d),\ \pi^1_\#\gamma=\mu,\ \pi^2_\#\gamma=\nu\}
    \end{equation*}
    where $\pi^i$ denotes the projection on the $i^{th}$ coordinate.
\end{appendixdefinition}

\begin{appendixdefinition}[Optimal Transport Problem]\label{def:otprimal}
    Let $c:\R^d\times \R^d\to\R$ be a cost function. The optimal transport problem is defined as:
    \begin{equation}\label{eq:ot}
        \inf_{\gamma\in\Pi(\mu,\nu)}\ \int c(x,y)\ \mathrm{d}\gamma(x,y)
    \end{equation}
    The infimum in (\ref{eq:ot}) is called the \textit{transport cost}, and the minimizer $\gamma$, if it exists, the \textit{optimal transport plan}.
\end{appendixdefinition}

Under minimal assumptions on $c$, such as lower semicontinuity and boundedness from below, one can prove existence of an optimal transport plan; this result extends to the more general case of Polish spaces as well \citep{villani2021topics}.

\subsection{Discrepancies between Probability Measures}\label{sec:distances}

Different discrepancies can be used to compare probability measures, and each comes with its own advantages and disadvantages, e.g. from a computational point of view or in terms of their theoretical properties. In this work, we focus mainly on two discrepancies: the Wasserstein distance \citep{villani2021topics} and the Maximum Mean Discrepancy (MMD) \citep{gretton2012kernel}.

\begin{appendixdefinition}[Wasserstein Distance]
    For $1\le p<\infty$, the Wasserstein-$p$ distance $W_p$ is the $p^{th}$ root of the transport cost of the optimal transport problem \eqref{eq:ot} with cost $c(x,y)=\norm{x-y}_p^p$, i.e.
    \begin{equation*}
    W_p^p(\mu,\nu) = \inf_{\gamma\in\Pi(\mu,\nu)}\ \int \|x-y\|_p^p\ \mathrm{d}\gamma(x,y).
    \end{equation*}
\end{appendixdefinition}

It can be shown that Wasserstein distances are indeed distances on $\mathcal{P}_p(\R^d)$. In particular, $\cP(\R^d)$ endowed with $W_2$, denoted by $(\cP(\R^d),W_2)$ and called the \textit{Wasserstein space}, has a Riemannian structure \citep{otto2001geometry}, which allows defining notions such as tangent spaces or gradients. However, it is costly to compute in practice. In this work, we leverage the Riemannian structure of the Wasserstein space to minimize the MMD between probability distributions, which we define now.

\begin{appendixdefinition}[Reproducing Kernel Hilbert Space]
    A Reproducing Kernel Hilbert Space (RKHS) is a Hilbert space $\mathcal{H}$ of functions from a space $\mathcal{X}$ to $\mathbb{R}$ in which point evaluations are continuous linear functions. By the Riesz representation theorem, this is equivalent to the existence of a symmetric, positive definite kernel $k:\mathcal{X}\times\mathcal{X}\to\mathbb{R}$ such that
    \begin{equation*}
        f(x)=\langle f,k(x,\cdot)\rangle_{\mathcal{H}}\quad \forall f\in\mathcal{H}.
    \end{equation*}
\end{appendixdefinition}

\begin{appendixdefinition}[Maximum Mean Discrepancy]
    Let $\mathcal{H}$ be a RKHS with kernel $k:\R^d\times\R^d\to \R$. The MMD is defined as \citep{gretton2012kernel}
\begin{equation*}
    \begin{aligned}
        \mathrm{MMD}(\mu,\nu) &= \iint k(x,y)\ \mathrm{d}(\mu-\nu)(x)\mathrm{d}(\mu-\nu)(y) \\
        &= \sup_{f,\ \|f\|_\mathcal{H}\le 1}\ \left|\int f\mathrm{d}\mu - \int f\mathrm{d}\nu\right|.
    \end{aligned}
\end{equation*}
\end{appendixdefinition}

The MMD can be shown to be a distance if and only if the \textit{kernel mean embedding} $\mu\mapsto \int k(x,\cdot)\d \mu(x)$ is injective, which is e.g. the case for the Gaussian kernel $k(x,y)= e^{-\frac{\norm{x-y}^2}{2\sigma^2}}$ \citep{chewi2024statistical}. The Riesz kernel $k(x,y)=-\norm{x-y}$ used in this work is not positive definite, hence does not define an RKHS and is thus technically not an MMD. However, it can be shown to be a distance \citep{sejdinovic2013equivalence}.

The MMD between discrete probability measures can be evaluated in $\mathcal{O}(n^2)$, where $n$ is the number of particles. In contrast, computing the Wasserstein distance requires $\mathcal{O}(n^3 \log n)$ computations \citep{peyre2019computational}.

\subsection{Wasserstein Gradient Flows}\label{sec:wgfs}

In this section, we recall some definitions and properties of Wasserstein gradient flows. As $(\cP(\R^d), W_2)$ is a metric space that admits a quasi-Riemannian structure~\citep{otto2001geometry} (i.e., there exists a natural scalar product in each tangent space, even though it is not a Riemannian manifold per se), we can define a tangent space at $\mu\in\cP(\R^d)$ as
\begin{equation}
    T_\mu \cP(\mathbb{R}^d)=\overline{\{\nabla \psi\ |\ \psi:\mathbb{R}^d\to\mathbb{R} \text{ compactly supported, smooth}\}}^{L^2(\mu)} \subset L^2(\mu), \label{eq:8}
\end{equation}
where the closure is taken in $L^2(\mu)$, see \citep[Definition 8.4.1]{ambrosio2008gradient}.
Moreover, the tangent space is endowed with the scalar product
\begin{equation*}
    \langle v,w\rangle_\mu:=\int\langle v,w\rangle \d \mu
\end{equation*}

for $v,w\in T_\mu \cP(\mathbb{R}^d)$. Intuitively, the vector fields in the tangent space can be thought of as pointing into the direction that mass moves to at any point $x\in\mathbb{R}^d$. 

One can show that the Wasserstein distance is indeed the distance induced by this structure, i.e. that the squared Wasserstein distance between two distributions $\mu_0,\mu_1\in \cP(\mathbb{R}^d)$ is equal to the minimum length of a constant-speed geodesic connecting $\mu_0$ and $\mu_1$ (\textit{Benamou-Brenier dynamic formulation}):
\begin{equation}\label{eq:1}
    W^2_2(\mu_0,\mu_1)=\inf \left\{\int_0^1 \norm{v_t}^2_{\mu_t}\d t\ \Big|\ (\mu_t,v_t) \text{ solves } \partial_t\mu_t+\div(\mu_t v_t)=0\right\}.
\end{equation}
Here,
\begin{equation}\label{eq:continuityequation}
    \partial_t\mu_t+\div(\mu_t v_t)=0 
\end{equation}
is the so-called \textit{continuity equation}, which is satisfied whenever a path $(\mu_t)_t\subset \cP(\mathbb{R}^d)$ evolves according to a time vector field $(v_t)_t$, i.e. whenever for $X_t\sim \mu_t$, it holds that
\begin{equation*}
    \dot{X}_t=v_t(X_t).
\end{equation*}
Hence, the integral in equation (\ref{eq:1}) equals the length of a geodesic $(\mu_t)_t$ which connects $\mu_0$ to $\mu_1$ along a vector field $(v_t)_t$. In general, the vector field $(v_t)_t$ satisfying the continuity equation for $(\mu_t)_t$ is not unique. However, we will see below that there exists a unique such vector field which lies in the tangent space of $\mu_t$ at all times.

Now, let $\cG: \cP(\mathbb{R}^d)\to\mathbb{R}$ be a functional. 
 
Let us introduce the notion of sub- and super-differentiability on the Wasserstein space (see \emph{e.g.} \citep{bonnet2019pontryagin, lanzetti2022first}).
\begin{appendixdefinition}
    Let $\mu\in\cP(\R^d)$. A map $\xi\in L^2(\mu)$ belongs to the sub-differential $\partial^-\cG(\mu)$ of $\cF$ at $\mu$ if for all $\nu\in\cP(\R^d)$,
    \begin{equation*}
        \cG(\nu) \ge \cF(\mu) + \sup_{\gamma\in\Pi_o(\mu,\nu)}\ \int \langle \xi(x), y-x\rangle\ \mathrm{d}\gamma(x,y) + o\big(W_2(\mu,\nu)\big),
    \end{equation*}
    with $\Pi_o(\mu,\nu)$ the set of optimal couplings between $\mu$ and $\nu$. Similarly, $\xi\in L^2(\mu)$ belongs to the super-differential $\partial^+\cG(\mu)$ of $\cG$ at $\mu$ if $-\xi\in\partial^-(-\cG)(\mu)$.
\end{appendixdefinition}
We also say that a functional is Wasserstein differentiable if it admits sub- and super-differentials which coincide.
\begin{appendixdefinition} \label{def:wasserstein_gradient}
    $\cG$ is called \textit{Wasserstein differentiable at $\mu\in\cP(\R^d)$} if $\partial^-\cG(\mu)\cap \partial^+\cG(\nu)\neq \emptyset$. In this case, we say that $\gW\cG(\mu)\in\partial^-\cG(\mu)\cap \partial^+\cG(\mu)$ is a \textit{Wasserstein gradient} of $\cG$ at $\mu$, and it satisfies for any $\nu\in\cP(\R^d)$, $\gamma\in\Pi_o(\mu,\nu)$,
    \begin{equation}
        \cG(\nu) = \cG(\mu) + \int \langle \gW\cG(\mu)(x), y-x\rangle\ \mathrm{d}\gamma(x,y) + o\big(W_2(\mu,\nu)\big).
    \end{equation}
\end{appendixdefinition}
If a Wasserstein gradient exists, then there is always a unique Wasserstein gradient in the tangent space $T_\mu\cP(\R^d)$ \citep[Proposition 2.5]{lanzetti2022first}, and we restrict to it in practice. Examples of Wasserstein differentiable functionals include potential energies $\cG(\mu)=\int V\mathrm{d}\mu$ or interaction energies $\cG(\mu)=\frac12 \iint W(x,y)\ \mathrm{d}\mu(x)\mathrm{d}\mu(y)$ for $V:\R^d\to\R$ and $W:\R^d\times\R^d\to\R$ differentiable and smooth, see \citep[Section 2.4]{lanzetti2022first}.

We also recall the following result of \citet{lanzetti2022first}, which states that the Wasserstein gradients can be used for the Taylor expansion of functionals for any coupling.

\begin{appendixproposition}[Proposition 2.6 in \citep{lanzetti2022first}] \label{prop:strong_diff_w}
    Let $\mu,\nu\in\cP(\R^d)$, $\gamma\in\Pi(\mu,\nu)$ any coupling and let $\cF:\cP(\mathbb{R}^d)\to\R$ be Wasserstein differentiable at $\mu$ with Wasserstein gradient $\gW\cF(\mu)\in T_\mu\cP(\R^d)$. Then,
    \begin{equation*}
        \cF(\nu) = \cF(\mu) + \int \langle \gW\cF(\mu)(x), y-x\rangle \ \mathrm{d}\gamma(x,y) + o\left(\sqrt{\int \|x-y\|_2^2\ \mathrm{d}\gamma(x,y)}\right).
    \end{equation*}
\end{appendixproposition}

It turns out that under suitable regularity assumptions, the Wasserstein gradient coincides with the gradient of the first variation of the functional \citep{chewi2024statistical}.
\begin{appendixdefinition}[First Variation]
    If it exists, the first variation of $\cG$ at $\mu$, denoted by $\frac{\delta \cG}{\delta\mu}(\mu):\mathbb{R}^d\to\mathbb{R}$, is defined as the unique map such that
    \begin{equation*}
        \lim_{\epsilon\to 0}\frac{\cG(\mu+\epsilon \chi)-\cG(\mu)}{\epsilon}=\int\frac{\delta\cG}{\delta\mu}(\mu)\d \chi
    \end{equation*}
    for all perturbations $\chi$ such that $\int\mathrm{d}\chi=0$.
\end{appendixdefinition}
\begin{appendixproposition}[Lemma 10.4.1 in \citep{ambrosio2008gradient}]
    Assume that $\mu$ is absolutely continuous with respect to the Lebesgue measure, and that its density lies in $C^1(\R^d)$. Also assume that $\cG(\mu)<\infty$. Then the Wasserstein gradient of $\cG$ at $\mu$ is equal to the gradient of the first variation at $\mu$, i.e.:
    \begin{equation*}
        \nabla_W\cG(\mu)=\nabla\frac{\delta \cG}{\delta\mu}(\mu).
    \end{equation*}
\end{appendixproposition}
Moreover, for any curve $(\mu_t)_{t\ge0}$ in $\mathcal{P}_{2}(\R^d)$ with associated tangent vectors $(v_t)_{t\ge0}$, it holds:
\begin{equation} \label{eq:chain_rule}
    \partial_t \cG(\mu_t) = \langle \nabla_W \cG(\mu_t),v_t\rangle_{\mu_t}.
\end{equation}

Indeed, applying Proposition \ref{prop:strong_diff_w} between $\mu_t$ and $(\mathrm{Id} + s v_t)_\#\mu_t$ for any $s>0$, with $\gamma=(\mathrm{Id}, \mathrm{Id}+sv_t)_\#\mu_t\in\Pi\big(\mu_t, (\mathrm{Id}+sv_t)_\#\mu_t\big)$, we get
\begin{equation*}
    \frac{\cG\big((\mathrm{Id} + sv_t)_\#\mu_t\big) - \cG(\mu_s)}{s} = \int \langle \gW \cG(\mu_t)(x), v_t(x)\rangle\ \mathrm{d}\mu_t(x) + o(1) \xrightarrow[s\to 0]{} \langle \gW \cG(\mu_t), v_t\rangle_{\mu_t},
\end{equation*}
and thus $\partial_t\cG(\mu_t) = \langle \gW \cG(\mu_t), v_t\rangle_{\mu_t}$.

\begin{appendixdefinition}[Wasserstein Gradient Flow]\label{def:wgf}
    A Wasserstein gradient flow (WGF) of $\cG$ is a path $(\mu_t)_{t\ge0}\subset\mathcal{P}_{2}(\mathbb{R}^d)$ starting at some $\mu_0$ and moving along the negative Wasserstein gradient of $\cG$. In terms of the continuity equation, this reads (in the distributional sense)
    \begin{equation*}
        \partial_t\mu_t + \div\big(\mu_tv_t\big) = 0,
    \end{equation*}
    where $-v_t$ is a subgradient of $\cG$ at $\mu_t$, i.e. for all $\nu\in\cP(\R^d)$,
    \begin{equation*}
        \cG(\nu) \ge \cG(\mu_t) + \sup_{\gamma\in\Pi_o(\mu_t,\nu)}\ \int \langle-v_t(x), y-x\rangle\ \mathrm{d}\gamma(x,y) + o\big(W_2(\mu_t,\nu)\big),
    \end{equation*}
    with $\Pi_o(\mu,\nu)$ the set of optimal couplings.
\end{appendixdefinition}
Note that the velocity field generating the path $t\mapsto \mu_t$ in Definition \ref{def:wgf} is not unique. However, the Wasserstein gradient $\nabla_W \cG(\mu_t)$ is the unique element amongst all such velocity fields $(v_t)_{t\ge 0}$ which lies in the tangent space of $\mu_t$, i.e. for which $v_t\in T_{\mu_t}\cP(\R^d)$ holds for all $t$.

We now recall the Wasserstein gradient of the MMD to a fixed target measure, which can be written as a sum of potential and interaction energies, see \citep{mmdflow} for more details.
\begin{proposition}
    Let $\nu\in\cP(\R^d)$, and define $\cF(\mu) = \frac12 \mathrm{MMD}(\mu,\nu)^2$ for $\mu\in\cP(\R^d)$. Then, $\cF$ is Wasserstein differentiable at any $\mu\in\cP(\R^d)$, with Wasserstein gradient
    \begin{equation*}
        \gW\cF(\mu) = \int \nabla_1 k(\cdot, y)\ \mathrm{d}\mu(y) - \int \nabla_1 k(\cdot, z)\ \mathrm{d}\nu(z),
    \end{equation*}
    where $\nabla_1$ denotes the gradient w.r.t. the first argument.
\end{proposition}

\section{PROOFS}\label{sec:proofs}

\subsection{Wasserstein Gradient}
\setcounter{theoremcounter}{0}
\begin{proposition} \label{prop:composition_grad}
    Let $\mu\in\cP(\R^p)$, $\cF:\cP(\R^p)\to\R$, $T:\R^p\to\R^p \in L^2(\mu)$ a $\mu$-almost everywhere differentiable map with $\nabla T$ bounded in operator norm, \emph{i.e.} $\sup_x\ \|\nabla T(x)\|_{\mathrm{op}} < +\infty$, and define $\tilde{\cF}(\mu):=\cF(T_\#\mu)$. Up to a set of measure zero, if the Wasserstein gradient of $\cF$ at $T_\#\mu$ exists, then the Wasserstein gradient of $\tilde{\cF}$ at $\mu$ also exists, and it holds:
    \begin{equation*}
        \gW\tilde{\cF}(\mu) = (\nabla T) \cdot \gW\cF(T_\#\mu)\circ T = \nabla \left(\frac{\delta\cF}{\delta\mu}(T_\#\mu)\circ T\right).
    \end{equation*}
\end{proposition}

\begin{proof}
    
    Let $\mu,\nu\in\cP(\R^p)$ and $\gamma\in\Pi_o(\mu,\nu)$ an optimal coupling between $\mu$ and $\nu$. Let $\Tilde{\gamma}=(T,T)_\#\gamma\in\Pi(T_\#\mu,T_\#\nu)$ a coupling between $T_\#\mu$ and $T_\#\nu$. Since $\cF$ is Wasserstein differentiable at $T_\#\mu$, we have by Proposition \ref{prop:strong_diff_w} \citep[Proposition 2.6]{lanzetti2022first}
    \begin{equation*}
        \begin{aligned}
            \Tilde{\cF}(\nu) - \Tilde{\cF}(\mu) &= \cF(T_\#\nu) - \cF(T_\#\mu) \\
            &= \int \langle \gW\cF(T_\#\mu)(x), y-x\rangle\ \mathrm{d}\Tilde{\gamma}(x,y) + o\left(\sqrt{\int \|x-y\|^2\ \mathrm{d}\Tilde{\gamma}(x,y)}\right) \\
            &= \int \langle \gW\cF(T_\#\mu)(T(x)), T(y)-T(x)\rangle\ \mathrm{d}\gamma(x,y) + o\left(\sqrt{\int \|T(x)-T(y)\|^2\ \mathrm{d}\gamma(x,y)}\right).
        \end{aligned}
    \end{equation*}
    Then, by the Jet expansion which generalizes the Taylor expansion for maps $T:\R^p\to \R^p$ (see \emph{e.g.} \citep{chen2024jet}), we have for $x,y\in\R^p$,
    \begin{equation*}
        T(y) = T(x) + \nabla T(x)^T (y-x) + O(\|x-y\|^2).
    \end{equation*}
    Assuming that the norm operator of $\nabla T$ is bounded, we can write $\|\nabla T(x)^T (y-x)\| \le \|\nabla T(x)\|_{\mathrm{op}} \|y-x\| \le M\|y-x\|$ for $M>0$, and thus $\int \|\nabla T(x)^T(y-x)\|^2\ \mathrm{d}\gamma(x,y) \le M W_2^2(\mu,\nu)$. Therefore, we obtain
    \begin{equation*}
        \begin{aligned}
            \Tilde{\cF}(\nu) - \Tilde{\cF}(\mu) &= \int \langle \gW\cF(T_\#\mu)(T(x)), \nabla T(x)^T (y-x)\rangle\ \mathrm{d}\gamma(x,y) + o\big(W_2(\mu,\nu)\big) \\
            &= \int \langle \nabla T(x) \gW\cF(T_\#\mu)(T(x)), y-x\rangle\ \mathrm{d}\gamma(x,y) + o\big(W_2(\mu,\nu)\big).
        \end{aligned}
    \end{equation*}
    Thus, we conclude that $\gW\Tilde{\cF}(\mu) = (\nabla T)\cdot \gW\cF(T_\#\mu)\circ T$ by Definition \ref{def:wasserstein_gradient}.
    
\end{proof}

\setcounter{theoremcounter}{1}
\begin{corollary}
    Let $\cG(\mu):=\frac{1}{2}\mmds{\mu}{T_\#\mu}$. Define the witness function $f_\mu$ as
    \begin{equation*}
        f_\mu:=\int k(\cdot,y)\d\mu(y)-\int k(\cdot,y)\d (T_\#\mu)(y),
    \end{equation*}
    where $k:\R^p\times\R^p\to\R$ is the kernel of the MMD, which ought to be differentiable almost everywhere. Then the Wasserstein gradient of $\cG$ exists, and is equal to
    \begin{equation}
        \nabla_W\cG(\mu) = \nabla f_\mu-\nabla(f_\mu\circ T).
    \end{equation}
\end{corollary}

\begin{proof}
    First, recall that for $\cF(\mu) = \frac12 \mathrm{MMD}^2(\mu,\nu)$, $\gW\cF(\mu) = \nabla f_{\mu,\nu}$ with $f_{\mu,\nu} = \int k(\cdot, y)\ \mathrm{d}\mu(y) - \int k(\cdot, z)\ \mathrm{d}\nu(z)$ \citep{mmdflow}; note that in this notation, $f_\mu$ from above corresponds to $f_{\mu,T_\#\mu}$.

    For $\cG(\mu) = \frac12 \mathrm{MMD}^2(\mu, T_\#\mu)$, we have $\gW\cG(\mu) = \gW \cG_1(\mu) + \gW\cG_2(\mu)$ with $\cG_1(\nu) = \frac12\mathrm{MMD}^2(\nu, T_\#\mu)$ and $\cG_2(\nu) =\frac12 \mathrm{MMD}^2(T_\#\nu, \mu)$. Moreover,
    \begin{equation*}
        \begin{aligned}
            \gW\cG_1(\nu) &= \nabla f_{\nu, T_\#\mu} \\
            &= \int \nabla_1 k(\cdot, x')\ \mathrm{d}\nu(x') - \int \nabla_1 k(\cdot, y)\ \mathrm{d}(T_\#\mu)(y) \\
            &= \int \nabla_1 k(\cdot, x')\ \mathrm{d}\nu(x') - \int \nabla_1 k(\cdot, T(y))\ \mathrm{d}\mu(y).
        \end{aligned}
    \end{equation*}
    
    Furthermore, using Proposition \ref{prop:composition_grad},
    \begin{equation*}
        \begin{aligned}
            \gW\cG_2(\nu) &= \nabla T  (\nabla f_{T_\#\nu, \mu}\circ T) \\
            &= \int (\nabla T) \nabla_1 k\big(T(\cdot), T(x')\big) \ \mathrm{d}\nu(x') - \int (\nabla T) \nabla_1 k\big(T(\cdot), y\big)\ \mathrm{d}\mu(y) \\
            &= - \nabla (f_{\mu, T_\#\nu}\circ T)
        \end{aligned}
    \end{equation*}
    Thus,
    \begin{equation*}
        \begin{aligned}
            \gW\cG(\mu)(x) &= \gW\cG_1(\mu)(x) + \gW\cG_2(\mu)(x) \\
            &= \nabla f_{\mu, T_\#\mu}(x) - \nabla (f_{\mu,T_\#\mu}\circ T)(x).
        \end{aligned}
    \end{equation*}
\end{proof}

\subsection{Implicit Gradient}\label{sec:proof_th_implicit_grad}

In this section, we prove Theorem \ref{thm:gradient}. A rigorous proof of this statement is more involved than one might think, since it requires the derivative of a map $F:\cP(\R^p)\to \cP(\R^p)$, as well as of maps $f:\R^w\to \cP(\R^p)$, which we will carefully define in the following, following the construction in \citep{lessel2020differentiablemapswassersteinspaces} for maps $F:\cP(\R^p)\to \cP(\R^p)$, and extending it to maps $f:\R^w\to \cP(\R^p)$.

\begin{appendixdefinition}[Absolutely Continuous Curve]
    Let $(X,d)$ be a metric space, and $I\subset \mathbb{R}$ an interval. A curve $\gamma:I\to X$ is called \textit{absolutely continuous} (a.c.), if there exists a function $f\in L^1(I)$ such that
    \begin{equation*}
        d(\gamma(t),\gamma(s))\le \int_t^2 f(r)\d r, \quad \forall s,t\in I,\ t\le s.
    \end{equation*}
\end{appendixdefinition}

We have seen in section \ref{sec:wgfs} that given a functional and its Wasserstein gradient, there exists a unique vector field which induces this gradient (in the sense of the continuity equation) which lies in the tangent space. This statement can be extended to any absolutely continuous curve $\mu_t$ of measures, cmp. \citep{lessel2020differentiablemapswassersteinspaces}.
\begin{appendixdefinition}[Tangent Couple]
    Let $\mu_t:[0,1]\to \cP(\R^p)$ be an absolutely continuous curve. Then there exists a unique vector field $v_t:[0,1]\times \R^p\to\R^p$, such that $(\mu_t,v_t)$ fulfills the continuity equation \eqref{eq:continuityequation}, and such that $v_t$ lies in the tangent space $T_{\mu_t}\cP(\R^p)$ for (almost all) times $t$. We will call such a couple $(\mu_t,v_t)$ a \emph{tangent couple}.
\end{appendixdefinition}
Similarly, a tangent couple in Euclidean space is a pair $(\theta_t,h_t)$, where $\theta_t:[0,1]\to\R^q$, $h_t:[0,1]\to\R^q$ such that $h_t=\partial_t \theta_t$ (which one could extend to a vector field $h_t$ that coincides with $\partial_t\theta_t$ on the support of $\theta_t$, which would be a closer equivalent to the measure-theoretic definition above).

\begin{appendixdefinition}[Absolutely Continuous Map]
    Let $X=\R^q$ or $X=\cP(\R^p)$, and $F:X\to\cP(\R^d)$. $F$ is called \emph{absolutely continuous} if for any a.c. curve $\mu_t$ in $X$, the curve $F(\mu_t)\subset\cP(\R^d)$ is a.c. (up to redefining it on a set of measure $0$).
\end{appendixdefinition}

We are now ready to define the differential of a map $F:\cP(\R^p)\to\cP(\R^d)$ (which we will use for $d=p$), following Definition 27 from \citep{lessel2020differentiablemapswassersteinspaces}, which we can also naturally extend to functions $F:\R^q\to\cP(\R^d)$.
\begin{appendixdefinition}[Differentiable Map] \label{def:differentiablemap}
    Let $X=\R^q$ or $X=\cP(\R^p)$, and $F:X\to\cP(\R^d)$. We say that $F$ is \emph{differentiable} if there exist bounded linear maps $DF_\mu:T_\mu X\to T_{F(\mu)}\cP(\R^d)$ such that for every tangent couple $(\mu_t,v_t)$ in $X$, the pair $(F(\mu_t), DF_{\mu_t}(v_t))$ fulfills the continuity equation and is a tangent couple in $\cP(\R^d)$.
\end{appendixdefinition}

As noted in \citep{lessel2020differentiablemapswassersteinspaces}, it is difficult to define pointwise differentiability for maps between Wasserstein spaces, as tangent vector fields $v_t$ are not defined pointwise. Hence, to be precise, all pointwise statements in the following will only hold almost everywhere.

We can now state the chain rule from \citep{lessel2020differentiablemapswassersteinspaces}, and extend it to compositions between Euclidean and Wasserstein space.

\begin{appendixproposition}[Chain Rule]\label{prop:chainrule}
    Let $X=\R^q$ or $X=\cP(\R^q)$ and $Z=\R^d$ or $Z=\cP(\R^d)$. Let $G:X\to\cP(\R^p)$, and $F:\cP(\R^p)\to Z$ be differentiable. Then $F\circ G:X\to Z$ is also differentiable, and
    \begin{equation*}
        D(F\circ G)_\mu (v)=(D F_{G(\mu)}\circ D G_\mu)(v).
    \end{equation*}
\end{appendixproposition}
\begin{proof}
    The case where both $X$ and $Z$ are Wasserstein spaces is proven in \citep[Corollary 33. 3)]{lessel2020differentiablemapswassersteinspaces}. Note that the proof works exactly the same if $X$ or $Z$ is Euclidean using the right notion of tangent couple.
\end{proof}

\begin{appendixremark}
    The result from Proposition \ref{prop:chainrule} recovers equation \eqref{eq:chain_rule} in the case $X=Z=\mathbb{R}$. To see this,
    let $F:\cP(\R^p)\to \R$ be a Wasserstein differentiable map, $(\mu_t)_{t\ge 0}$ a curve in $\cP(\R^p)$ with associated tangent vectors $(v_t)_{t\ge 0}$. By \eqref{eq:chain_rule} we get the chain rule
    \begin{equation*}
        \partial_t F(\mu_t) = \langle \gW F(\mu_t), v_t\rangle_{\mu_t}.
    \end{equation*}
    Thus $\big(F(\mu_t), \langle \gW F(\mu_t), v_t\rangle_{\mu_t}\big)$ is a tangent couple on $\R$, and by Definition \ref{def:differentiablemap} this implies
    \begin{equation*}
        DF_{\mu_t}(v_t)=\langle \gW F(\mu_t), v_t\rangle_{\mu_t}.
    \end{equation*}
   
    Moreover, with $G:t\mapsto \mu_t$ we get that $DG_{t}(v_t) = v_t$ is the differential of $G$ at $t$ by Definition \ref{def:differentiablemap} since $(\mu_t, v_t)$ is a tangent couple. By Proposition \ref{prop:chainrule}, we have
    \begin{equation}
        \partial_t F(\mu_t) = D(F\circ G)_{t}(v_t) = DF_{G(t)}\big( DG_t(v_t)\big) = DF_{\mu_t}(v_t) = \langle \gW F(\mu_t), v_t\rangle_{\mu_t},
    \end{equation}
    which is the formula from equation \eqref{eq:chain_rule}.
\end{appendixremark}

We are now ready to prove Theorem \ref{thm:gradient}. In the following, the notation $D_x f(x,y)$ will denote the derivative of $f$ w.r.t. its first argument at $x$.

\setcounter{theoremcounter}{2}
\begin{theorem}
    Let $\mathcal{Y}=\R^n$ or $\mathcal{Y}=\cP(\R^n)$, $\rho\in\cP(\mathbb{R}^d)$ and $y\in\mathcal{Y}$ be fixed, and $\mu^*_{\theta,\rho}\in\cP(\mathbb{R}^p)$ be a fixed point of $F(\cdot,\theta):=F_\theta(\cdot,\rho)$ found by a given fixed point solver, where we assume that the fixed point solver $\R^w\to\cP(\R^p)$ is deterministic, i.e. given $\theta\in\R^w$, it finds a unique fixed point $\mu^*_{\theta,\rho}$ differentiable w.r.t. $\theta$. Let $\ell=\ell(h(\mu^*_{\theta,\rho}, \theta),y)$ be the differentiable loss, where $h:\cP(\R^d)\times\R^w\to\mathcal{Y}$ is differentiable. 
    Assume further that $F$ is differentiable, and that $\Id - D_\mu F$ is injective, where $\Id$ denotes the identity on $T_{\mu^*_{\theta,\rho}}\cP(\R^p)$, and where $D_\mu F$ is short-hand notation for $D_\mu F(\mu,\theta)$, i.e. the derivative of $F(\mu,\theta)$ w.r.t. $\mu$. 
    Then $\ell$ is differentiable w.r.t. $\theta$, and it holds:

    \begin{equation*}
        \frac{\d \ell}{\d \theta} = D_h \ell(h(\mu^*_{\theta,\rho},\theta),y)\circ\left[D_\theta h (\mu^*_{\theta,\rho},\theta) + D_{\mu^*_{\theta,\rho}}h(\mu^*_{\theta,\rho},\theta)\circ\left(\Id-D_{\mu^*_{\theta,\rho}}F(\mu^*_{\theta,\rho},\theta\right)^{-1} \circ D_\theta F(\mu^*_{\theta,\rho},\theta)    \right].
    \end{equation*}
    \end{theorem}

    \begin{proof}
        We have by the chain rule, Proposition \ref{prop:chainrule},
        \begin{equation}\label{eq:chainruleeq}
            \frac{\d \ell}{\d \theta} = D_h \ell(h(\mu^*_{\theta,\rho},\theta),y)\circ\left[D_\theta h (\mu^*_{\theta,\rho},\theta) + D_{\mu^*_{\theta,\rho}}h(\mu^*_{\theta,\rho},\theta)\circ D_\theta \mu^*_{\theta,\rho}\right].
        \end{equation}
        Furthermore, by the chain rule once more,
        \begin{equation*}
            D_\theta \mu^*_{\theta,\rho} = D_{\mu^*_{\theta,\rho}} F(\mu^*_{\theta,\rho},\theta)\circ D_\theta \mu^*_{\theta,\rho} +D_\theta F(\mu^*_{\theta,\rho},\theta).
        \end{equation*}
        Using the fact that $\Id-D_{\mu^*_{\theta,\rho}}F(\mu^*_{\theta,\rho},\theta)$ is injective, we can rewrite this as
        \begin{equation*}
            D_\theta \mu^*_{\theta,\rho} =\left(\Id - D_{\mu^*_{\theta,\rho}}F(\mu^*_{\theta,\rho}\theta)\right)^{-1}\circ D_\theta F(\mu^*_{\theta,\rho},\theta).
        \end{equation*}
        Note that by definition, $D_{\mu^*_{\theta,\rho}}F(\mu^*_{\theta,\rho},\theta)$ is a map $T_{\mu^*_{\theta,\rho}}\cP(\R^p)\mapsto T_{F(\mu^*_{\theta,\rho},\theta)}\cP(\R^p)$. However, since $\mu^*_{\theta,\rho}=F(\mu^*_{\theta,\rho},\theta)$, it can be viewed as a map $T_{\mu^*_{\theta,\rho}}\cP(\R^p)\mapsto T_{\mu^*_{\theta,\rho}}\cP(\R^p)$, hence the above expression makes sense.
        Plugging this back into equation \eqref{eq:chainruleeq} yields the result.
    \end{proof}

An important question regarding Theorem \ref{thm:gradient} is under what conditions we get differentiability. In \citep{lessel2020differentiablemapswassersteinspaces}, a sufficient condition is given for pushforward maps. The study of the differentiability of non pushforward maps is left for future works.

\begin{appendixproposition}
    Let $F:\mathcal{P}_2(\mathbb{R}^p)\to \mathcal{P}_2(\mathbb{R}^p)$ be such that $F(\mu)=f_{\#}\mu$ for a function $f:\mathbb{R}^p\to\mathbb{R}^p$ that is proper (i.e., preimages of compact sets are compact), smooth, and such that $\sup_x \norm{Df_x}<\infty$. Then $F$ is differentiable.
\end{appendixproposition}

\subsection{Loss Convergence} \label{appendix:loss_cv}

Let $F_\theta:\cP(\R^d)\times\cP(\R^d)\to \cP(\R^d)$ and $\cG_{\theta,\rho}(\mu) = \mathrm{D}\big(\mu, F_\theta(\mu,\rho)\big)$ with $\mathrm{D}$ a differentiable divergence, assumed symmetric for simplicity. Recall that we want to minimize $\mathcal{L}(\theta)=\mathbb{E}_{\rho, y}\big[\ell(h(\mu_{\theta,\rho}^*), y)\big]$ with $\ell:\mathcal{Y}\times \mathcal{Y}\to \R$ a differentiable loss function and $\mu_{\theta,\rho}^* = \argmin_\mu \ \cG_{\theta,\rho}(\mu)$.

We propose to solve this problem by alternating between a gradient descent step on $\theta$ and a Wasserstein gradient descent step on $\cG_{\theta,\rho}$. Inspired from \citep{dagreou2022framework, implicit}, we discuss the convergence of this bilevel optimization problem by analyzing the continuous process
\begin{equation*}
    \begin{cases}
        \forall \rho \in \cP(\R^d),\ \partial \mu_{t,\rho} = \div\big(\mu_{t,\rho} \gW \cG_{\theta_t,\rho}(\mu_{t,\rho})\big) \\
        \d\theta_t = -\varepsilon_t \nabla\mathcal{L}(\theta_t)\d t,
    \end{cases}    
\end{equation*}
with $t\mapsto \varepsilon_t$ a decreasing and positive curve. We analyze the convergence under several assumptions which we describe now.

\begin{appendixassumption} \label{assumption:V}
    For any $\rho\in\cP(\R^d)$, $\theta\in \R^w$, the point fixes $\mu_{\theta,\rho}^*$ are absolutely continuous with respect to the Lebesgue measure and $\mu_{\theta,\rho}^* \propto e^{-V_\rho(\cdot,\theta)}$, with $V_\rho:\R^p\times \R^w \to \R^p$ continuously differentiable and satisfying $\|\nabla_2 V_\rho(x,\theta)\| \le C_V$.
\end{appendixassumption}

The assumption of absolute continuity allows us to write $\nabla\mathcal{L}(\theta) = \mathbb{E}_{\rho,y}[\Gamma(\mu_{\theta, \rho}^*, \theta)]$ with $\Gamma(\mu,\theta)=\nabla_1 \ell\big(h(\mu), y\big) \cdot \mathrm{Cov}_\mu\big(\frac{\delta h}{\delta\mu}(\mu)(Z), \nabla_2 V_{\rho}(Z,\theta)\big)$. We also assume boundedness of $\Gamma$ and Liptschitzness with respect to $\mathrm{D}$.
\begin{appendixassumption} \label{assumption:gamma}
    For any $\mu,\nu\in\cP(\R^d)$, $\theta\in\R^w$, $\|\Gamma(\mu,\theta)-\Gamma(\nu,\theta)\| \le K_\Gamma \mathrm{D}(\mu,\nu)$ ($K_\Gamma$-Lipschitzness) and $\|\Gamma(\mu,\nu)\|\le C_\Gamma$.
\end{appendixassumption}
Now, let us define $\cF_\nu(\mu) := \mathrm{D}(\mu,\nu)$. In this notation, $\nu$ is assumed to be fixed, meaning that in expressions like $\nabla_W \cF_\nu(\mu)$, we differentiate only w.r.t. the argument $\mu$. We assume that $\cF_{\mu_{\theta,\rho}^*}$ satisfies the Polyak-Łojasiewicz (PL) inequality, which provides a sufficient condition for the convergence of the gradient flows of $\cF_{\mu_{\theta,\rho}^*}$ \citep{blanchet2018family}.

\begin{appendixassumption} \label{assumption:pl}
    For all $\mu\in\cP(\R^d)$, $\cF_{\mu_{\theta,\rho}^*}(\mu) \le C_{PL} \|\gW\cF_{\mu_{\theta,\rho}^*}(\mu)\|_{L^2(\mu)}^2$ for $C_{PL}\ge 0$ (PL inequality). Additionally, we suppose $\|\frac{\delta\cF_{\mu_t}}{\delta\mu_t} (\mu_{\theta_t,\rho}^*) \|_{L^2(\mu_{\theta_t,\rho}^*)} \le C_\cF$.
\end{appendixassumption}
Finally, we also need to control the discrepancy between the directions given by the gradient of $\cG_{\theta,\rho}$ and the gradients of $\cF_{\mu_{\theta,\rho}^*}$ towards the objective. We note that this assumption also depends on properties of the map $F_\theta$.

\begin{appendixassumption} \label{assumption:grads}
    There exists some constant $M>0$ such that for all $t\ge 0$, 
    \begin{equation}\label{eq:ass4}
        \|\gW\cG_{\theta_t,\rho}(\mu_t) - \gW\cF_{\mu_{\theta_t,\rho}^*}(\mu_t)\|_{L^2(\mu_t)}^2 \le M\varepsilon_t. 
    \end{equation}
\end{appendixassumption}
We now show the convergence of the average of the objective gradients, which is the best we can hope for since we do not make any convexity assumptions.

\begin{appendixtheorem} \label{thm:cv_loss}
    Take $\varepsilon_t = \min(1,t^{-\frac12})$, and assume Assumptions \ref{assumption:V}, \ref{assumption:gamma}, \ref{assumption:pl} and \ref{assumption:grads}. Then for $T>0$ and some constant $c>0$,
    \begin{equation*}
        \frac{1}{T}\int_0^T \|\nabla \mathcal{L}(\theta_t)\|^2\ \mathrm{d}t \le c \frac{(\log T)^2}{\sqrt{T}}.
    \end{equation*}
\end{appendixtheorem}

\begin{proof}

    First, by Assumption \ref{assumption:V}, $\mathrm{d}\mu_{\theta,\rho}^* = \frac{e^{-V_{\rho}(\cdot,\theta)}}{Z_{\theta,\rho}}\mathrm{d}Leb$ with $Z_{\theta,\rho} = \int e^{-V_{\rho}(z,\theta)}\ \mathrm{d}z$. 
    Moreover, observe that
    \begin{equation*}
        \frac{\partial \mu_{\theta,\rho}^*}{\partial \theta} = \big(-\nabla_2 V_\rho(\cdot, \theta) + \mathbb{E}_{\mu_{\theta,\rho}^*}[\nabla_2 V_\rho(Z, \theta)]\big) \mu_{\theta,\rho}^*,
    \end{equation*}
    where we write $\mu_{\theta,\rho}^*$ for the density with an abuse of notation.
    
    Then, taking the gradient of $\mathcal{L}(\theta) = \mathbb{E}_{\rho,y}\big[\ell(h(\mu_\theta^*), y)\big]$, we obtain
    \begin{equation*}
        \begin{aligned}
            \nabla\mathcal{L}(\theta) &= \mathbb{E}_{\rho,y}\left[\nabla_1 \ell\big(h(\mu_{\theta,\rho}^*), y\big) \int \frac{\delta h}{\delta\mu}(\mu_{\theta,\rho}^*)(z) \frac{\partial \mu_{\theta,\rho}^*}{\partial \theta}(z)\ \mathrm{d}z \right] \\
            &= \mathbb{E}_{\rho, y}\left[\nabla_1 \ell\big(h(\mu_{\theta,\rho}^*, y\big)\cdot \int \frac{\delta h}{\delta\mu}(\mu_{\theta,\rho}^*)(z) \left(-\frac{\nabla_2 V_{\rho}(z,\theta) e^{-V_{\rho}(z,\theta)}}{Z_{\theta,\rho}} + \frac{e^{-V_{\rho}(z,\theta)} \int \nabla_\theta V_{\rho}(y,\theta) e^{-V_{\rho}(y,\theta)}\ \mathrm{d}y}{Z_{\theta,\rho}^2}\right)\ \mathrm{d}z\right] \\
            &= \mathbb{E}_{\rho,y}\left[\nabla_1\ell\big(h(\mu_{\theta,\rho}^*, y\big) \cdot \left(-\mathbb{E}_{\mu_{\theta,\rho}^*}\left[\frac{\delta h}{\delta\mu}(\mu_{\theta,\rho}^*)(Z) \nabla_2 V_{\rho}(Z, \theta)\right] + \mathbb{E}_{\mu_{\theta,\rho}^*}\left[\nabla_2 V_{\rho}(Z,\theta)\right] \mathbb{E}_{\mu_{\theta,\rho}^*}\left[\frac{\delta h}{\delta\mu}(\mu_{\theta,\rho}^*)(Z)\right]\right)\right] \\
            &= \mathbb{E}_{\rho, y}\left[\nabla_1 \ell\big(h(\mu_{\theta,\rho}^*), y\big) \cdot \mathrm{Cov}_{\mu_{\theta,\rho}^*}\left(\frac{\delta h}{\delta\mu}(\mu_{\theta,\rho}^*)(Z), \nabla_2 V_{\rho}(Z,\theta)\right)\right].
        \end{aligned}
    \end{equation*}
    Let $\Gamma(\mu,\theta) = \nabla_1 \ell(h(\mu), y) \cdot \mathrm{Cov}_\mu\left(\frac{\delta h}{\delta\mu}(\mu)(Z), \nabla_2 V_{\rho}(Z,\theta)\right)$, then we have $\nabla \mathcal{L}(\theta) = \mathbb{E}_{\rho,y}[\Gamma(\mu_{\theta,\rho}^*, \theta)]$.

    Now, we compute $\frac{\mathrm{d}}{\mathrm{d}t}\mathcal{L}(\theta_t)$ by first applying the chain rule, using that $\frac{\mathrm{d}\theta_t}{\mathrm{d}t} = -\varepsilon_t \mathbb{E}_{\rho,y}\big[\Gamma(\mu_t, \theta_t)\big]$ and remembering that $\nabla\mathcal{L}(\theta_t) = \mathbb{E}_{\rho,y}\big[\Gamma(\mu_{\theta_t}^*, \theta_t)\big]$.
    \begin{equation} \label{eq:1st_ineq}
        \begin{aligned}
            \frac{\mathrm{d}}{\mathrm{d}t}\mathcal{L}(\theta_t) &= \left\langle \nabla \mathcal{L}(\theta_t), \frac{\mathrm{d}\theta_t}{\mathrm{d}t}\right\rangle \\
            &= -\varepsilon_t \big\langle \nabla\mathcal{L}(\theta_t), \mathbb{E}_{\rho, y}\big[\Gamma(\mu_{t,\rho}, \theta_t)\big] \big\rangle \\
            &= -\varepsilon_t \langle \nabla\mathcal{L}(\theta_t), \mathbb{E}_{\rho,y}[\Gamma(\mu_{\theta_t,\rho}^*, \theta_t)]\rangle + \varepsilon_t \langle \nabla\mathcal{L}(\theta_t), \mathbb{E}_{\rho, y}[\Gamma(\mu_{\theta_t,\rho}^*,\theta_t) - \Gamma(\mu_{t,\rho},\theta_t)]\rangle \\
            &= -\varepsilon_t \|\nabla\mathcal{L}(\theta_t)\|^2 + \varepsilon_t \mathbb{E}_{\rho, y}[\langle \nabla\mathcal{L}(\theta_t), \Gamma(\mu_{\theta_t,\rho}^*,\theta_t) - \Gamma(\mu_{t,\rho},\theta_t)\rangle].
        \end{aligned}
    \end{equation}    
    Then, applying the Cauchy-Schwartz inequality, and the inequality $ab\le \frac12(a^2 + b^2) \iff ab-a^2 \le \frac12 b^2 - \frac12 a^2$ with $a=\|\nabla\mathcal{L}(\theta_t)\|$ and $b=\mathbb{E}_{\rho,y}\big[\|\Gamma(\mu_{\theta_t,\rho}^*, \theta_t) - \Gamma(\mu_{t,\rho},\theta_t)\|\big]$, we get
    \begin{equation*}
        \begin{aligned}
            \frac{\mathrm{d}}{\d t}\mathcal{L}(\theta_t) &\le -\varepsilon_t \|\nabla\mathcal{L}(\theta_t)\|^2 + \varepsilon_t \|\nabla\mathcal{L}(\theta_t)\| \mathbb{E}_{\rho, y}\big[\|\Gamma(\mu_{\theta_t,\rho}^*,\theta_t) - \Gamma(\mu_{t,\rho},\theta_t)\|\big] \\
            &\le -\frac{\varepsilon_t}{2} \|\nabla\mathcal{L}(\theta_t)\|^2 + \frac{\varepsilon_t}{2} \mathbb{E}_{\rho, y}\big[\|\Gamma(\mu_{\theta_t,\rho}^*, \theta_t) - \Gamma(\mu_{t,\rho},\theta_t)\|\big]^2 \\
            &\le -\frac{\varepsilon_t}{2} \|\nabla\mathcal{L}(\theta_t)\|^2 + \frac{\varepsilon_t}{2} \mathbb{E}_{\rho, y}\big[\|\Gamma(\mu_{\theta_t,\rho}^*, \theta_t) - \Gamma(\mu_{t,\rho},\theta_t)\|^2\big],
        \end{aligned}
    \end{equation*}
    where we applied the Jensen inequality in the last line.

    Then, since $t\mapsto \varepsilon_t$ is non-increasing (and thus $\varepsilon_T\le \varepsilon_t$), integrating from $t=0$ to $t=T$, we get
    \begin{equation*} 
        \begin{aligned}
            \frac{\varepsilon_T}{2} \int_0^T  \|\nabla\mathcal{L}(\theta_t)\|^2\ \mathrm{d}t &\le \int_0^T \frac{\varepsilon_t}{2}  \|\nabla\mathcal{L}(\theta_t)\|^2\ \mathrm{d}t \\
            &\le \mathcal{L}(\theta_0) - \mathcal{L}(\theta_T) + \int_0^T \frac{\varepsilon_t}{2} \mathbb{E}_{\rho, y}\big[\|\Gamma(\mu_{\theta_t,\rho}^*, \theta_t) - \Gamma(\mu_{t,\rho},\theta_t)\|^2\big]\ \mathrm{d}t.
        \end{aligned}
    \end{equation*}
    And we finally obtain:
    \begin{equation} \label{eq:ineq}
        \frac{1}{T} \int_0^T \|\nabla\mathcal{L}(\theta_t)\|^2\ \mathrm{d}t \le \frac{2}{\varepsilon_T T} \big(\mathcal{L}(\theta_0) - \inf_\theta \mathcal{L}(\theta)\big) + \frac{1}{\varepsilon_T T} \int_0^T \varepsilon_t \mathbb{E}_{\rho,y}\big[\|\Gamma(\mu_{\theta_t,\rho}^*, \theta_t) - \Gamma(\mu_{t,\rho},\theta_t)\|^2\big]\ \mathrm{d}t.
    \end{equation}

    Using Assumption \ref{assumption:gamma}, $\|\Gamma(\mu_{\theta_t,\rho}^*, \theta_t) - \Gamma(\mu_{t,\rho},\theta_t)\|^2 \le K_\Gamma^2 \D(\mu_{t,\rho}, \mu_{\theta_t,\rho}^*)^2$. Let us note $\cF_1^\nu(\mu) = \D(\mu,\nu)^2$ and $\cF_2^\nu(\mu)=\D(\nu,\mu)^2$, and therefore let us bound $d(t) = \cF_1^{\mu_{\theta_t,\rho}^*}(\mu_{t,\rho})$. 

    By differentiating, we obtain
    \begin{equation*}
        \begin{aligned}
            d'(t) = \frac{\mathrm{d}}{\mathrm{d}t}\D(\mu_{t,\rho}, \mu_{\theta_t,\rho}^*) &= \underbrace{\int \frac{\delta \cF_1^{\mu_{\theta_t,\rho}^*}}{\delta\mu}(\mu_t) \frac{\partial\mu_t}{\partial t}}_\textrm{(1)} + \underbrace{\int \frac{\delta \cF_2^{\mu_t,\rho}}{\delta \mu}(\mu_{\theta_t,\rho}^*)\ \frac{\partial \mu_{\theta_t,\rho}^*}{\partial t}}_\textrm{(2)}.
        \end{aligned}
    \end{equation*}
    On one hand, using that $\frac{\partial \mu_{t,\rho}}{\partial t} = \div\big(\mu_{t,\rho} \gW\cG_{\theta_t,\rho}(\mu_{t,\rho})\big)$, and doing an integration by parts, we get
    \begin{equation*}
        \begin{aligned}
            (1) &= \int \frac{\delta\cF_1^{\mu_{\theta_t,\rho}^*}}{\partial\mu}(\mu_{t,\rho})\div\big(\mu_{t,\rho}\gW\cG_{\theta_{t,\rho}}(\mu_{t,\rho})\big) \\
            &= - \int \left\langle \gW\cF_1^{\mu_{\theta_t,\rho}^*}(\mu_{t,\rho}), \gW\cG_{\theta_{t,\rho}}(\mu_t)\right\rangle\ \mathrm{d}\mu_{t,\rho} \\
            &= -\|\gW\cF_1^{\mu_{\theta_t,\rho}^*}(\mu_{t,\rho})\|_{L^2(\mu_{t,\rho})}^2 + \langle \gW\cF_1^{\mu_{\theta_t,\rho}^*}(\mu_{t,\rho}), \gW\cF_1^{\mu_{\theta_t,\rho}^*}(\mu_{t,\rho}) - \gW\cG_{\theta_t,\rho}(\mu_{t,\rho})\rangle_{L^2(\mu_{t,\rho})}.
        \end{aligned}
    \end{equation*}
    Applying the inequality $\langle a,b\rangle \le \|a\|^2 + \frac14 \|b\|^2$ (see \emph{e.g.} \citep[Equation (33)]{vempala2019rapid}) and Assumptions \ref{assumption:pl} and \ref{assumption:grads}, we get
    \begin{equation*}
        \begin{aligned}
            (1) &\le -\frac34 \|\gW\cF_1^{\mu_{\theta_t,\rho}^*}(\mu_{t,\rho})\|_{L^2(\mu_{t,\rho})}^2 + \|\gW\cF_1^{\mu_{\theta_t,\rho}^*}(\mu_{t,\rho}) - \gW\cG_{\theta_t,\rho}(\mu_{t,\rho})\|_{L^2(\mu_{t,\rho})}^2 \\
            &\le -\frac{3}{4 C_{PL}} \cF_1^{\mu_{\theta_t,\rho}^*}(\mu_{t,\rho}) + M \varepsilon_t \\
            &= -\frac{3}{4 C_{PL}} d(t) + M\varepsilon_t.
        \end{aligned}
    \end{equation*}

    On the other hand, for (2), we have
    \begin{equation*}
        \begin{aligned}
            \frac{\partial \mu_{\theta_t,\rho}^*}{\partial t} &= \left\langle \frac{\partial \mu_{\theta_t,\rho}^*}{\partial \theta}, \frac{\mathrm{d}\theta_t}{\mathrm{d}t}\right\rangle \\
            &= \left\langle \nabla_2 V_{\rho}(\cdot,\theta_t) - \mathbb{E}_{\mu_{\theta_t,\rho}^*}[\nabla_2 V_{\rho}(Z,\theta_t)], -\varepsilon_t \mathbb{E}_{\rho,y}\big[\Gamma(\mu_{t,\rho}, \theta_t)]\right\rangle \mu_{\theta_t,\rho}^*.
        \end{aligned}
    \end{equation*}
    Thus, using that $\|\nabla_2 V_\rho(x, \theta)\|\le C_V$ by Assumption \ref{assumption:V},  $\|\Gamma(\mu,\nu)\|\le C_\Gamma$ by Assumption \ref{assumption:gamma} and $\|\frac{\delta \cF_2^{\mu_{t,\rho}}}{\delta\mu}(\mu_{\theta_t,\rho}^*)\|_{L^2(\mu_{\theta_t,\rho}^*)} \le C_\cF$ by Assumption \ref{assumption:pl}, we get
    \begin{equation} \label{eq:bound_(2)}
        \begin{aligned}
            (2) &= -\varepsilon_t \left\langle \int \frac{\delta \cF_2}{\delta\mu}(\mu_{\theta_t,\rho}^*)(x) \cdot (\nabla_2 V_\rho(x,\theta_t) - \mathbb{E}_{\mu_{\theta_t,\rho}^*}[\nabla_2 V_\rho(Z,\theta)])\ \mathrm{d}\mu_{\theta_t,\rho}^*(x), \mathbb{E}_{\rho,y}[\Gamma(\mu_{t,\rho},\theta_t)]\right\rangle \\
            &\le \varepsilon_t \|\mathbb{E}_{\rho,y}[\Gamma(\mu_{t,\rho},\theta_t)]\| \int \left|\frac{\delta \cF_2}{\delta \mu}(\mu_{\theta_t^*,\rho})(x)\right|\cdot \|\nabla_2 V_\rho(x,\theta_t) - \mathbb{E}_{\mu_{\theta_t,\rho}^*}[\nabla_2 V_\rho(Z,\theta_t)]\|\ \mathrm{d}\mu_{\theta_t,\rho}^*(x) \\
            &\le 2 \varepsilon_t C_\Gamma C_V \int \left|\frac{\delta\cF_2}{\delta\mu}(\mu_{\theta_t,\rho}^*)\right|\ \mathrm{d}\mu_{\theta_t,\rho}^* \\
            &\le 2 \varepsilon_t C_\Gamma C_V C_\cF.
        \end{aligned}
    \end{equation}

    Putting everything together, we obtain
    \begin{equation*}
        d'(t) \le -\frac{3}{4 C_{PL}} d(t) + (M+2C_\Gamma C_V C_\cF)\varepsilon_t.
    \end{equation*}
    Now, note $\Tilde{C} = M+2C_\Gamma C_V C_\cF$, and let's apply the Grönwall lemma \citep{dragomir2003some}:
    \begin{equation*}
        d(t) \le d(0) e^{-\frac{3}{4C_{PL}} t} + \Tilde{C} \int_0^t \varepsilon_s e^{\frac{3}{4 C_{PL}} (s-t)}\ \mathrm{d}s.
    \end{equation*}
    Plugging it into \eqref{eq:ineq}, we get
    \begin{equation*}
        \frac{1}{T} \int_0^T \|\nabla\mathcal{L}(\theta_t)\|^2\ \mathrm{d}t \le \frac{2}{\varepsilon_T T} \big(\mathcal{L}(\theta_0) - \inf_\theta \mathcal{L}(\theta)\big) + \frac{K_\Gamma^2 \mathbb{E}_{\rho,y}\big[d(0)]}{\varepsilon_T T} \int_0^T \varepsilon_t e^{-\frac{3}{4 C_PL} t} \ \mathrm{d}t + \frac{K_\Gamma^2 \Tilde{C}}{\varepsilon_T T} \int_0^T \varepsilon_t \int_0^t \varepsilon_s e^{\frac{3}{4 C_{PL}} (s-t)}\ \d s\d t .
    \end{equation*}

    Recall that $\varepsilon_t = \min(1, t^{-\frac12})$, thus $T\varepsilon_T = \sqrt{T}$. The two first terms converge, hence they lie in $O(\sqrt{T})$. For the third one, we use the same trick as in \citep[Proof of Proposition 4.5]{implicit}. Let $T_0 \ge 2$ such that $\frac12 T_0 \ge \frac{4 C_{PL}}{3} \log(T_0)$. For $t\ge T_0$, let $\alpha(t) = t - \frac{4 C_{PL}}{3} \log(t)$. Notice that $\alpha(t) = \frac{t}{2} + \frac{t}{2} - \frac{4 C_{PL}}{3} \ge \frac{t}{2}$ for all $t\ge T_0$. Then, observe that for $t\ge T_0$,
    \begin{equation*}
        \begin{aligned}
            \int_0^t \varepsilon_s e^{\frac{3}{4 C_{PL}}(s-t)}\ \mathrm{d}s &= \int_0^{\alpha(t)} \varepsilon_s e^{\frac{3}{4 C_{PL}}(s-t)}\ \mathrm{d}s + \int_{\alpha(t)}^t \varepsilon_s e^{\frac{3}{4 C_{PL}}(s-t)}\ \mathrm{d}s \\
            &\le \varepsilon_0 e^{-\frac{3}{4 C_{PL}}} \int_0^{\alpha(t)} e^{\frac{3}{4 C_{PL}} s}\mathrm{d}s + \big(t-\alpha(t)\big) \varepsilon_{\alpha(t)} \\
            &\le \frac{4 C_{PL}\varepsilon_0}{3} e^{\frac{3}{4C_{PL}}(\alpha(t) - t)} + \frac{4 C_{PL}}{3} \log(t) \varepsilon_\alpha(t) \\
            &\le \frac{4 C_{PL}\varepsilon_0}{3 t} + \frac{4 C_{PL}}{3} \log(t) \varepsilon_{t/2}.
        \end{aligned}
    \end{equation*}
    Thus,
    \begin{equation*}
        \begin{aligned}
            \int_0^T \varepsilon_t \int_0^t \varepsilon_s e^{\frac{3}{4 C_{PL}}(s-t)}\ \mathrm{d}s\mathrm{d}t &\le \int_0^{T_0} \varepsilon_t \varepsilon_0 T_0\ \mathrm{d}t + \int_{T_0}^T \frac{4 C_{PL} \varepsilon_t \varepsilon_0}{3 t}\ \mathrm{d}t + \int_{T_0}^T \frac{4 C_{PL} \varepsilon_t \varepsilon_{t/2} \log(t)}{3}\ \mathrm{d}t \\
            &\le \varepsilon_0 T_0^2 + \frac{4 C_{PL}}{3} \big(\log T - \log T_0\big) + \frac{4 C_{PL}}{3} \log T \cdot \int_{T_0}^T \frac{2}{t}\ \mathrm{d}t \\
            &= \mathcal{O}\big((\log T)^2\big).
        \end{aligned}
    \end{equation*}
    Putting everything together, we obtain that there exists a constant $c>0$ such that
    \begin{equation*}
        \frac{1}{T} \int_0^T \|\nabla\mathcal{L}(\theta_t)\|^2\ \mathrm{d}t \le c \frac{(\log T)^2}{\sqrt{T}}.
    \end{equation*}

\end{proof}

We note that the MMD that we use has only been proven to fulfill a PL inequality with a different power \citep{mmdflow}. We provide this result as it gives a first step towards deriving guarantees of the method, e.g. for other divergence which would satisfy these assumptions. We leave for future work deriving the PL inequality for our MMD, or deriving alternative divergences which satisfy these assumptions. In Section \ref{sec:empiricalverification}, we show empirically that Assumptions \ref{assumption:pl} and \ref{assumption:grads}, as well as Theorem \ref{thm:cv_loss}, hold true in practice. Note that a weaker version of Theorem \ref{thm:cv_loss}, which does not give a specific rate of convergence for the average gradient, holds true for any gradient flow:

\begin{appendixtheorem}
    For a gradient flow $t\mapsto \theta_t$ on a non-negative loss $\cL(\theta)$ with bounded gradients, it holds that
    \begin{equation*}
        \frac{1}{T}\int_0^T\norm{\nabla \cL(\theta_t)}^2\d t\xrightarrow{T\to\infty} 0.
    \end{equation*}
\end{appendixtheorem}
\begin{proof}
    Note that
    \begin{equation*}
        \partial_t \cL(\theta_t)=\langle \nabla_\theta \cL(\theta_t),\partial_t\theta_t\rangle=\langle \nabla_\theta \cL(\theta_t),-\nabla_\theta \cL(\theta_t)\rangle=-\norm{\nabla_\theta \cL(\theta_t)}^2\le 0.
    \end{equation*}
    Since $\cL(\theta_t)\ge 0$, this implies that $\norm{\nabla_\theta \cL(\theta_t)}^2\to 0$ as $t\to\infty$. Now for any $\epsilon>0$, we can find a $t_\epsilon>0$ such that for any $t>t_\epsilon$, $\norm{\nabla_\theta \cL(\theta_t)}^2\le \epsilon$. This means that for all $T\ge t_\epsilon$,
    \begin{equation*}
        \frac{1}{T}\int_0^T\norm{\nabla \cL(\theta_t)}^2\d t \le \frac{1}{T}(Mt_\epsilon+(T-t_\epsilon)\epsilon)\le \frac{Mt_\epsilon}{T}+\epsilon\xrightarrow{T\to\infty}\epsilon
    \end{equation*}
    for some constant $M>0$, and as this holds for any $\epsilon>0$, the claim follows.
\end{proof}

We now discuss in more details the MMD case.

\textbf{MMD case.} Let $\D=\frac12\mathrm{MMD}^2$ with $k$ a kernel satisfying $k(x,x)\le 1$. We first notice that it is currently unknown whether the MMD satisfies the PL inequality from Assumption \ref{assumption:pl}. However, it was shown in \citep[Proposition 7]{mmdflow}, that under some conditions, $\cF_{\mu_{\theta,\rho}^*}(\mu)^2 \le \frac{C}{4} \|\gW\cF_{\mu_{\theta,\rho}^*}(\mu)\|_{L^2(\mu)}^2$.

For this, we first need to recall the definition of the weighted negative Sobolev distance, e.g. \citep[Definition 5]{mmdflow}.
\begin{appendixdefinition}
    Let $\mu\in\cP(\R^d)$ and denote by $\|\cdot\|_{\dot{H}(\mu)}$ its corresponding weighted Sobolev semi-norm, defined for all differentiable $f\in L^2(\mu)$ as $\|f\|_{\dot{H}(\mu)}^2 = \int \|\nabla f(x)\|^2\ \mathrm{d}\mu(x)$. The weighted negative Sobolev distance between $p,q\in\cP(\R^d)$ is defined as 
    \begin{equation*}
        \|p-q\|_{\dot{H}^{-1}(\mu)} = \sup_{f\in L^2(\mu),\ \|f\|_{\dot{H}(\mu)}\le 1}\ \left|\int f\mathrm{d}p - \int f\mathrm{d}q\right|.
    \end{equation*}
\end{appendixdefinition}
This allows us to derive the following lemma.
\begin{appendixlemma}[PL Inequality] \label{lemma:pl_mmd}
    Assume $k$ is continuously differentiable with $L$-Lipschitz gradient, i.e. for all $x,y,x',y'\in\R^d$, $\|\nabla k(x,x') - \nabla k(y,y')\|\le L\big(\|x-x'\| + \|y-y'\|\big)$. Then, if for all $t\ge 0$, $\|\mu_{\theta_t,\rho}^* - \mu_{t,\rho}\|_{\dot{H}^{-1}(\mu_t)} \le C$,
    \begin{equation*}
        \cF_{\mu_{\theta,\rho}^*}(\mu)^2 \le \frac{C}{4} \|\gW\cF_{\mu_{\theta,\rho}^*}(\mu)\|_{L^2(\mu)}^2.
    \end{equation*}
\end{appendixlemma}
\begin{proof}
    See \citep[Proof of Proposition 7]{mmdflow}.
\end{proof}

We now discuss what changes in the MMD case compared to the setting developed above. First, we will also assume Assumptions \ref{assumption:V}, \ref{assumption:gamma} and \ref{assumption:grads}. Assumption \ref{assumption:pl} is replaced by the assumptions of Lemma \ref{lemma:pl_mmd} to get the PL inequality for the MMD.

Under Assumption \ref{assumption:V}, we recover \eqref{eq:ineq}, i.e.
\begin{equation*}
    \frac{1}{T} \int_0^T \|\nabla\mathcal{L}(\theta_t)\|^2\ \mathrm{d}t \le \frac{2}{\varepsilon_T T} \big(\mathcal{L}(\theta_0) - \inf_\theta \mathcal{L}(\theta)\big) + \frac{1}{\varepsilon_T T} \int_0^T \varepsilon_t \mathbb{E}_{\rho,y}\big[\|\Gamma(\mu_{\theta_t,\rho}^*, \theta_t) - \Gamma(\mu_{t,\rho},\theta_t)\|^2\big]\ \mathrm{d}t.
\end{equation*}
Remember also that $\Gamma$ can be written as
\begin{equation*}
    \Gamma(\mu,\theta) = \nabla_1 \ell\big(h(\mu),y\big) \cdot \mathrm{Cov}_\mu\left(\frac{\delta h}{\delta \mu}(\mu)(Z), \nabla_2 V(Z,\theta)\right).
\end{equation*}
First, we need to assume $\|\Gamma(\mu,\theta)-\Gamma(\nu,\theta)\| \le \frac{K_\Gamma}{2} \mathrm{MMD}(\mu,\nu)$. Thus, we need to upper bound $d:t\mapsto \frac12 \mathrm{MMD}^2(\mu_t, \mu_{\theta_t,\rho}^*)$.

Let $d(t) = \frac12 \mathrm{MMD}^2(\mu_t, \mu_{\theta_t,\rho}^*)$. By differentiating, we also get
\begin{equation*}
    \begin{aligned}
        d'(t) = \frac{\mathrm{d}}{\mathrm{d}t}\D(\mu_t, \mu_{\theta_t,\rho}^*) &= \underbrace{\int \frac{\delta \cF_1^{\mu_{\theta_t,\rho}^*}}{\delta\mu}(\mu_t) \frac{\partial\mu_t}{\partial t}}_\textrm{(1)} + \underbrace{\int \frac{\delta \cF_2^{\mu_t,\rho}}{\delta \mu}(\mu_{\theta_t,\rho}^*)\ \frac{\partial \mu_{\theta_t,\rho}^*}{\partial t}}_\textrm{(2)}.
    \end{aligned}
\end{equation*}
By following the same computations as in Theorem \ref{thm:cv_loss}, assuming that Assumption \ref{assumption:grads} holds and using Lemma \ref{lemma:pl_mmd}, we have
\begin{equation}
    (1) \le - \frac{3}{C_{PL}} d(t)^2 + M \varepsilon_t.
\end{equation}
For (2), we have by using \eqref{eq:bound_(2)},
\begin{equation} 
    \begin{aligned}
        (2) &= -\varepsilon_t \left\langle \int \frac{\delta \cF_2}{\delta\mu}(\mu_{\theta_t,\rho}^*)(x) \cdot (\nabla_2 V_\rho(x,\theta_t) - \mathbb{E}_{\mu_{\theta_t,\rho}^*}[\nabla_2 V_\rho(Z,\theta)])\ \mathrm{d}\mu_{\theta_t,\rho}^*(x), \mathbb{E}_{\rho,y}[\Gamma(\mu_{t,\rho},\theta_t)]\right\rangle \\
        &\le \varepsilon_t \|\mathbb{E}_{\rho,y}[\Gamma(\mu_{t,\rho},\theta_t)]\| \int \left|\frac{\delta \cF_2}{\delta \mu}(\mu_{\theta_t^*,\rho})(x)\right|\cdot \|\nabla_2 V_\rho(x,\theta_t) - \mathbb{E}_{\mu_{\theta_t,\rho}^*}[\nabla_2 V_\rho(Z,\theta_t)]\|\ \mathrm{d}\mu_{\theta_t,\rho}^*(x) \\
        &\le 2 \varepsilon_t C_\Gamma C_V \int \left|\frac{\delta\cF_2}{\delta\mu}(\mu_{\theta_t,\rho}^*)\right|\ \mathrm{d}\mu_{\theta_t,\rho}^* \\
        &= 2 \varepsilon_t C_\Gamma C_V \int \left|\int k(x,x')\ \mathrm{d}\mu_{\theta_t,\rho}^*(x') - \int k(x,y)\ \mathrm{d}\mu_t(y)\right|\ \mathrm{d}\mu_{\theta_t,\rho}^* \\
        &\le 2 \varepsilon_t C_\Gamma C_V \cdot \frac12\mathrm{MMD}(\mu_{\theta_t,\rho}^*, \mu_t) \\
        &= 2 \varepsilon_t C_\Gamma C_V d(t) \\
        &\le 2\varepsilon_t C_\Gamma C_V \big(1+d(t)^2\big),
    \end{aligned}
\end{equation}
where we used that for $\mu_{\theta_t,\rho}$-almost every $x$, $k(x,\cdot) \in \mathcal{H}$ and $\|k(x,\cdot)\|_\mathcal{H} = k(x,x)^\frac12 \le 1$ and taking the supremum.

Putting everything together, we get,
\begin{equation} \label{eq:gronwall}
    d'(t) \le \left(2C_\Gamma C_V - \frac{3}{C_{PL}}\right) d(t)^2 + (M  + 2C_\Gamma C_V)\varepsilon_t.
\end{equation}
Here, we observe that we need to assume $2C_\Gamma C_V - \frac{3}{C_{PL}}\le 0$ for the rate to be converging.

To apply the Grönwall lemma to \eqref{eq:gronwall}, we need to solve the following Riccati equation $y'(t) = -ay^2(t) + b\varepsilon(t)$\footnote{ \url{math.stackexchange.com/questions/4773818/generalized-gronwall-inequality-covering-many-different-applications}} for $a=\frac{3}{C_{PL}} -2 C_\Gamma C_V\ge 0$, $b=M+2C_\Gamma C_V$, which however does not have a simple solution for $b\neq 0$ \citep{zaitsev2002handbook}.

For $b=0$, we would have
\begin{equation*}
    d(t) \le \frac{d(0)}{1 + a d(0) t},
\end{equation*}
and 
\begin{equation*}
    \begin{aligned}
        \int_0^T \varepsilon_t d(t)\ \mathrm{d}t &\le d(0) \int_0^T \frac{1}{1+ad(0) t}\ \mathrm{d}t \\
        &= d(0) \cdot \frac{\log(1 + a d(0) T)}{a d(0)} \\
        &= \mathcal{O}(\log T).
    \end{aligned}
\end{equation*}
Thus, the rate would be in $\mathcal{O}\big(\log T / \sqrt{T}\big)$.

If we bound $\varepsilon_t$ by $1$, then we can show that
\begin{equation*}
    d(t) \le \frac{\sqrt{ba}}{a} \cdot \frac{1}{\tanh(\sqrt{ba} t + K)},
\end{equation*}
for $K=\mathrm{arctanh}\left(\frac{\sqrt{ba}}{a d(0)}\right)$. Observing that
\begin{equation*}
    \begin{aligned}
        \int_0^T \varepsilon_t \frac{1}{\tanh(\sqrt{ba}t + K)} \ \mathrm{d}t &\le \int_0^T \frac{1}{\tanh(\sqrt{ba}t + K)}\ \mathrm{d}t \\
        &= \frac{1}{\sqrt{ba}}\big(\log\big(\sinh(\sqrt{ba} T + K)\big) - \log\big(\sinh (K)\big)\big) \\
        &= \mathcal{O}(T).
    \end{aligned}
\end{equation*}
However, this bound is too big, and will diverge.

\subsection{EI Property}

\setcounter{theoremcounter}{4}
\begin{proposition}
    Let $F^{\textnormal{bil}}:\mathbb{R}^{N\times p}\times\mathbb{R}^{M\times d}\to\mathbb{R}^{N\times p}$, $(\Z,\X)\mapsto F^{\textnormal{bil}}(\Z,\X)$. Then $F^{\textnormal{bil}}$ is a bilinear map that fulfills the EI property
    if and only if $F^{\textnormal{bil}}$ is of the form
    \begin{equation}
        F^{\textnormal{bil}}(\Z,\X)=\Z\alpha \X^\top 1_M+1_N 1_N^\top \Z\beta \X^\top 1_M\label{eq:9}
    \end{equation}
    for some tensors $\alpha,\beta\in\mathbb{R}^{p\times p\times d}$, where $1_N,1_M$ denote the vectors with ones in $\mathbb{R}^N$ and $\mathbb{R}^M$.
\end{proposition}
\begin{proof}
    $F^{\textnormal{bil}}$ is bilinear if it can be written as
    \begin{equation*}
        F^{\textnormal{bil}}(Z,X)_{ij}=\sum_{klmn}F^{\textnormal{bil}}_{ijklmn}Z_{kl}X_{mn}
    \end{equation*}
    for some $F^{\textnormal{bil}}_{ijklmn}\in\mathbb{R}$.
    By definition, $F^{\textnormal{bil}}$ is equivariant in $Z$ and invariant in $X$ if and only if for any $\sigma_N$, $\sigma_M$, it holds that
    \begin{equation*}
        F^{\textnormal{bil}}(Z,X)_{\sigma_N(i)j}=\sum_{klmn}F^{\textnormal{bil}}_{ijklmn}Z_{\sigma_N(k)l}X_{\sigma_M(m)n},
    \end{equation*}
    which is equivalent to
    \begin{equation*}
        F^{\textnormal{bil}}(Z,X)_{ij}=\sum_{klmn}F^{\textnormal{bil}}_{\sigma_N^{-1}(i)j\sigma_N^{-1}(k)l\sigma_M^{-1}(m)n}Z_{kl}X_{mn}.
    \end{equation*}
    Hence, $F^{\textnormal{bil}}$ has the desired properties if and only if for any $\sigma_N$ and $\sigma_M$, it holds that for all indices $i,j,k,l,m,n$,
    \begin{equation*}
        F^{\textnormal{bil}}_{ijklmn}=F^{\textnormal{bil}}_{\sigma_N(i)j\sigma_N(k)l\sigma(m)n}.
    \end{equation*}
    (Note we can replace the inverse permutations by permutations since the set of all inverse permutations is the same as the set of all permutations.)
    This is the case if and only if the following two conditions hold:
    \begin{itemize}
        \item $F^{\textnormal{bil}}_{ijklmn}=F^{\textnormal{bil}}_{\sigma_N(i)j\sigma_N(k)lmn}$. For fixed $i,k$, consider two cases:
        \begin{itemize}
            \item if $i\neq k$, then for all $i'\neq k'$, we can find a $\sigma_N$ s.t. $\sigma_N(i)=i'$ and $\sigma_N(k)=k'$, which implies $F^{\textnormal{bil}}_{ijklmn}=F^{\textnormal{bil}}_{i'jk'lmn}$
            \item if $i=k$, then $\sigma_N(i)=\sigma_N(k)$ for any $\sigma_N$, hence for all $i'$, $F^{\textnormal{bil}}_{ijilmn}=F^{\textnormal{bil}}_{i'ji'lmn}$ (since we can always find $\sigma_N$ such that $\sigma_N(i)=i'$).
        \end{itemize}
        This means we can write $F^{\textnormal{bil}}_{ijklmn}=\tilde{\alpha}_{ljmn}\delta_{ik}+\tilde{\beta}_{ljmn}(1-\delta_{ik})$ for some $\tilde{\alpha}$ and $\tilde{\beta}$ in $\mathbb{R}^{p\times p\times d}$, where $\delta_{ik}$ is the Kronecker delta;
        \item $F^{\textnormal{bil}}_{ijklmn}=F^{\textnormal{bil}}_{ijkl\sigma_M(m)n}$, which is equivalent to $F^{\textnormal{bil}}_{ijklmn}=F^{\textnormal{bil}}_{ijklm'n}$ for all $m,m'$. This means we can write $F^{\textnormal{bil}}_{ijklmn}=F^{\textnormal{bil}}_{ijkln}$.
    \end{itemize}
    Putting these two things together, this means $F^{\textnormal{bil}}$ has the desired properties if and only if it can be written as
    \begin{equation*}
        F^{\textnormal{bil}}_{ijklmn}=\tilde{\alpha}_{ljn}\delta_{ik}+\tilde{\beta}_{ljn}(1-\delta_{ik})
    \end{equation*}
    for some $\tilde{\alpha}$ and $\tilde{\beta}$ in $\mathbb{R}^{p\times p\times d}$. If we let $\alpha:=\tilde{\alpha}-\tilde{\beta}$ and $\beta:=\tilde{\beta}$, this is equivalent to
    \begin{equation*}
        F^{\textnormal{bil}}_{ijklmn}=\alpha_{ljn}\delta_{ik}+\beta_{ljn}.
    \end{equation*}
    By writing out equation (\ref{eq:9}) with indices, one can verify that this is equivalent.
\end{proof}

In order to proof that MultiHead attention layers fulfill the EI property, we prove two easy, but helpful lemmas.
\addtocounter{theoremcounter}{3}
\begin{appendixlemma}\label{lem:entrywise}
    Let $F:\mathbb{R}^{N\times p}\times \R^{M\times d}\to\mathbb{R}^{N\times p}$, such that it can be written as
    \begin{equation*}
        F(\Z,\X)=F(\Z)=\begin{bmatrix}
                f(z_1)\\
                f(z_2)\\
                \vdots\\
                f(z_N)
            \end{bmatrix}
    \end{equation*}
    for some function $f:\mathbb{R}^p\to\mathbb{R}^p$. Then $F$ fulfills the EI property.
\end{appendixlemma}
\begin{proof}
    $F$ is trivially invariant in $\X$, as it does not depend on $\X$. Furthermore, it is clear that $F$ is equivariant in $\Z$, as the $i^{th}$ entry of $F$ only depends on $z_i$.
\end{proof}

\begin{appendixlemma}\label{lem:composition}
    The EI property is preserved under compositions of functions that each fulfill the EI property.
\end{appendixlemma}
\begin{proof}
    This follows immediately from the fact that both row-equivariance as well as -invariance are preserved under compositions.
\end{proof}

\setcounter{theoremcounter}{5}
\begin{proposition}
    Denote by MultiHead(tgt, src) multi-head attention \citep{transformer} between a target and source sequence. Then MultiHead$(\Z, \Z)$ and MultiHead$(\Z, \X)$ both fulfill the EI property.
\end{proposition}
\begin{proof}
Recall that single-head cross-attention between $\X$ and $\Z$ is defined as
\begin{equation*}
    \textnormal{Att}(\Z,\X)_i=\text{LayerNorm}\left(\frac{\sum_j e^{\tau\langle Qz_i,Kx_j\rangle}Vx_j}{\sum_j e^{\tau \langle Qz_i,Kx_j\rangle}}\right)
    \end{equation*}
for some matrices $Q\in\R^{h\times p},\ K\in\R^{h\times d},\ V\in\R^{p\times d}$, where $h\in\R$ (typically, $h=p=d$), and some $\tau\in\R$; here, the $z_i$ denote the rows of $\Z$, and the $x_i$ denote the rows of $\X$.
The term $\frac{\sum_j e^{\tau\langle Qz_i,Kx_j\rangle}Vx_j}{\sum_j e^{\tau \langle Qz_i,Kx_j\rangle}}$ only depends on $z_i$ but none of the other $z$'s, and is clearly invariant under permutations of the $x_i$, which makes it fulfill the EI property. Layer normalization fulfills the EI property by lemma \ref{lem:entrywise}, and thus by lemma \ref{lem:composition} single-head cross-attention does as well. Again by lemma \ref{lem:composition}, the result holds for multi-head cross-attention as well. Similarly, we can show that multi-head self-attention on $\Z$ fulfills the EI property, since it is trivially invariant under permutations of $\X$ (as it does not depend on $\X$).
\end{proof}

\section{Training Details}\label{sec:training_details}

\subsection{Architecture}\label{sec:appendixarchitecture}
In this section, we provide additional details about the architecture. The encoder blocks are stacked multi-head attention layers \citep{transformer}, as can be seen in Figure \ref{fig:mha}.
\begin{figure}[h]
\centering
\begin{tikzpicture}
\definecolor{myyellow}{rgb}{1.0, 1.0, 0.6}
\definecolor{mygreen}{rgb}{.0, 1.0, .0}
    \sffamily
    \tikzset{
        myarrow/.style={
            -{Latex[length=2.5mm]},
            thick,         
            line width=1.5pt
        }
    }
    \tikzstyle{block} = [rectangle, draw, line width=1.5pt, text centered, minimum height=1.2cm, minimum width=3cm, rounded corners=5pt]
    \tikzstyle{addnorm} = [rectangle, draw, fill=myyellow!80, line width=1.5pt, text centered, minimum height=0.8cm, minimum width=3.5cm, rounded corners=5pt]
    \tikzstyle{attention} = [rectangle, draw, fill=orange!40, line width=1.5pt, text centered, minimum height=1.2cm, minimum width=3.5cm, rounded corners=5pt]
    \tikzstyle{feedforward} = [rectangle, draw, fill=blue!70!green!30, line width=1.5pt, text centered, minimum height=1.2cm, minimum width=3.5cm, rounded corners=5pt]
    \tikzstyle{residual} = [myarrow]  

    \draw[line width=1.5pt, rounded corners=10pt, thick, fill=gray!20] (-2.5, -5) rectangle (2.5, .75);

    \node[attention] (mha) at (0, -3.9) {Multi-Head Attention};

    \draw[myarrow] (0, -5.25) -- (mha);

    \node at (-3, -2) {\large{N$\times$}};

    \node[addnorm] (addnorm1) at (0, -2.8) {Add \& Norm};

    \node[feedforward] (ffn) at (0, -1.1) {FFN};

    \node[addnorm] (addnorm2) at (0, 0) {Add \& Norm};

    \draw[myarrow] (0, 0.4) -- (0, 1.1);

    \draw[-, thick, line width=1.5pt] (mha) -- (addnorm1);
    \draw[myarrow] (addnorm1) -- (ffn);
    \draw[-, thick, line width=1.5pt] (ffn) -- (addnorm2);

    \draw[residual] (0, -4.9) -| (-2.1, -4.9) -- (-2.1, -2.8) -- (addnorm1.west);

    \draw[residual] (0, -2.1) -- (-2.1, -2.1) -- (-2.1, 0.0) -- (addnorm2.west);

\end{tikzpicture}
\caption{Our encoders are stacked multi-head attention layers \citep{transformer}.}
\label{fig:mha}
\end{figure}
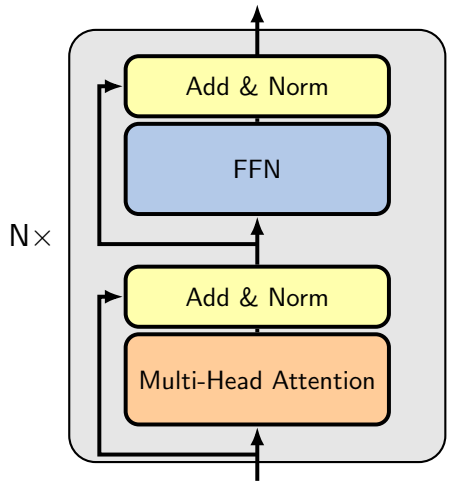

\textbf{Layer Sizes.} The hidden dimensions of the first two feedforward networks (the ones before the bilinear layer) are the same as the dimension of the bilinear layer (16). The hidden dimensions of the feedforward networks right after the bilinear layer is the same as the encoder dimension (128). The feedforward network at the end is a typical transformer feedforward layer, where the hidden dimension is four times that of the encoder dimension (i.e. 512). We would like to highlight here that our head size (of four), which we tuned extensively, is extremely small, compared to what is common in the literature. Further investigation into what such a small head size means for the model and its training is outstanding.

\textbf{Invertible Network.} The invertible network used in the point cloud completion pipeline takes the form described in the paper, i.e. for the padded $\tilde{\Z}$ (which is $\X$ with the free particles, padded with zeros to match the encoder dimension):
\begin{align*}
    \Z=q(\tilde{\Z})&=[\tilde{\Z}_1, \tilde{\Z}_2 e^{\phi(\tilde{\Z}_1)}+\psi(\tilde{\Z}_1)]\\
    q^{-1}(\Z)&=[\Z_1, (\Z_2-\psi(\Z_1))e^{-\phi(\Z_1)}],
\end{align*}
where $\phi$ and $\psi$ are two-layer neural networks, and $\tilde{\Z}=[\tilde{\Z}_1,\tilde{\Z}_2]$ is a partition of $\X$ into two equal-sized chunks along the particle dimension. Here, both the input and output dimensions, as well as the hidden dimensions of $\phi$ and $\psi$ are equal to 64, i.e. half the encoder dimension. It can be readily seen that these layers are indeed invertible, as we have
\begin{equation*}
    q^{-1}(q(\Z))=[\Z_1, (\Z_2 e^{\phi(\Z_1)}+\psi(\Z_1)-\psi(\Z_1))e^{-\phi(\Z_1)}]=[\Z_1,\Z_2]=\Z.
\end{equation*}

We have tried various different methods for the invertible layer. This includes other invertible neural networks, but also non-parametric methods, such as an orthogonal matrix, where for the inverse we use the pseudo-inverse; using the real Fourier transform the map from a lower- to a higher-dimensional space; and also just using a regular, non-invertible, trainable linear layer, and simply setting the output particles equal to the input particles manually once a fixed point has been found. Except for the latter, all of these worked comparably well, but the one we opted for seemed to have a slight edge in practice. Since all methods that are provably invertible are equivalent on the input points, as they will faithfully recover the input points, the only difference in performance can stem from how they process points that are not contained in the partial input point cloud. Possibly, the architecture we opted for is naturally well suited for this type of task; this is supported by the fact that it originated in the invertible normalizing flow literature \citep{nice}, which also deals with mapping distributions to distributions.

\textbf{Masking.} Since different samples have different numbers of particles, our batches are padded with zeros. To account for this, all intermediate states of the network are masked accordingly. For point cloud completion, we fix the particles that correspond to the input. We do this by setting the outputs of all encoders at the respective positions equal to the inputs to the network at those positions, and also by masking out the gradient in the MMD flow with zeros for the fixed particles. This ensures that the particles remain unchanged, and are converted to the original low-dimensional particles by the inverse invertible layer.

\subsection{Training}
In this section, we provide some more details regarding the training procedure. 

\textbf{Approximating the Wasserstein Gradient.} In practice, we are not directly using the formula for the Wasserstein gradient we derived, but instead rely on autodifferentiation in \verb|PyTorch|. This is both for simplicity, as well as to account for the fact that, as discussed in Section \ref{sec:pushforward}, our network architecture cannot strictly be written as a pushforward operator on the particles, hence the Wasserstein gradient formula is not exact\footnote{We found autodifferentiation to have slightly better performance than using the formula for the Wasserstein gradient, which can likely be accounted for by the fact that our network is only ``almost'' a pushforward operator.}. To account for the fact that in the discrete measure corresponding to $\Z$, each particle is assigned a mass of $1/N_\Z$, where $N_\Z$ denotes the number of particles in $\Z$, we rescale the gradient obtained from autodifferentiation by the number of particles in $\Z$, similar to \citep[Definition 2.2]{chizat2018globalconvergencegradientdescent}.

\textbf{Phantom Gradient.} As described in the paper, we rely on the phantom gradient to compute the gradients in the backward pass of the DDEQ. For a definition of the phantom gradient, we refer the reader to the original paper \citep{phantomgradient}. We use the \verb|DEQSliced| implementation from the \verb|torchdeq| library \citep{torchdeq}, and use just a single gradient step for the phantom gradient, as well as a single state (where increasing the number of states corresponds to sampling multiple fixed points from the forward solver, and using the best one).

\section{Additional Experiments}\label{sec:add_expes}

\subsection{Empirical Verification of Theorem \ref{thm:cv_loss}}\label{sec:empiricalverification}
In this section, we empirically verify that Theorem \ref{thm:cv_loss}, as well as Assumptions \ref{assumption:pl} and \ref{assumption:grads} hold true in practice across all four datasets.

\textbf{Assumption \ref{assumption:pl}.} We plot
\begin{equation*}
    \frac{\cF_{\mu^*_{\theta_t,\rho}}(\mu)}{\norm{\nabla_W \cF_{\mu^*_{\theta_t,\rho}}(\mu)}^2_{L^2(\mu)}}
\end{equation*}
averaged across $8$ randomly created $\mu\in\cP(\R^d)$ over the course of training. The plots can be seen in Figure \ref{fig:pl}. We see that across all four datasets, the above quantity is bounded, which verifies Assumption \ref{assumption:pl}.

\begin{figure}[h]
    \centering
    \begin{subfigure}{0.24\textwidth}
        \includegraphics[width=\linewidth]{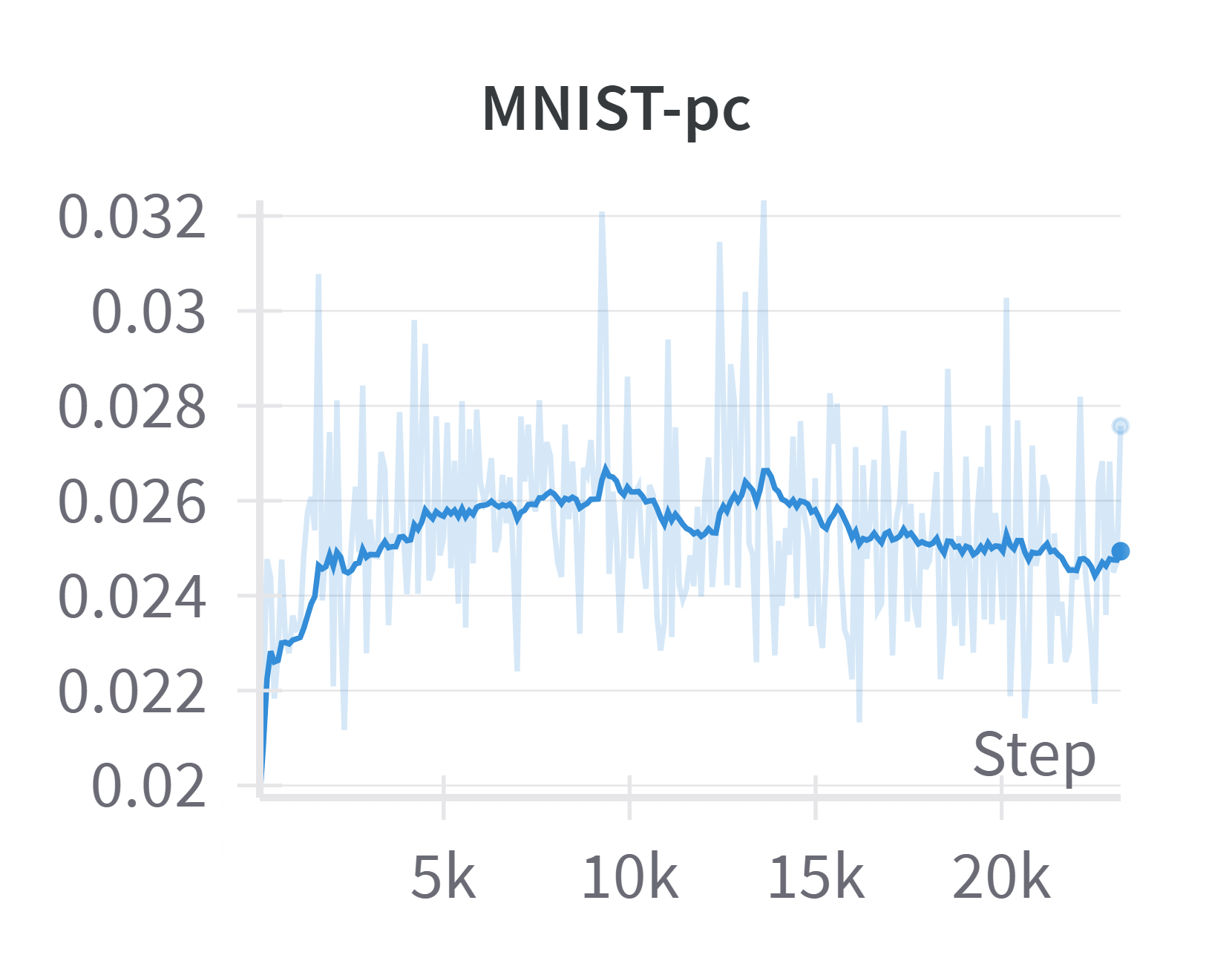}
    \end{subfigure}
    \hspace{-0.2cm}
    \begin{subfigure}{0.24\textwidth}
        \includegraphics[width=\linewidth]{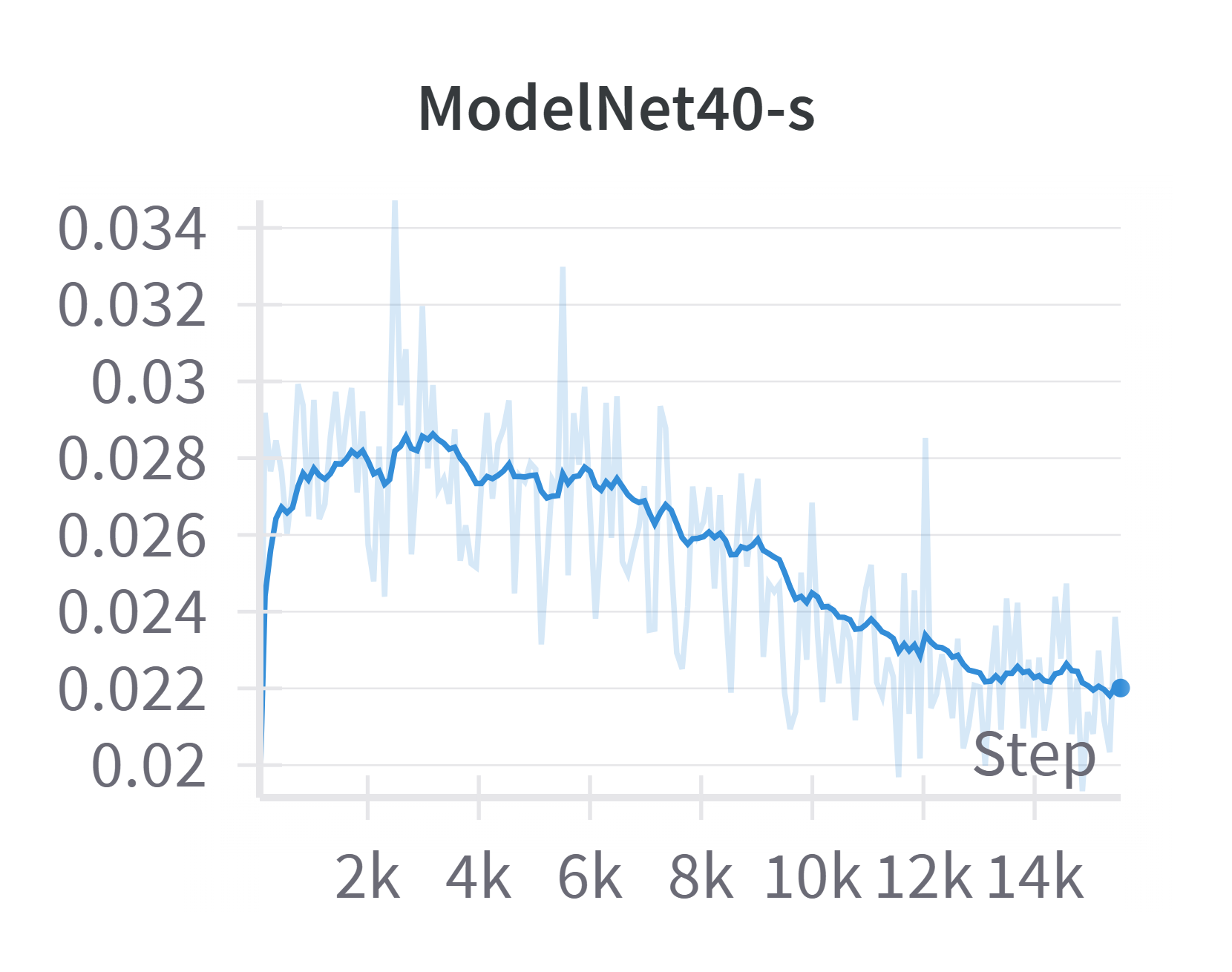}
    \end{subfigure}
    \hspace{-0.2cm}
    \begin{subfigure}{0.24\textwidth}
        \includegraphics[width=\linewidth]{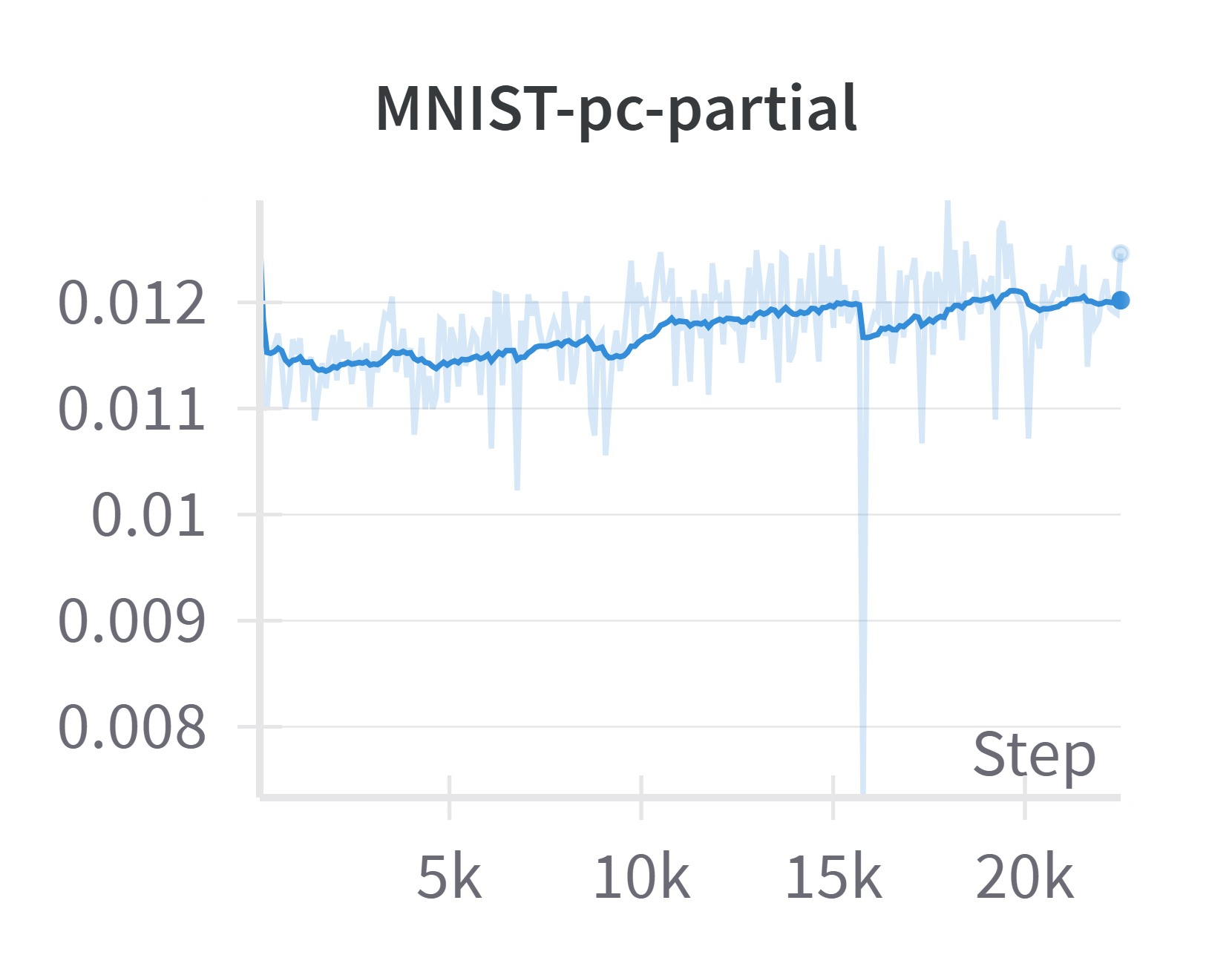}
    \end{subfigure}
    \hspace{-0.2cm}
    \begin{subfigure}{0.24\textwidth}
        \includegraphics[width=\linewidth]{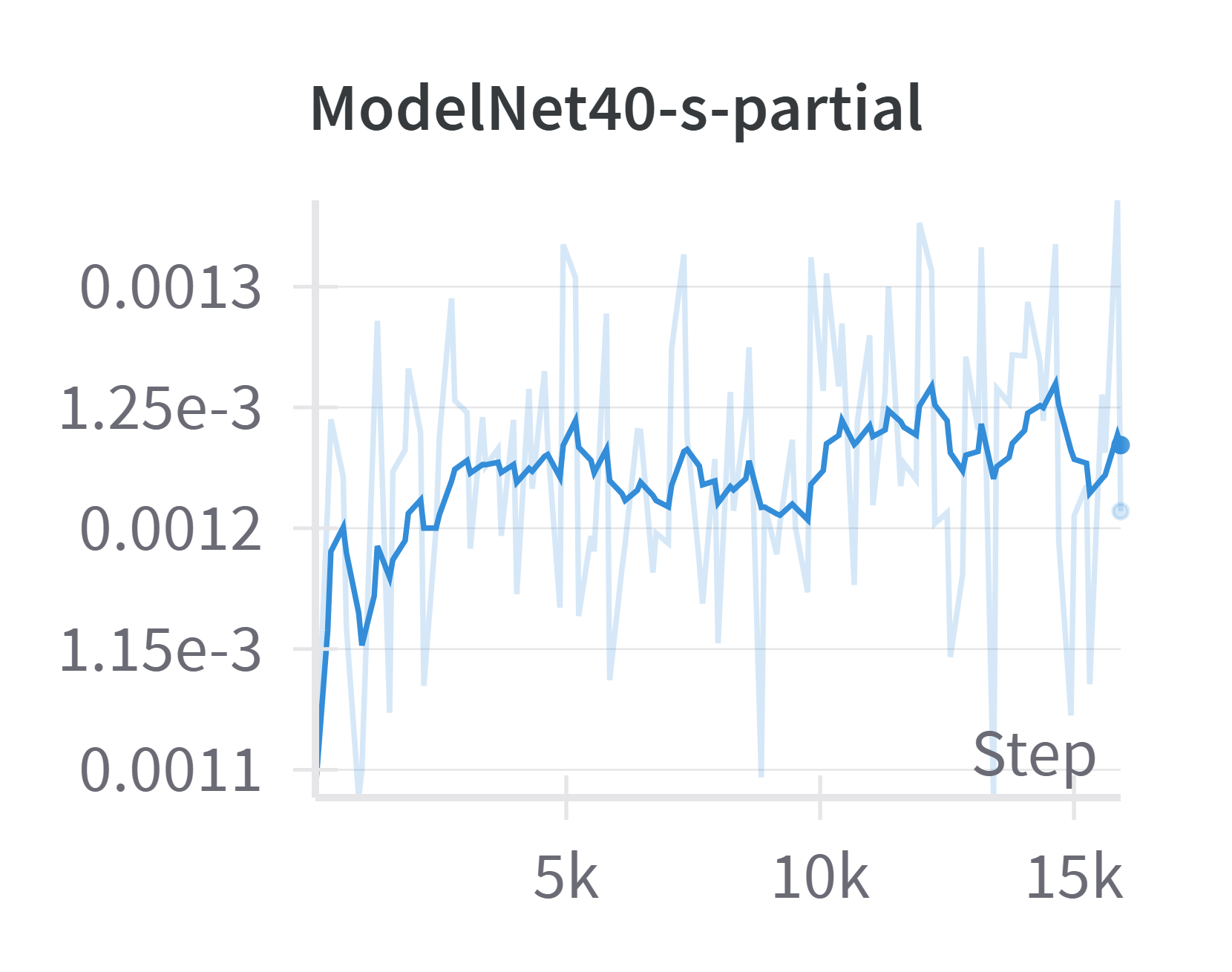}
    \end{subfigure}
    \caption{Empirical Verification of Assumption \ref{assumption:pl} across all four datasets, with a smoothed EMA.}
    \label{fig:pl}
\end{figure}

\textbf{Assumption \ref{assumption:grads}.}
We plot
\begin{equation*}
    \frac{\norm{\nabla_W\cG_{\theta_t,\rho}(\mu_t)-\nabla_W \cF_{\mu^*_{\theta_t,\rho}}(\mu_t)}^2_{L^2(\mu_t)}}{\epsilon_t},
\end{equation*}
where we set $\epsilon_t=\min{(1,t^{-1/2})}$, as in Theorem \ref{thm:cv_loss}, and choose $\mu_t=\mu^*_{\theta_t,\rho}$ for simplicity; the plots can be seen in Figure \ref{fig:wg-diff}. Again, we have boundedness across all four datasets.

\begin{figure}[h]
    \centering
    \begin{subfigure}{0.24\textwidth}
        \includegraphics[width=\linewidth]{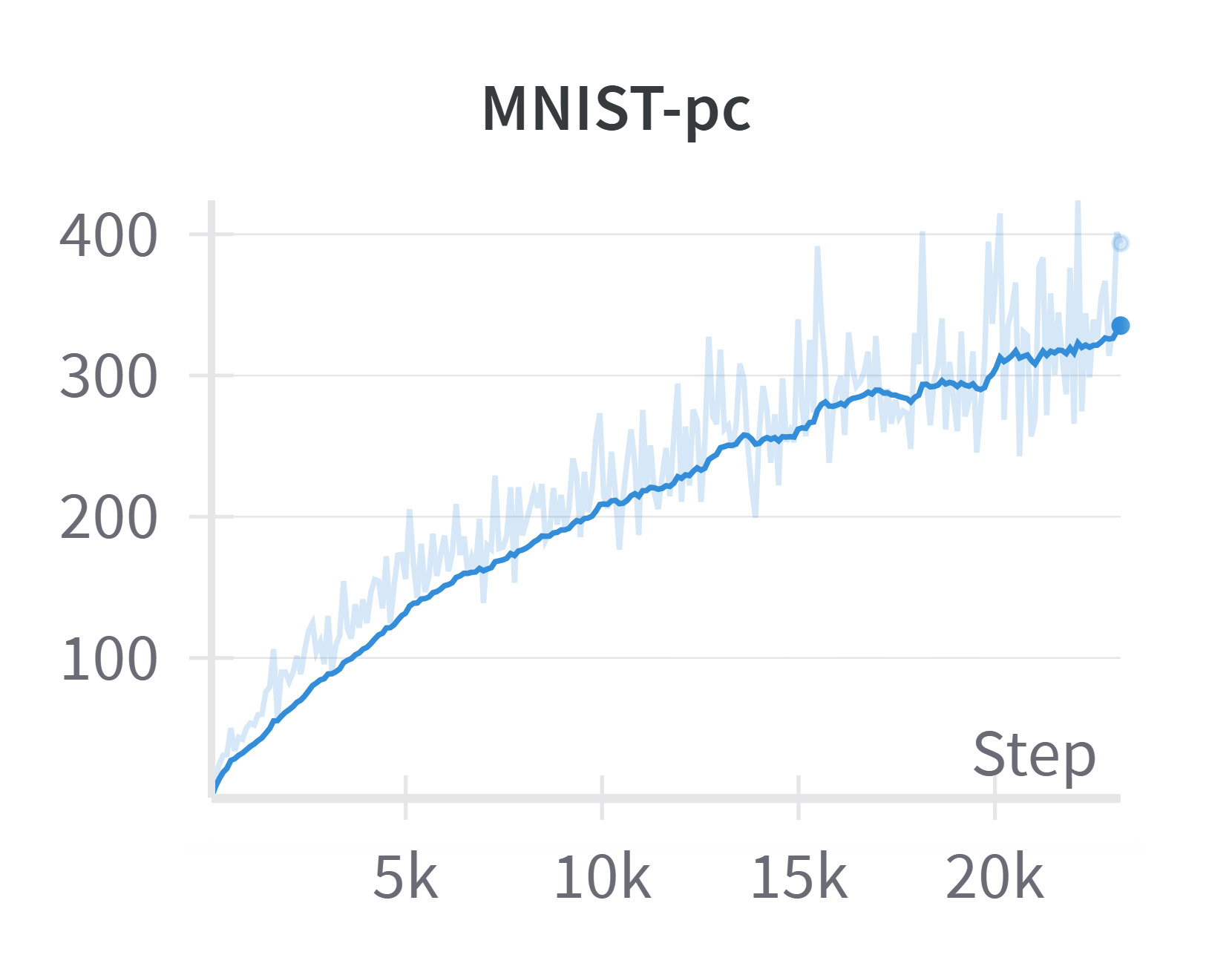}
    \end{subfigure}
    \hspace{-0.2cm}
    \begin{subfigure}{0.24\textwidth}
        \includegraphics[width=\linewidth]{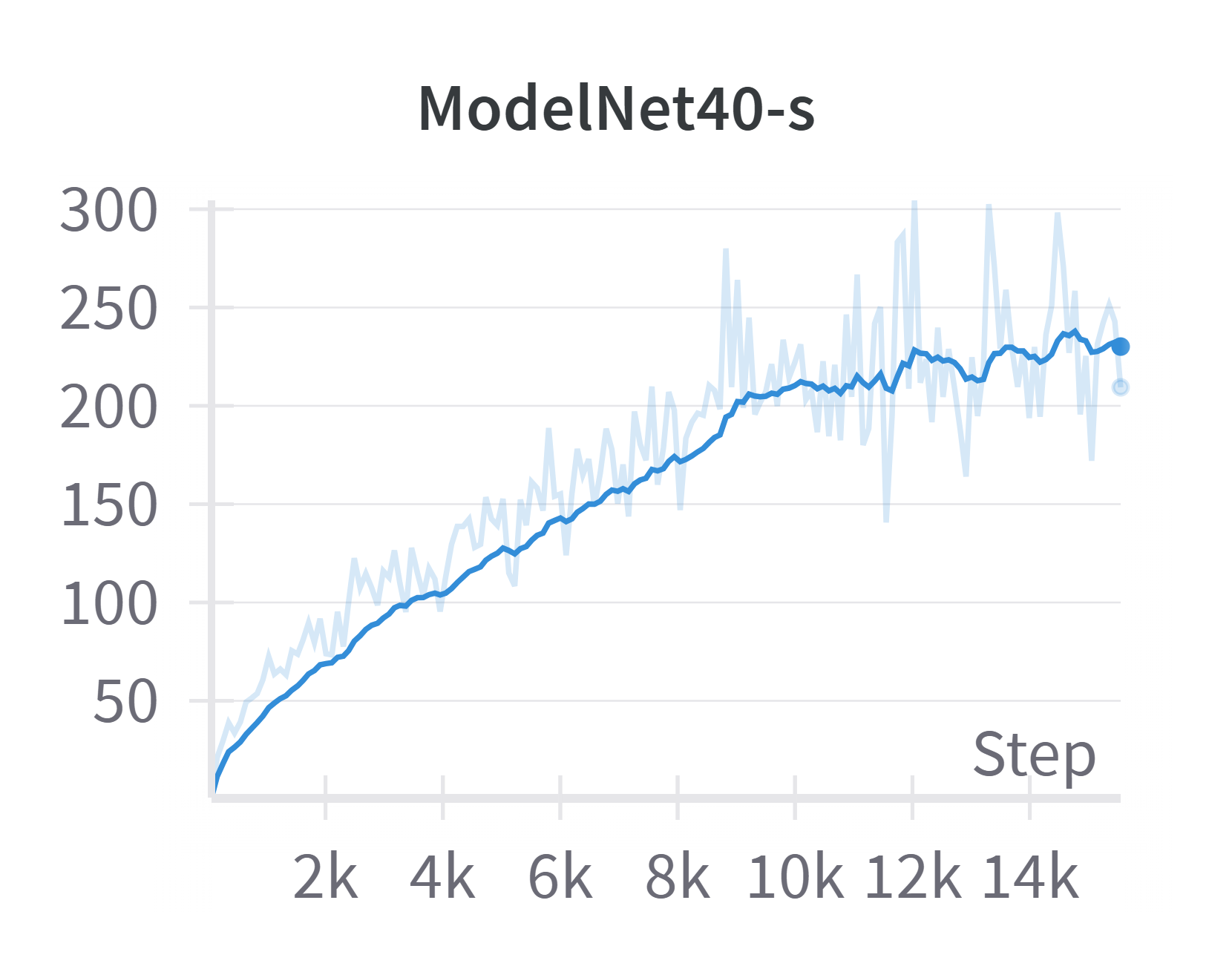}
    \end{subfigure}
    \hspace{-0.2cm}
    \begin{subfigure}{0.24\textwidth}
        \includegraphics[width=\linewidth]{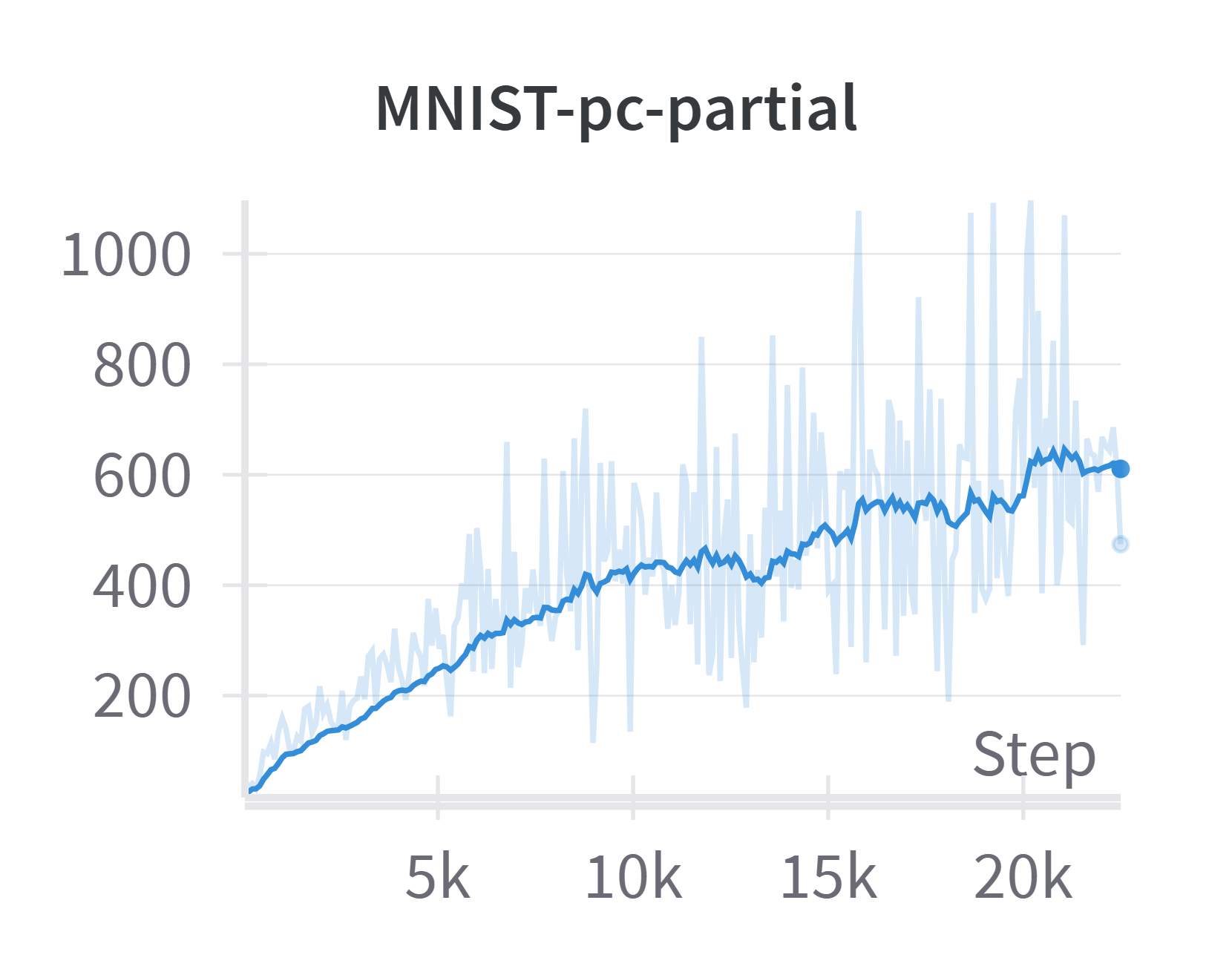}
    \end{subfigure}
    \hspace{-0.2cm}
    \begin{subfigure}{0.24\textwidth}
        \includegraphics[width=\linewidth]{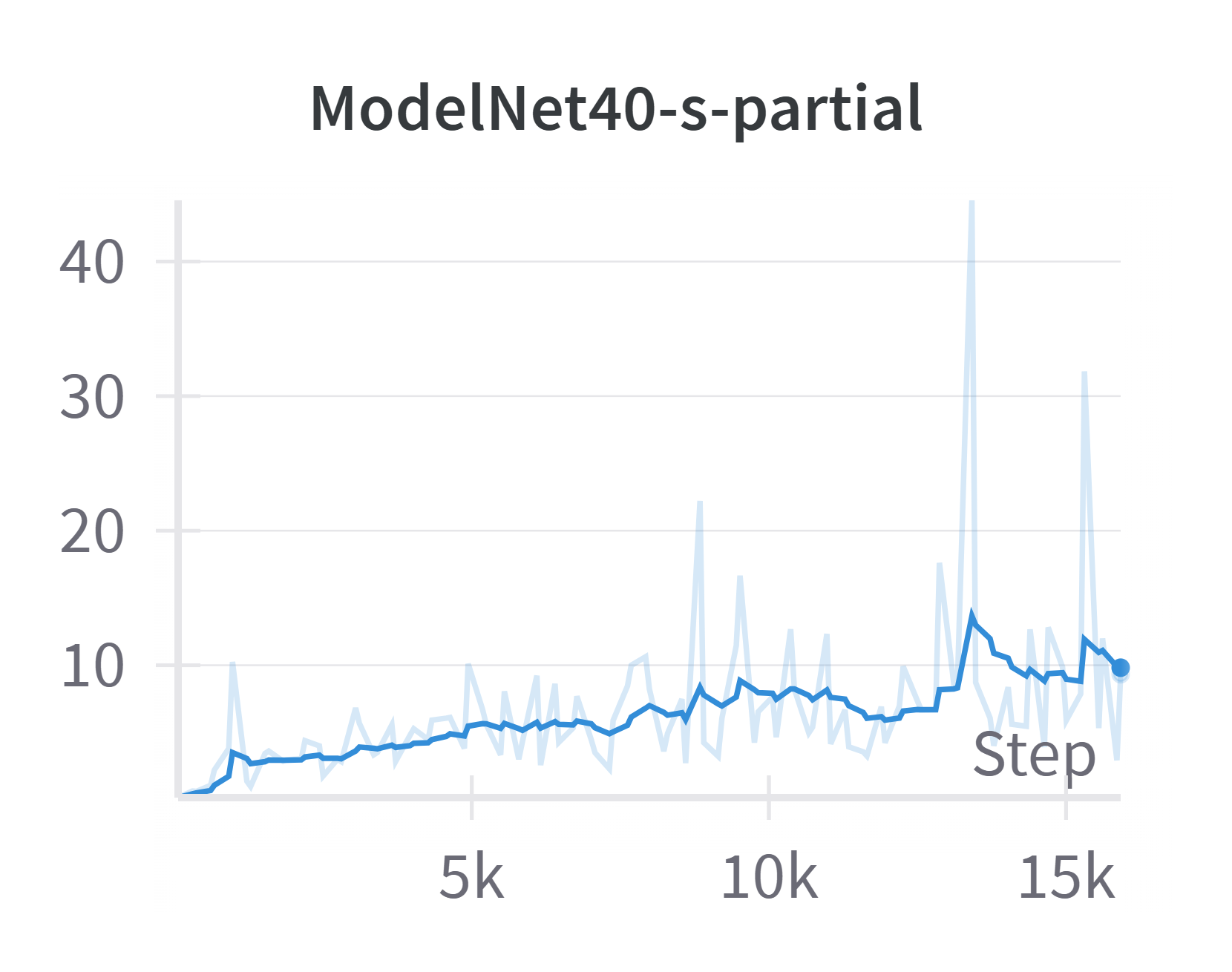}
    \end{subfigure}
    \caption{Empirical Verification of Assumption \ref{assumption:grads} across all four datasets, with a smoothed EMA.}
    \label{fig:wg-diff}
\end{figure}

\textbf{Theorem \ref{thm:cv_loss}.}

We plot
\begin{equation*}
    \frac{\int_0^T \norm{\nabla \cL(\theta_t)}^2\d t}{\sqrt{T} (\log T)^2}
\end{equation*}
over the course of training across all three datasets in Figure \ref{fig:theorem}; here, integration is discrete over the training steps. We see that across datasets, the above quantity is bounded, which verifies Theorem \ref{thm:cv_loss} empirically.

\begin{figure}[h]
    \centering
    \begin{subfigure}{0.24\textwidth}
        \includegraphics[width=\linewidth]{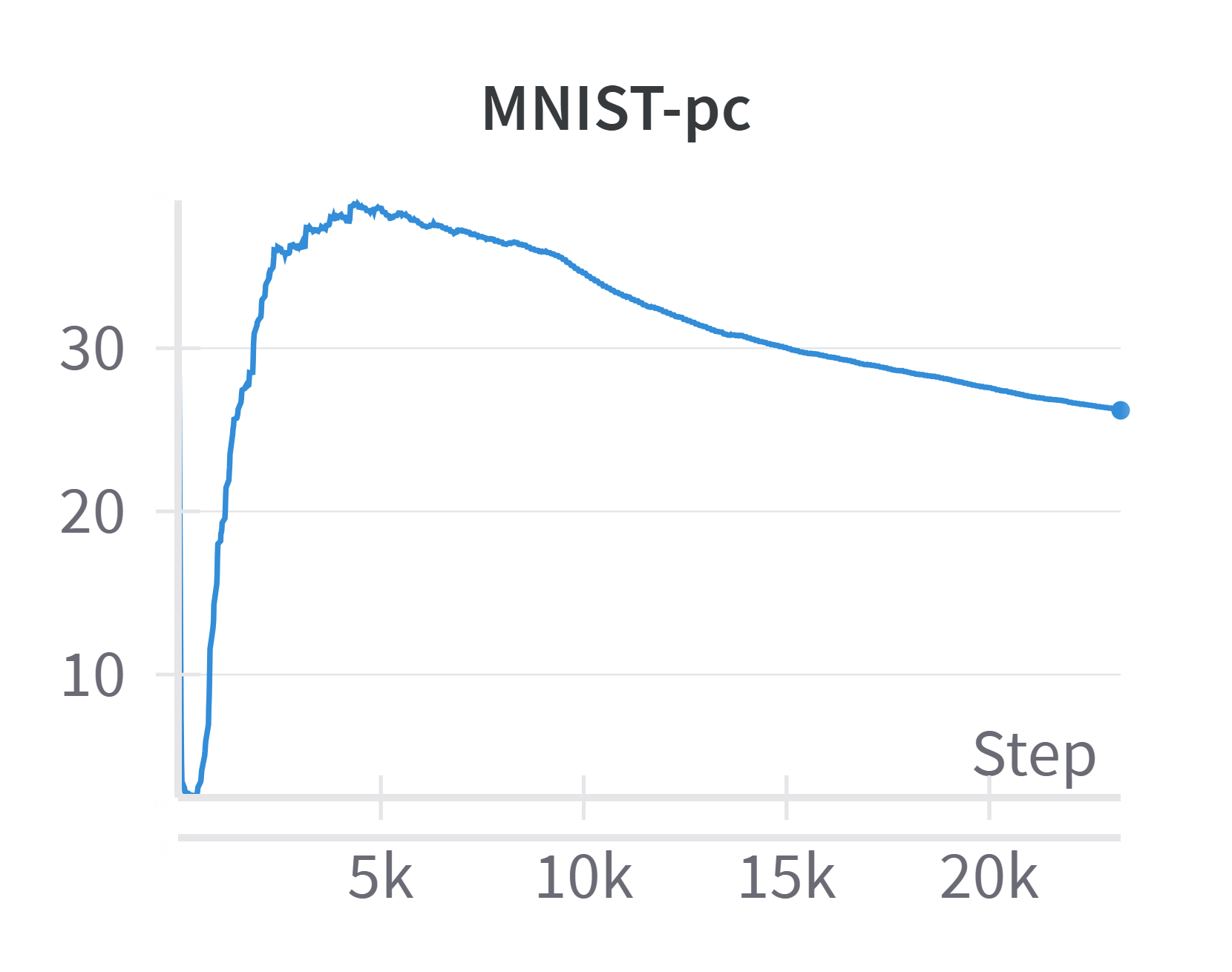}
    \end{subfigure}
    \hspace{-0.2cm}
    \begin{subfigure}{0.24\textwidth}
        \includegraphics[width=\linewidth]{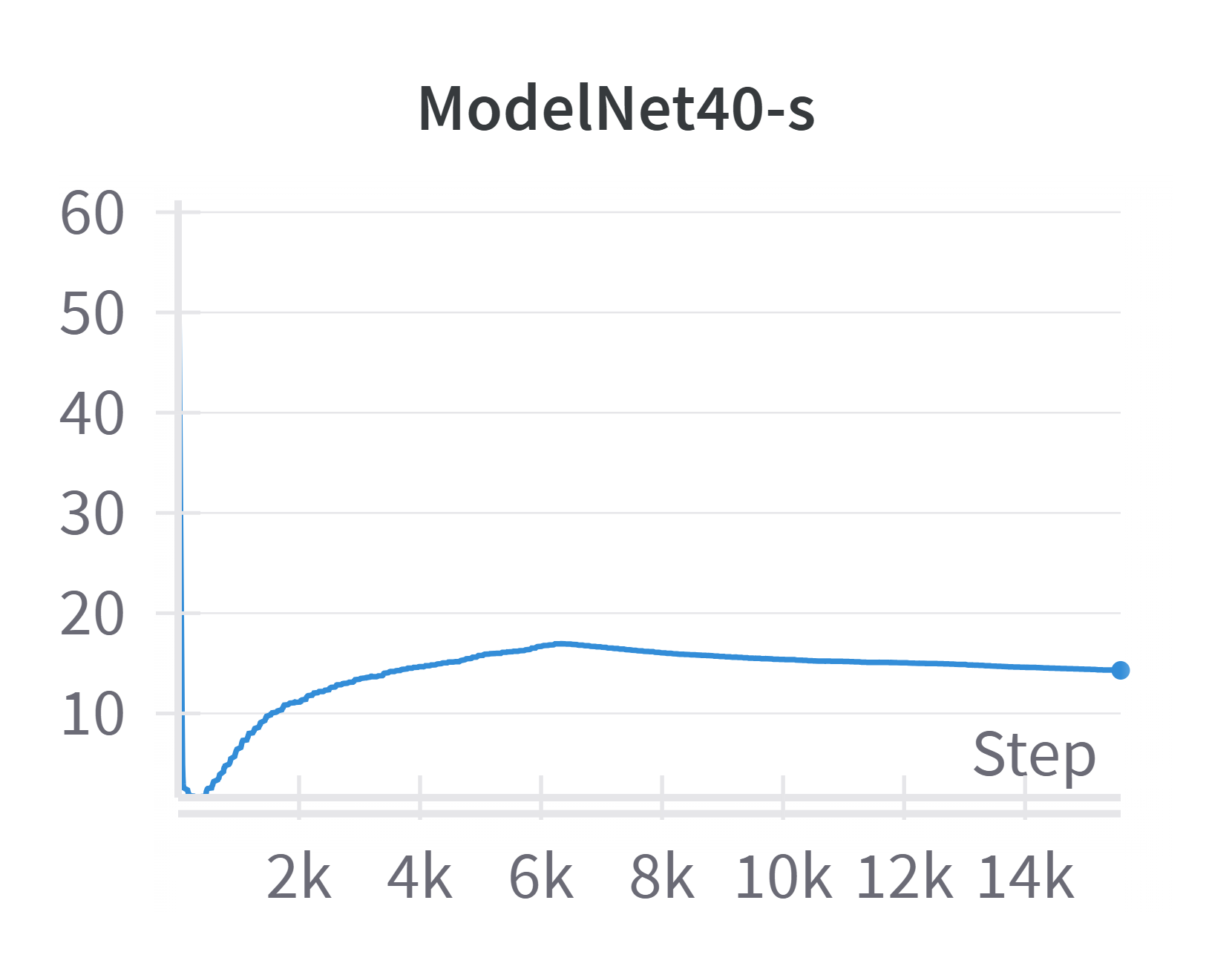}
    \end{subfigure}
    \hspace{-0.2cm}
    \begin{subfigure}{0.24\textwidth}
        \includegraphics[width=\linewidth]{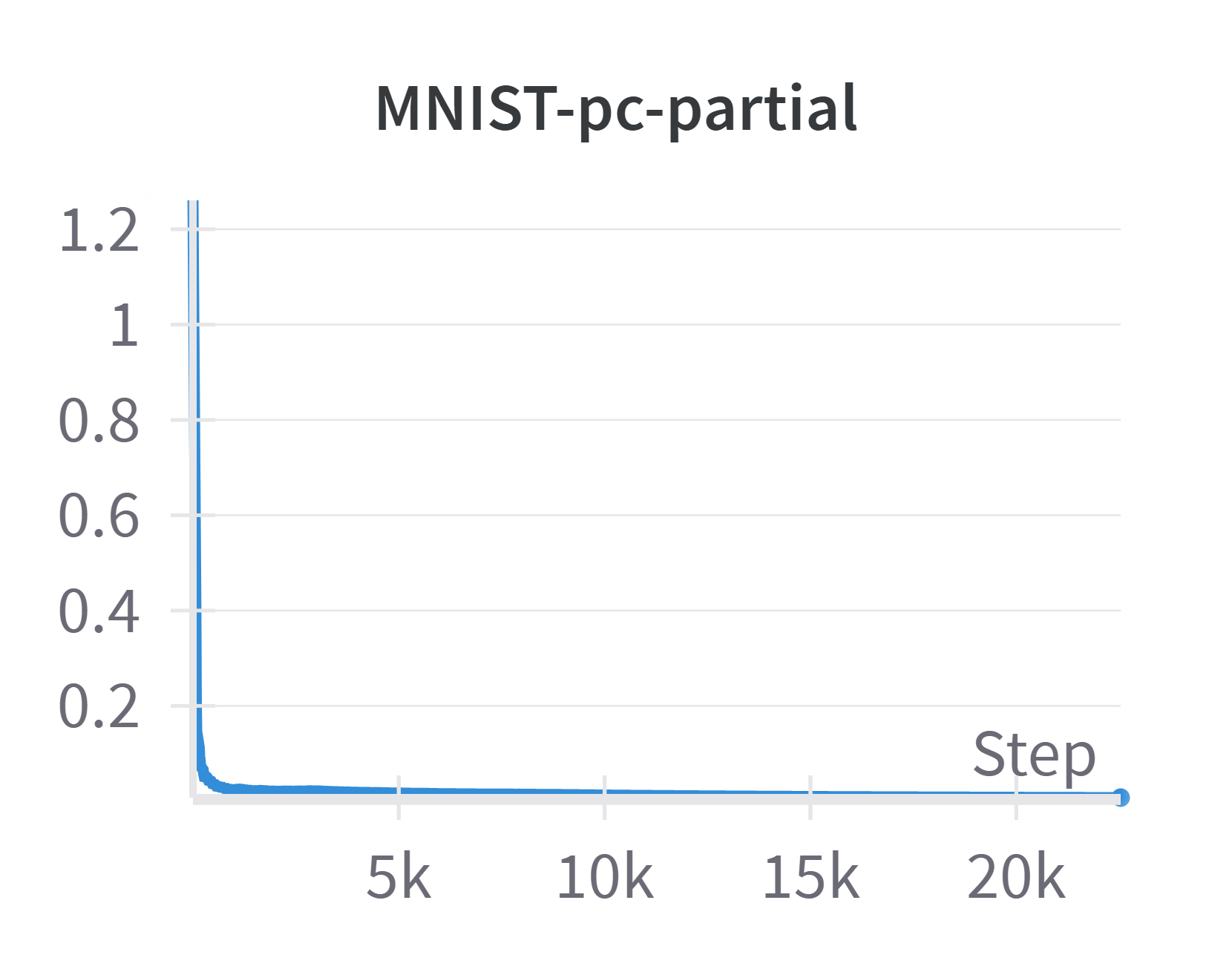}
    \end{subfigure}
    \hspace{-0.2cm}
    \begin{subfigure}{0.24\textwidth}
        \includegraphics[width=\linewidth]{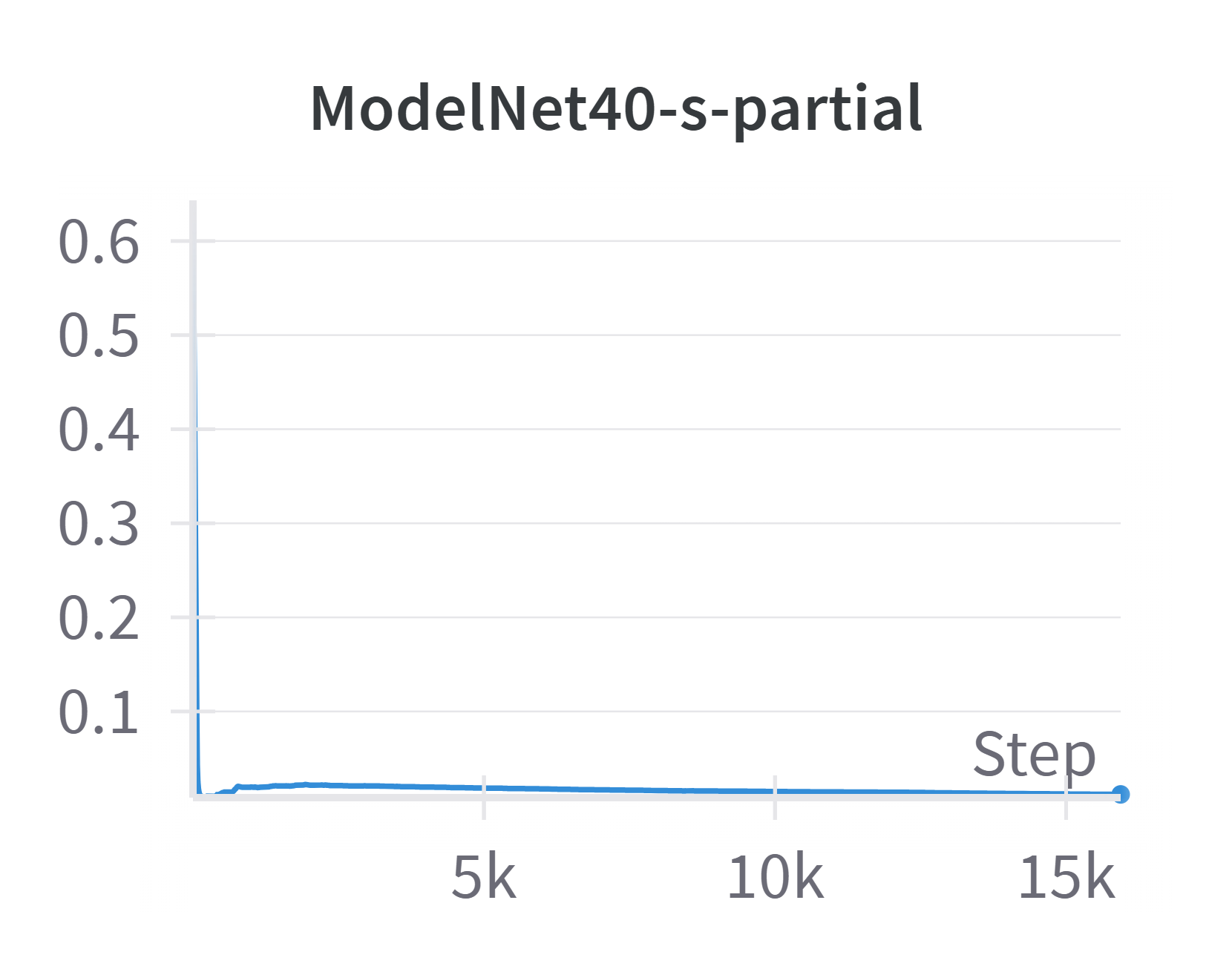}
    \end{subfigure}
    \caption{Empirical Verification of Theorem \ref{thm:cv_loss} across all four datasets.}
    \label{fig:theorem}
\end{figure}

\subsection{Ablation Study on Number of Particles in $\Z$}\label{sec:ab_study_z}
We provide an ablation study on the number of particles in $\Z$. While for point cloud completion, this number is dictated by the number of particles in the input/target, for point cloud classification, this is a free hyperparameter which can be tuned. In Figure \ref{fig:lenz}, we see the average accuracy of a DDEQ for the two classification datasets ModelNet40-s and MNIST-pc over varying numbers of particles in $\Z$.

\begin{figure}[h]
    \centering
    \includegraphics[width=0.5\linewidth]{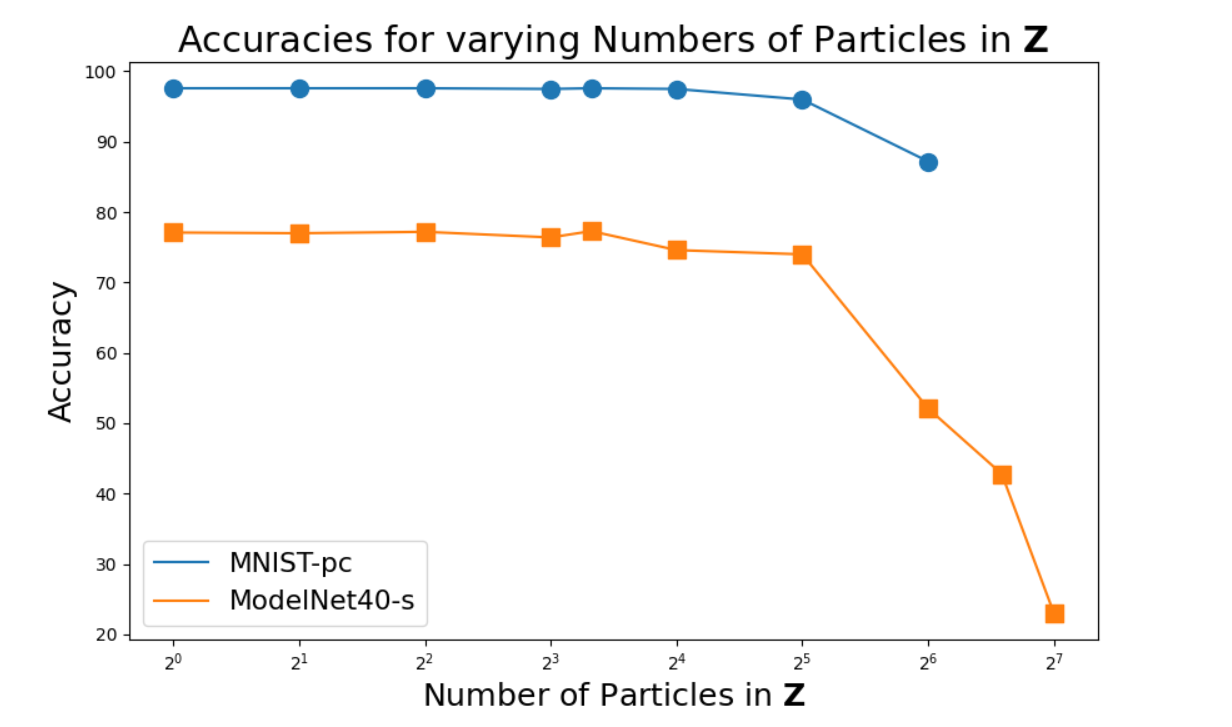}
    \caption{Average accuracy on ModelNet40-s and MNIST-pc for varying numbers of particles in $\Z$.}
    \label{fig:lenz}
\end{figure}
We can see that for both datasets, the accuracy tends to be very similar for between one and 16 particles in $\Z$ (with a very slight increase in accuracy in the range of four to ten particles), and then quickly drops when increasing the number of particles even further.

\subsection{Ablation Study on Pushforward Architecture}\label{sec:pushforward}
In the derivation of the Wasserstein gradient of $\cG(\mu)=\frac{1}{2}\mmds{\mu}{{F_\theta}_\#\mu}$, as well as in the proof of Theorem \ref{thm:cv_loss}, we assumed that we can write $F_\theta(\mu)={F_\theta}_\#(\mu)$ as a pushforward operator. This means there cannot be any interactions between particles in how $F_\theta$ processes them. However, in the architecture we use in practice, two parts do not comply with this assumption, namely the bilinear layer and the self-encoder on $\Z$ (which corresponds to $\mu$). In this section, we study how making our architecture a pushforward operator on $\Z$, namely by changing the bilinear layer from

\begin{equation*}
    F^{\textnormal{bil}}(\Z,\X)=\Z\alpha \X^\top 1_M+1_N 1_N^\top \Z\beta \X^\top 1_M
\end{equation*}

to

\begin{equation*}
    F^{\textnormal{bil}}(\Z,\X)=\Z\alpha \X^\top 1_M,
\end{equation*}
i.e. removing the part that introduces interactions between particles in $\Z$, and by removing the self-encoder on $\Z$ (while leaving the self-encoder on $\X$ and the cross-encoder unchanged) affects the performance.

\begin{table}[h!]
\centering
\caption{Accuracies and test losses of DDEQs, for the pushforward architecture compared against the complete architecture we use in practice.}
\begin{tabular}{c|cc|cc|c}
 & \multicolumn{2}{c|}{\textbf{MNIST-pc}} & \multicolumn{2}{c|}{\textbf{ModelNet40-s}} & \textbf{MNIST-pc-partial} \\

 & \textbf{Acc.} & \textbf{Cross-Entr. Loss} & \textbf{Acc.} & \textbf{Cross-Entr. Loss} & \textbf{MMD Loss} \\
\midrule
\textbf{pushforward}& 97.7 &0.101 & 75.8& 0.938& 0.0034 \\
\textbf{full architecture}  &98.1 &0.086 & 77.3& 0.888& 0.0033\\

\end{tabular}

\label{tab:pushforward}
\end{table}

In Table \ref{tab:pushforward}, we can see the differences in accuracy and test loss between the complete architecture, and the one that can be written as a pushforward. We see that using the architecture that's strictly a pushforward is overall comparable to the full architecture, but slightly worse in all metrics.

\subsection{Computational Complexity}
In this section, we address the computational complexity of DDEQs, and how it changes w.r.t. the number of forward iterations in the inner loop, as well as the number of points in the input measures. The results are for point cloud completion on \textit{MNIST-pc-partial} trained on a single H100 GPU. When reducing the number of forward pass iterations of the inner loop, we increase the inner loop learning rate accordingly, such that \textit{learning rate} $\times$ \textit{forward iterations} remains constant.

\begin{table}[h]
    \centering
    \begin{tabular}{lcccccc}
        \toprule
        \textbf{Number of Iterations} & 10 & 50 & 100 & 200 & 500 & 1000 \\
        \midrule
        \textbf{Runtime (1 Epoch)}    & 13m & 47m & 1h32m & 3h1m & 8h12m & 14h53m \\
        \bottomrule
    \end{tabular}
    \caption{Runtime per epoch for different numbers of iterations.}
    \label{tab:iterations_runtime}
\end{table}

\begin{table}[h]
    \centering
    \begin{tabular}{lccccc}
        \toprule
        \textbf{Input Particles (\%)} & 20\% & 40\% & 60\% & 80\% & 100\% \\
        \midrule
        \textbf{Runtime (1 Epoch)}    & 1h16m & 1h29m & 1h51m & 2h20m & 3h1m \\
        \bottomrule
    \end{tabular}
    \caption{Runtime per epoch for different input particle percentages.}
    \label{tab:particles_runtime}
\end{table}
In contrast, the runtime for PCN is 4min. This shows that the number of forward iterations is crucial in controlling the runtime of DDEQs. It can be reduced to around 50 iterations without significant loss in performance, but deriving methods to improve accuracy for fewer iterations is an important line of future research, as well as deriving other methods to speed up training, such as improving the optimizer, possibly with momentum. Fixed point acceleration methods such as Anderson acceleration that exist in a Euclidean setting could be derived for WGFs; another approach would be to use the sliced MMD \citep{rieszkernels}, which reduces the computational complexity of computing (the gradient of) the MMD from $O(n^2)$ to $O(n\log n)$ (where $n$ is the number of particles).

\subsection{Comparing DDEQs and DEQs}
To demonstrate that DDEQs outperform standard DEQs on point cloud completion, we trained classical DEQs with the same architecture, but with (1) Broyden, (2) Anderson, and (3) fixed point iteration forward solvers, three widely used forward solvers for DEQs. All three perform similarly and achieve a test loss, measured as the squared MMD between prediction and target, of $13\times$ that of DDEQs, and an average Wasserstein-2 loss of $4\times$ that of DDEQs, with very poor visual image quality. However, we want to mention that DEQs achieve (almost) comparable performance on point cloud classification, highlighting the fact that while DDEQs are capable of point cloud classification, this is not a task that necessitates them, whereas point cloud completion clearly is.

\subsection{Fixed Points}
In this section, we show what the fixed points for the classification tasks look like, and how close they are to being real fixed points. In Figures \ref{fig:fixedpointsmnist} and \ref{fig:fixedpointsmodelnet}, we see the fixed points $\Z^*$ obtained for both datasets, alongside $F_\theta(\Z^*)$ -- for perfect fixed points, these two should be identical. Since $\Z$ is high dimensional, we project it to two dimensions with UMAP \citep{mcinnes2020umapuniformmanifoldapproximation}, both with two, as well as 30 neighbors. Clearly, with both projection methods $\Z*$ and $F_\theta(\Z^*)$ tend to be visibly fairly close, but noticeable differences remain.

\begin{figure}[h!]
    \centering
    \includegraphics[width=0.8\linewidth]{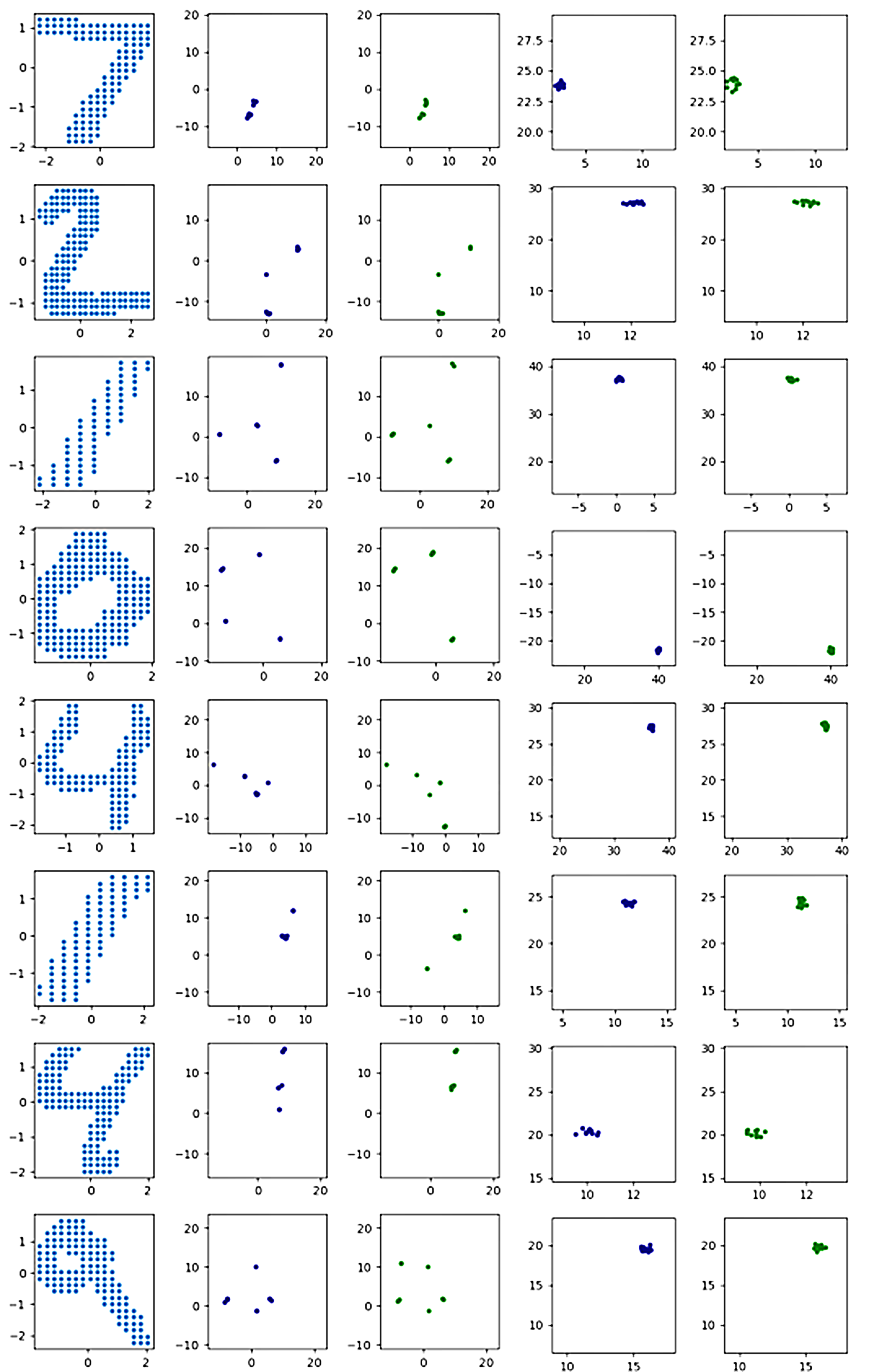}
    \caption{Fixed points for MNIST-pc classification. First column is inputs, second and third column are $\Z^*$ resp. $F_\theta(\Z*)$, mapped with UMAP with \texttt{n\_neighbors}=2; and the fourth and fifth column are $\Z^*$ resp. $F_\theta(\Z*)$, mapped with UMAP with \texttt{n\_neighbors}=30.}
    \label{fig:fixedpointsmnist}
\end{figure}

\begin{figure}[h!]
    \centering
    \includegraphics[width=0.8\linewidth]{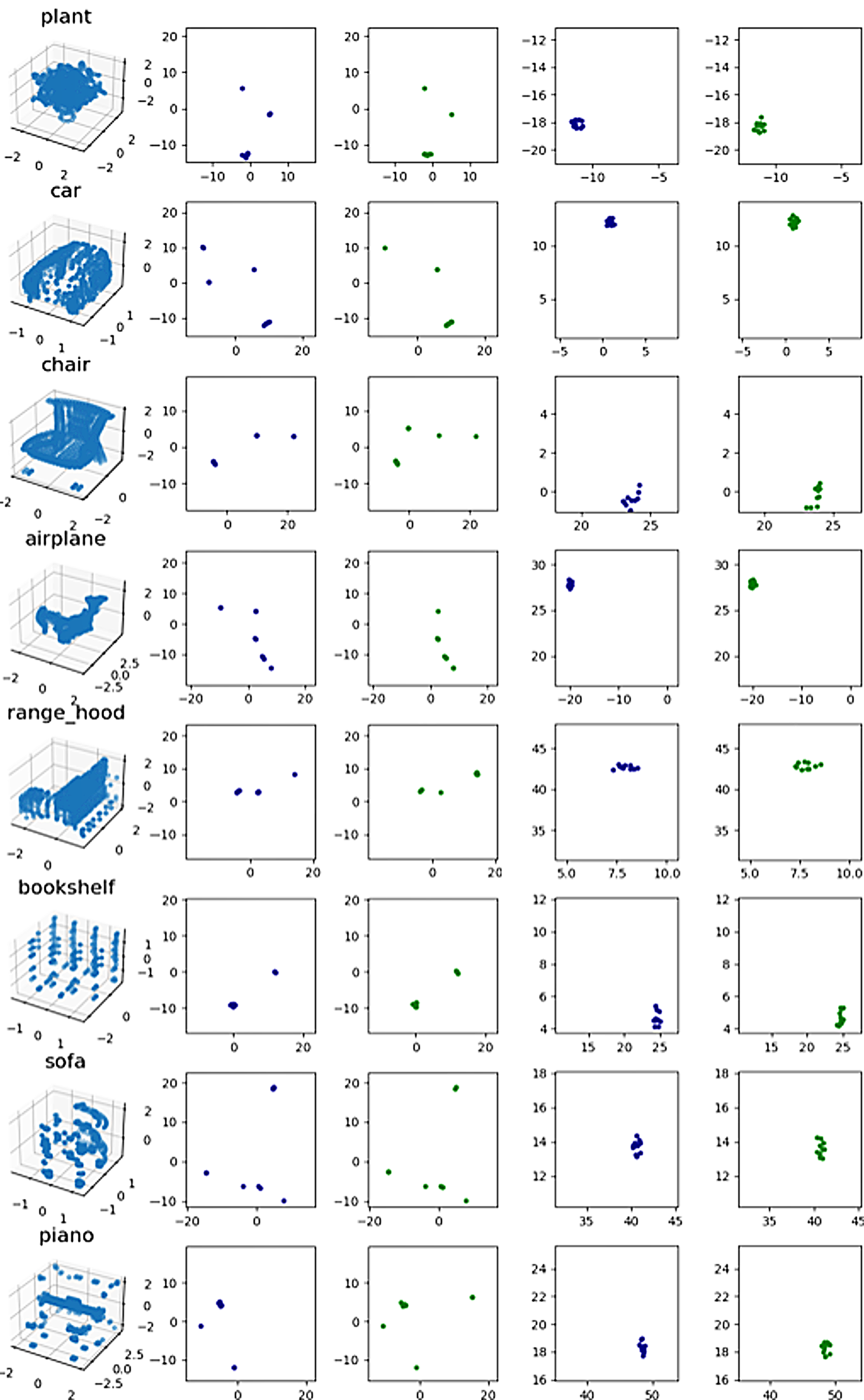}
    \caption{Fixed points for ModelNet40-s classification. First column is inputs, second and third column are $\Z^*$ resp. $F_\theta(\Z*)$, mapped with UMAP with \texttt{n\_neighbors}=2; and the fourth and fifth column are $\Z^*$ resp. $F_\theta(\Z*)$, mapped with UMAP with \texttt{n\_neighbors}=30.}
    \label{fig:fixedpointsmodelnet}
\end{figure}

\subsection{Class Embeddings}
In our classification models, the classification head is a single linear layer which gets as input only the MaxPooling of the fixed points. Thus, the DDEQ needs to embed the class information in the MaxPooling of the fixed points. In Figures \ref{fig:distributionmnist} and \ref{fig:distributionmodelnet}, we see the output of the MaxPooling for the samples of three batches of the test split, respectively.

\begin{figure}[h!]
    \centering
    \includegraphics[width=0.8\linewidth]{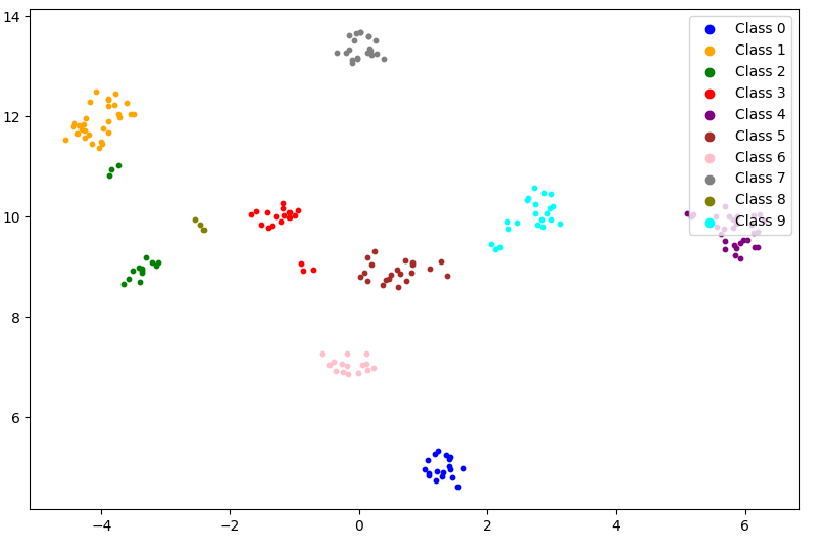}
    \caption{Distribution of the MaxPooling of the fixed points on MNIST-pc, visualized with UMAP with \texttt{n\_neighbors}=80.}
    \label{fig:distributionmnist}
\end{figure}

\begin{figure}[h!]
    \centering
    \includegraphics[width=1\linewidth]{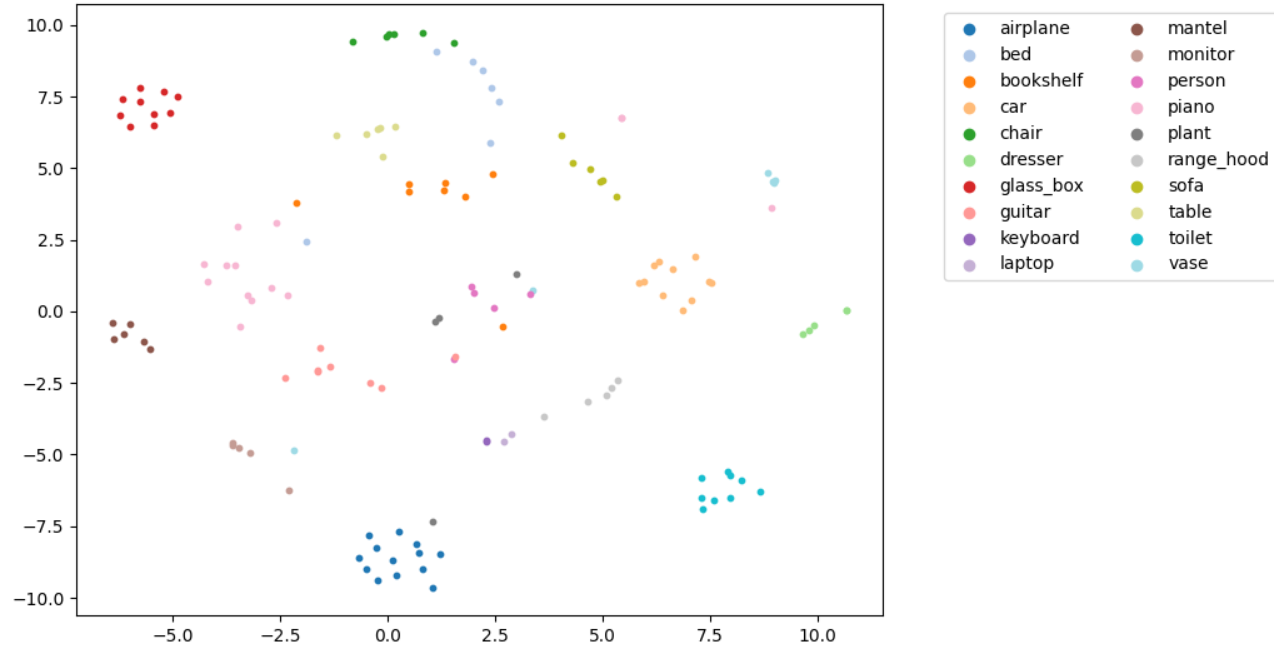}
    \caption{Distribution of the MaxPooling of the fixed points on ModelNet40-s, for a subset of 20 of the 40 classes. Visualized with TSNE \citep{tsne} with \texttt{perplexity}=30.}
    \label{fig:distributionmodelnet}
\end{figure}

As can be seen, the embeddings for MNIST-pc classes are very distinct, and the embeddings for ModelNet40-s are distinct for most classes, but for some, there is significant overlap. Note that these are visualizations of 128-dimensional points in two dimensions, and suffer from the limitations of the representation techniques used.

\subsection{Additional ModelNet40-s-partial Samples}\label{sec:additionalmodelnet}
Figure \ref{fig:modelnetsamples} contains additional samples for point cloud completion on ModelNet40-s-partial with DDEQ and PCN. We see from the samples that DDEQ excels at certain classes, such as airplane, but other classes, such as bowls or cones, can be failure modes. PCN tends to do better on airplanes as well, but generally produces point clouds that are too diffuse.
In Table \ref{tab:modelnetot}, we report the average OT distances (with Euclidean distance as cost function) between the predictions and ground truth for DDEQ and PCN. For DDEQ, we compute both the distance between the complete prediction and ground truth, as well as only for the "free particles" and the corresponding particles in the target.

\begin{table}[htbp]
\centering
\begin{tabular}{l  cc  c}

 & \multicolumn{2}{c}{\textbf{DDEQ}} & \hspace{1cm}\textbf{PCN} \\[5pt]
 & free particles & full point cloud & \\ \hline
\textbf{OT Distance}\hspace{.5cm} & $1.27\pm 0.04$ & $0.154 \pm 0.002$ & \hspace{1cm} $0.61 \pm 0.02$ \\ \hline
\end{tabular}
\caption{OT distances between predictions and targets for DDEQ and PCN on ModelNet40-s-partial.}
\label{tab:modelnetot}
\end{table}

\begin{figure}
    \centering
    \includegraphics[width=\linewidth]{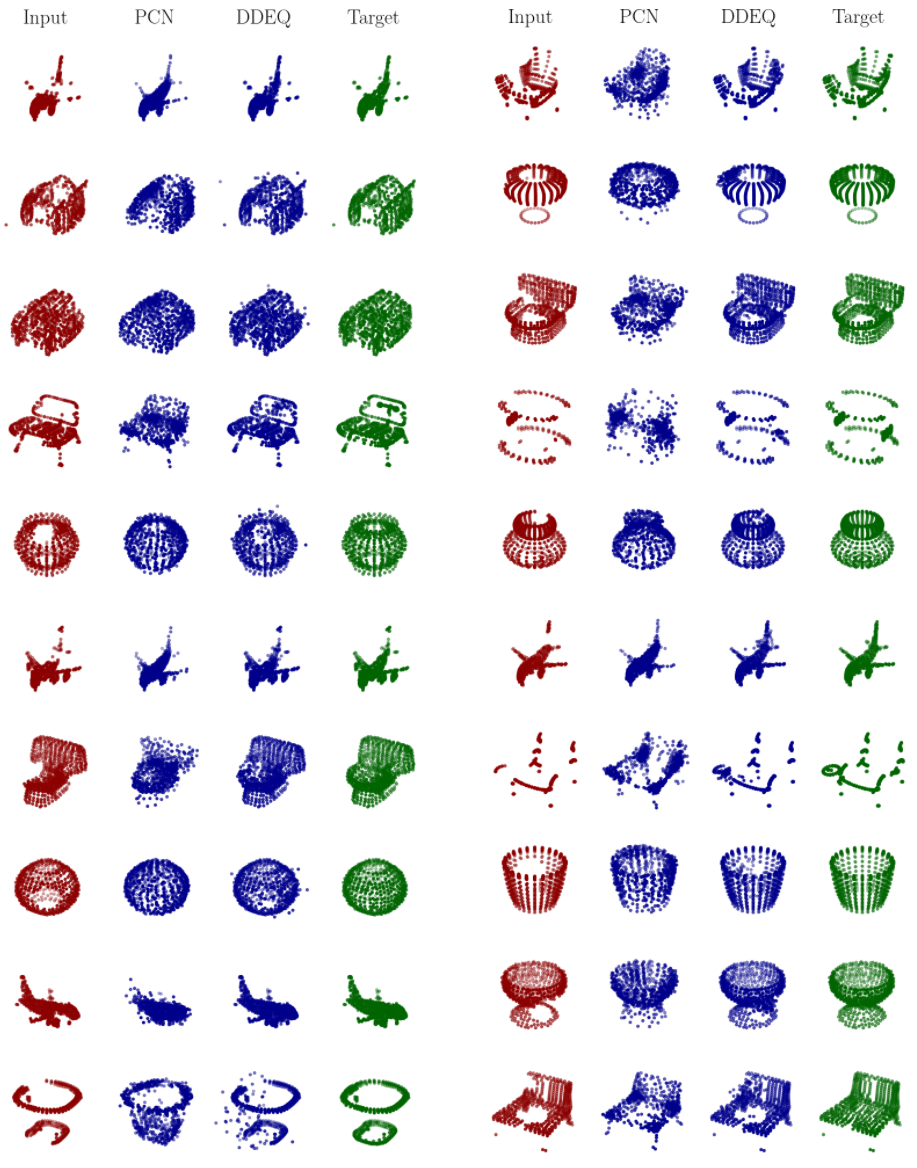}
    \caption{ModelNet40-s-partial samples.}
    \label{fig:modelnetsamples}
\end{figure}

\subsection{Visualization of MMD Gradient Flow With Rotational Target}
In the literature, MMD gradient flows are typically studied with a fixed target measure, i.e. gradient flows of $\mu\mapsto \mmds{\mu}{\nu}$ for a fixed measure $\nu$. To our knowledge, gradient flows of functionals of the form $\mu\mapsto \mmds{\mu}{F(\mu)}$, where $F$ is a functional, have not been studied yet. To illustrate what interesting emergent properties these gradient flows can have even for very simple target functions, we show what the gradient flow looks like when $F$ is a rotation of $2\pi/8$ resp. $2\pi/5$ around the origin, in Figures \ref{fig:rotation5} and \ref{fig:rotation8}. We initialize $x\sim\mathcal{N}\left(\begin{bmatrix}-10\\ 0\end{bmatrix},\begin{bmatrix}
    1&0\\0&12\end{bmatrix}\right)$, such that the initial particles are not centered around the origin, and again use the Riesz kernel, with $\eta=10$ and $\gamma_\eta=0.999$.

We see that instead of converging to a point cloud uniformly distributed around the origin, the MMD flow converges to diverse geometric patterns which are perfectly symmetric.

\begin{figure}
    \centering
    \includegraphics[width=\linewidth]{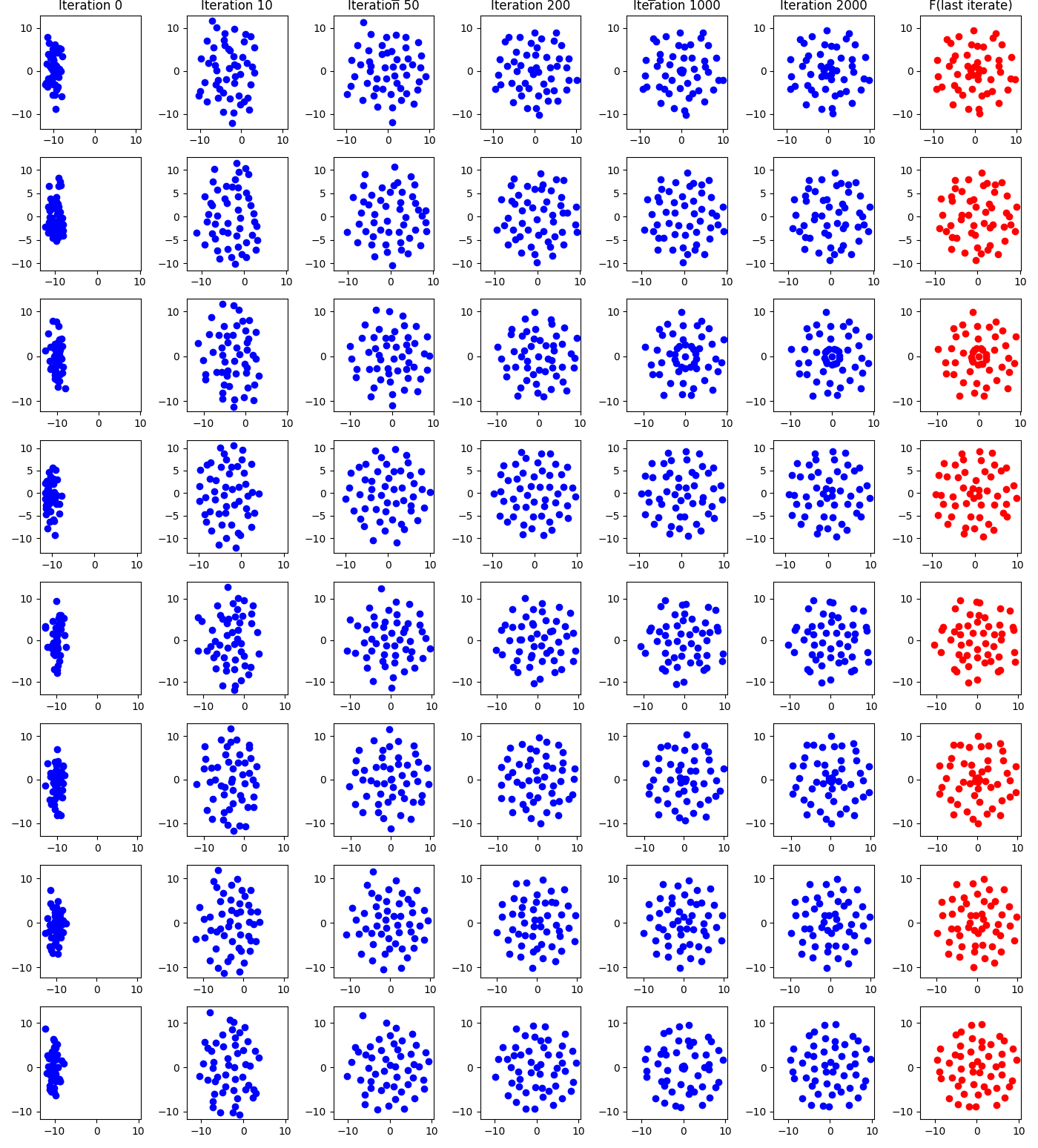}
    \caption{MMD gradient flow where the target functional is a rotation by $2\pi/5$. The columns show the flow at initialization, after 10, 50, 200, 1000, and 2000 iterations, and the rotation applied to the last iterate (in red).}
    \label{fig:rotation5}
\end{figure}

\begin{figure}
    \centering
    \includegraphics[width=\linewidth]{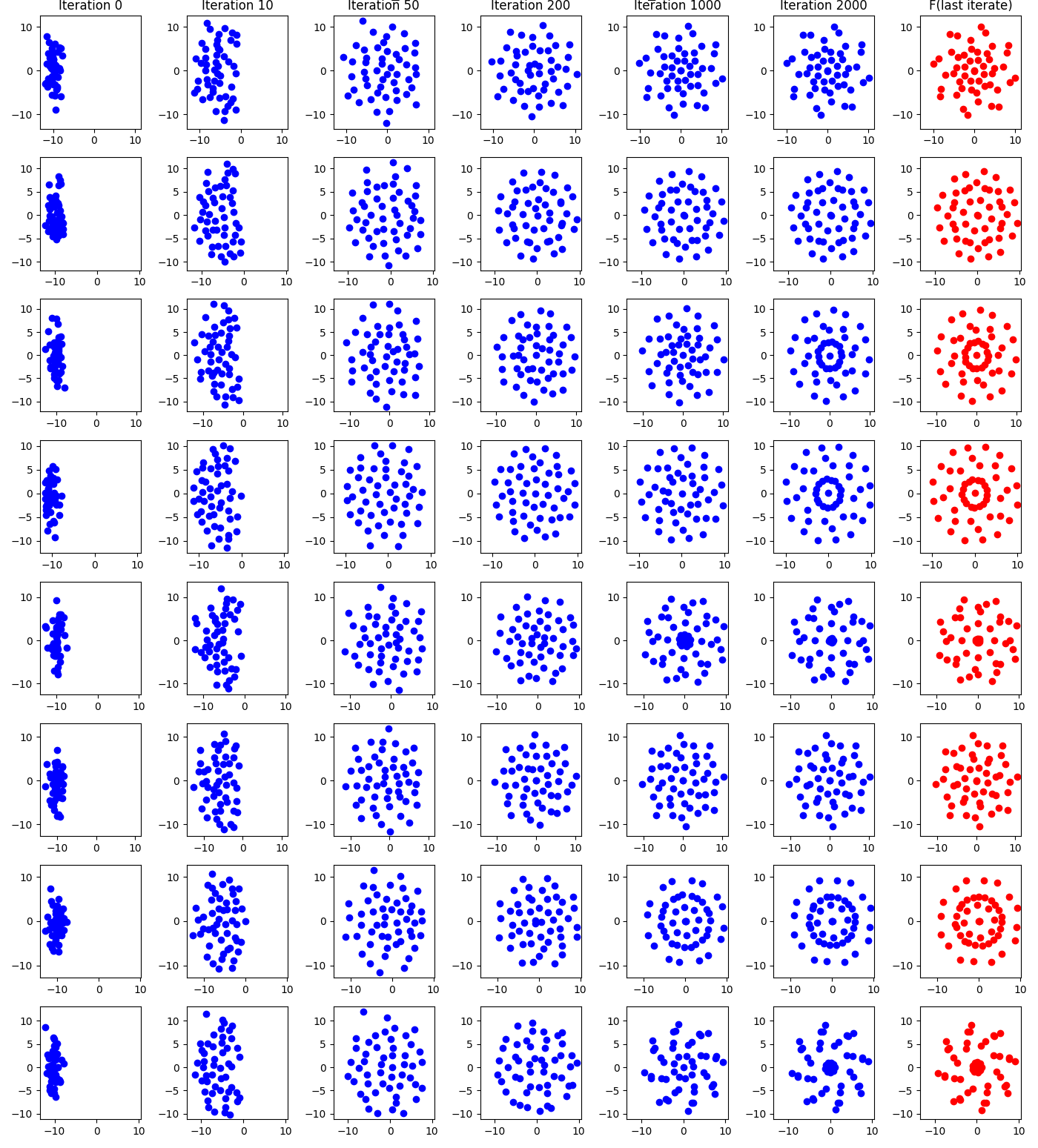}
    \caption{MMD gradient flow where the target functional is a rotation by $2\pi/8$. The columns show the flow at initialization, after 10, 50, 200, 1000, and 2000 iterations, and the rotation applied to the last iterate (in red).}
    \label{fig:rotation8}
\end{figure}

%% file: paper_arxiv.bbl
\begin{thebibliography}{83}
\providecommand{\natexlab}[1]{#1}
\providecommand{\url}[1]{\texttt{#1}}
\expandafter\ifx\csname urlstyle\endcsname\relax
  \providecommand{\doi}[1]{doi: #1}\else
  \providecommand{\doi}{doi: \begingroup \urlstyle{rm}\Url}\fi

\bibitem[Altekrüger et~al.(2023)Altekrüger, Hertrich, and
  Steidl]{altekruger2023neural}
Fabian Altekrüger, Johannes Hertrich, and Gabriele Steidl.
\newblock {Neural Wasserstein Gradient Flows for Discrepancies with Riesz
  Kernels}.
\newblock In \emph{International Conference on Machine Learning}, 2023.

\bibitem[Ambrosio et~al.(2008)Ambrosio, Gigli, and
  Savar{\'e}]{ambrosio2008gradient}
Luigi Ambrosio, Nicola Gigli, and Giuseppe Savar{\'e}.
\newblock \emph{Gradient flows: in metric spaces and in the space of
  probability measures}.
\newblock Springer Science \& Business Media, 2008.

\bibitem[Amos et~al.(2023)Amos, Luise, Cohen, and
  Redko]{amos2023metaoptimaltransport}
Brandon Amos, Giulia Luise, Samuel Cohen, and Ievgen Redko.
\newblock {Meta Optimal Transport}.
\newblock In Andreas Krause, Emma Brunskill, Kyunghyun Cho, Barbara Engelhardt,
  Sivan Sabato, and Jonathan Scarlett, editors, \emph{Proceedings of the 40th
  International Conference on Machine Learning}, volume 202 of
  \emph{Proceedings of Machine Learning Research}, pages 791--813. PMLR, 23--29
  Jul 2023.

\bibitem[Arbel et~al.(2019)Arbel, Korba, Salim, and Gretton]{mmdflow}
Michael Arbel, Anna Korba, Adil Salim, and Arthur Gretton.
\newblock Maximum {Mean} {Discrepancy} {Gradient} {Flow}.
\newblock In \emph{Advances in {Neural} {Information} {Processing} {Systems}},
  volume~32. Curran Associates, Inc., 2019.

\bibitem[Ba et~al.(2016)Ba, Kiros, and Hinton]{ba2016layernormalization}
Jimmy~Lei Ba, Jamie~Ryan Kiros, and Geoffrey~E. Hinton.
\newblock {Layer Normalization}.
\newblock \emph{arXiv preprint arXiv:1607.06450}, 2016.

\bibitem[Bai et~al.(2019)Bai, Kolter, and Koltun]{deq}
Shaojie Bai, J~Zico Kolter, and Vladlen Koltun.
\newblock Deep equilibrium models.
\newblock \emph{Advances in neural information processing systems}, 32, 2019.

\bibitem[Bai et~al.(2020)Bai, Koltun, and Kolter]{mdeqs}
Shaojie Bai, Vladlen Koltun, and J~Zico Kolter.
\newblock Multiscale deep equilibrium models.
\newblock \emph{Advances in neural information processing systems},
  33:\penalty0 5238--5250, 2020.

\bibitem[Bai et~al.(2021)Bai, Koltun, and
  Kolter]{bai2021stabilizingequilibriummodelsjacobian}
Shaojie Bai, Vladlen Koltun, and J.~Zico Kolter.
\newblock {Stabilizing Equilibrium Models by Jacobian Regularization}.
\newblock In \emph{International Conference on Machine Learning}, 2021.

\bibitem[Bai et~al.(2022)Bai, Geng, Savani, and
  Kolter]{bai2022deepequilibriumopticalflow}
Shaojie Bai, Zhengyang Geng, Yash Savani, and J~Zico Kolter.
\newblock {Deep Equilibrium Optical Flow Estimation}.
\newblock In \emph{Proceedings of the IEEE/CVF conference on computer vision
  and pattern recognition}, pages 620--630, 2022.

\bibitem[Blanchet and Bolte(2018)]{blanchet2018family}
Adrien Blanchet and J{\'e}r{\^o}me Bolte.
\newblock A family of functional inequalities: {\L}ojasiewicz inequalities and
  displacement convex functions.
\newblock \emph{Journal of Functional Analysis}, 275\penalty0 (7):\penalty0
  1650--1673, 2018.

\bibitem[Blondel and Roulet(2024)]{blondel2024elements}
Mathieu Blondel and Vincent Roulet.
\newblock The elements of differentiable programming.
\newblock \emph{arXiv preprint arXiv:2403.14606}, 2024.

\bibitem[Bonet et~al.(2024)Bonet, Uscidda, David, Aubin-Frankowski, and
  Korba]{bonet2024mirror}
Clément Bonet, Théo Uscidda, Adam David, Pierre-Cyril Aubin-Frankowski, and
  Anna Korba.
\newblock {Mirror and Preconditioned Gradient Descent in Wasserstein Space}.
\newblock In \emph{Advances in Neural Information Processing Systems}, 2024.

\bibitem[Bonet et~al.(2025)Bonet, Drumetz, and Courty]{bonet2024sliced}
Clément Bonet, Lucas Drumetz, and Nicolas Courty.
\newblock {Sliced-Wasserstein Distances and Flows on Cartan-Hadamard
  Manifolds}.
\newblock \emph{Journal of Machine Learning Research}, 26\penalty0
  (32):\penalty0 1--76, 2025.

\bibitem[Bonnet(2019)]{bonnet2019pontryagin}
Beno{\^\i}t Bonnet.
\newblock {A Pontryagin Maximum Principle in Wasserstein spaces for constrained
  optimal control problems}.
\newblock \emph{ESAIM: Control, Optimisation and Calculus of Variations},
  25:\penalty0 52, 2019.

\bibitem[Carlier et~al.(2024)Carlier, Chizat, and
  Laborde]{carlier2024displacement}
Guillaume Carlier, L{\'e}na{\"\i}c Chizat, and Maxime Laborde.
\newblock Displacement smoothness of entropic optimal transport.
\newblock \emph{ESAIM: Control, Optimisation and Calculus of Variations},
  30:\penalty0 25, 2024.

\bibitem[Castin et~al.(2024)Castin, Ablin, and Peyr{\'e}]{castin2024smooth}
Val{\'e}rie Castin, Pierre Ablin, and Gabriel Peyr{\'e}.
\newblock {How Smooth Is Attention?}
\newblock In \emph{International Conference of Machine Learning}, 2024.

\bibitem[Chazal et~al.(2024)Chazal, Korba, and Bach]{chazal2024statistical}
Cl{\'e}mentine Chazal, Anna Korba, and Francis Bach.
\newblock {Statistical and Geometrical properties of regularized Kernel
  Kullback-Leibler divergence}.
\newblock In \emph{Advances in Neural Information Processing Systems}, 2024.

\bibitem[Chen et~al.(2018)Chen, Rubanova, Bettencourt, and
  Duvenaud]{neuralodes}
Ricky~TQ Chen, Yulia Rubanova, Jesse Bettencourt, and David~K Duvenaud.
\newblock Neural ordinary differential equations.
\newblock \emph{Advances in neural information processing systems}, 31, 2018.

\bibitem[Chen et~al.(2024{\natexlab{a}})Chen, Xu, Lu, Stenetorp, and
  Franceschi]{chen2024jet}
Yihong Chen, Xiangxiang Xu, Yao Lu, Pontus Stenetorp, and Luca Franceschi.
\newblock {Jet Expansions of Residual Computation}.
\newblock \emph{arXiv preprint arXiv:2410.06024}, 2024{\natexlab{a}}.

\bibitem[Chen et~al.(2024{\natexlab{b}})Chen, Mustafi, Glaser, Korba, Gretton,
  and Sriperumbudur]{chen2024regularized}
Zonghao Chen, Aratrika Mustafi, Pierre Glaser, Anna Korba, Arthur Gretton, and
  Bharath~K Sriperumbudur.
\newblock {(De)-regularized Maximum Mean Discrepancy Gradient Flow}.
\newblock \emph{arXiv preprint arXiv:2409.14980}, 2024{\natexlab{b}}.

\bibitem[Chewi et~al.(2024)Chewi, Niles-Weed, and
  Rigollet]{chewi2024statistical}
Sinho Chewi, Jonathan Niles-Weed, and Philippe Rigollet.
\newblock Statistical optimal transport.
\newblock \emph{arXiv preprint arXiv:2407.18163}, 2024.

\bibitem[Chizat and Bach(2018)]{chizat2018globalconvergencegradientdescent}
Lenaic Chizat and Francis Bach.
\newblock {On the Global Convergence of Gradient Descent for Over-parameterized
  Models using Optimal Transport}.
\newblock \emph{Advances in Neural Information Processing Systems}, 31, 2018.

\bibitem[Choi et~al.(2024)Choi, Choi, and Kang]{choi2024scalable}
Jaemoo Choi, Jaewoong Choi, and Myungjoo Kang.
\newblock {Scalable Wasserstein Gradient Flow for Generative Modeling through
  Unbalanced Optimal Transport}.
\newblock In \emph{Forty-first International Conference on Machine Learning},
  2024.

\bibitem[Cuturi(2013)]{cuturi2013sinkhorndistanceslightspeedcomputation}
Marco Cuturi.
\newblock {Sinkhorn distances: Lightspeed computation of optimal transport}.
\newblock \emph{Advances in Neural Information Processing Systems}, 26, 2013.

\bibitem[Dagr{\'e}ou et~al.(2022)Dagr{\'e}ou, Ablin, Vaiter, and
  Moreau]{dagreou2022framework}
Mathieu Dagr{\'e}ou, Pierre Ablin, Samuel Vaiter, and Thomas Moreau.
\newblock A framework for bilevel optimization that enables stochastic and
  global variance reduction algorithms.
\newblock \emph{Advances in Neural Information Processing Systems},
  35:\penalty0 26698--26710, 2022.

\bibitem[Dinh et~al.(2014)Dinh, Krueger, and Bengio]{nice}
Laurent Dinh, David Krueger, and Yoshua Bengio.
\newblock Nice: Non-linear independent components estimation.
\newblock \emph{arXiv preprint arXiv:1410.8516}, 2014.

\bibitem[Dragomir(2003)]{dragomir2003some}
Sever~Silvestru Dragomir.
\newblock {Some Gronwall type inequalities and applications}.
\newblock \emph{Science Direct Working Paper}, \penalty0 (S1574-0358):\penalty0
  04, 2003.

\bibitem[Du et~al.(2023)Du, Li, Pang, Yan, and Lin]{chao2023nonparametric}
Chao Du, Tianbo Li, Tianyu Pang, Shuicheng Yan, and Min Lin.
\newblock {Nonparametric Generative Modeling with Conditional
  Sliced-{W}asserstein Flows}.
\newblock In Andreas Krause, Emma Brunskill, Kyunghyun Cho, Barbara Engelhardt,
  Sivan Sabato, and Jonathan Scarlett, editors, \emph{Proceedings of the 40th
  International Conference on Machine Learning}, volume 202 of
  \emph{Proceedings of Machine Learning Research}, pages 8565--8584. PMLR,
  23--29 Jul 2023.

\bibitem[Dvurechensky et~al.(2018)Dvurechensky, Gasnikov, and
  Kroshnin]{sinkhornconvergence}
Pavel Dvurechensky, Alexander Gasnikov, and Alexey Kroshnin.
\newblock {Computational Optimal Transport: Complexity by Accelerated Gradient
  Descent Is Better Than by Sinkhorn’s Algorithm}.
\newblock In Jennifer Dy and Andreas Krause, editors, \emph{Proceedings of the
  35th International Conference on Machine Learning}, volume~80 of
  \emph{Proceedings of Machine Learning Research}, pages 1367--1376. PMLR,
  10--15 Jul 2018.

\bibitem[El~Ghaoui et~al.(2021)El~Ghaoui, Gu, Travacca, Askari, and
  Tsai]{ghaoui2020implicitdeeplearning}
Laurent El~Ghaoui, Fangda Gu, Bertrand Travacca, Armin Askari, and Alicia Tsai.
\newblock {Implicit Deep Learning}.
\newblock \emph{SIAM Journal on Mathematics of Data Science}, 3\penalty0
  (3):\penalty0 930--958, 2021.

\bibitem[Fan et~al.(2022)Fan, Zhang, Taghvaei, and Chen]{fan2022variational}
Jiaojiao Fan, Qinsheng Zhang, Amirhossein Taghvaei, and Yongxin Chen.
\newblock Variational {W}asserstein gradient flow.
\newblock In Kamalika Chaudhuri, Stefanie Jegelka, Le~Song, Csaba Szepesvari,
  Gang Niu, and Sivan Sabato, editors, \emph{Proceedings of the 39th
  International Conference on Machine Learning}, volume 162 of
  \emph{Proceedings of Machine Learning Research}, pages 6185--6215. PMLR,
  17--23 Jul 2022.

\bibitem[Fei et~al.(2022)Fei, Yang, Chen, Li, Li, Ma, Hu, and
  Ma]{pointcloudcompletion}
Ben Fei, Weidong Yang, Wen-Ming Chen, Zhijun Li, Yikang Li, Tao Ma, Xing Hu,
  and Lipeng Ma.
\newblock {Comprehensive Review of Deep Learning-Based 3D Point Cloud
  Completion Processing and Analysis}.
\newblock \emph{IEEE Transactions on Intelligent Transportation Systems},
  23\penalty0 (12):\penalty0 22862–22883, December 2022.
\newblock ISSN 1558-0016.
\newblock \doi{10.1109/tits.2022.3195555}.

\bibitem[Furuya et~al.(2024)Furuya, de~Hoop, and
  Peyr{\'e}]{furuya2024transformers}
Takashi Furuya, Maarten~V de~Hoop, and Gabriel Peyr{\'e}.
\newblock {Transformers are Universal In-context Learners}.
\newblock \emph{arXiv preprint arXiv:2408.01367}, 2024.

\bibitem[Gabor et~al.(2024)Gabor, Piotrowski, and
  Cavalcante]{gabor2024positiveconcavedeepequilibrium}
Mateusz Gabor, Tomasz Piotrowski, and Renato~LG Cavalcante.
\newblock Positive concave deep equilibrium models.
\newblock \emph{arXiv preprint arXiv:2402.04029}, 2024.

\bibitem[Geng and Kolter(2023)]{torchdeq}
Zhengyang Geng and J~Zico Kolter.
\newblock Torchdeq: A library for deep equilibrium models.
\newblock \emph{arXiv preprint arXiv:2310.18605}, 2023.

\bibitem[Geng et~al.(2021)Geng, Zhang, Bai, Wang, and Lin]{phantomgradient}
Zhengyang Geng, Xin-Yu Zhang, Shaojie Bai, Yisen Wang, and Zhouchen Lin.
\newblock {On Training Implicit Models}.
\newblock In A.~Beygelzimer, Y.~Dauphin, P.~Liang, and J.~Wortman Vaughan,
  editors, \emph{Advances in Neural Information Processing Systems}, 2021.

\bibitem[Geuter et~al.(2025)Geuter, Kornhardt, Tomasson, and
  Laschos]{geuter2025universalneuraloptimaltransport}
Jonathan Geuter, Gregor Kornhardt, Ingimar Tomasson, and Vaios Laschos.
\newblock {Universal Neural Optimal Transport}.
\newblock \emph{arXiv preprint arXiv:2212.00133}, 2025.

\bibitem[Glaser et~al.(2021)Glaser, Arbel, and Gretton]{glaser2021kale}
Pierre Glaser, Michael Arbel, and Arthur Gretton.
\newblock {KALE flow: A relaxed KL gradient flow for probabilities with
  disjoint support}.
\newblock \emph{Advances in Neural Information Processing Systems},
  34:\penalty0 8018--8031, 2021.

\bibitem[Gretton et~al.(2012)Gretton, Borgwardt, Rasch, Sch{\"o}lkopf, and
  Smola]{gretton2012kernel}
Arthur Gretton, Karsten~M Borgwardt, Malte~J Rasch, Bernhard Sch{\"o}lkopf, and
  Alexander Smola.
\newblock A kernel two-sample test.
\newblock \emph{The Journal of Machine Learning Research}, 13\penalty0
  (1):\penalty0 723--773, 2012.

\bibitem[Groueix et~al.(2018)Groueix, Fisher, Kim, Russell, and
  Aubry]{groueix2018atlasnetpapiermacheapproachlearning}
Thibault Groueix, Matthew Fisher, Vladimir~G Kim, Bryan~C Russell, and Mathieu
  Aubry.
\newblock A papier-m{\^a}ch{\'e} approach to learning 3d surface generation.
\newblock In \emph{Proceedings of the IEEE conference on computer vision and
  pattern recognition}, pages 216--224, 2018.

\bibitem[Hagemann et~al.(2024)Hagemann, Hertrich, Altekr{\"u}ger, Beinert,
  Chemseddine, and Steidl]{hagemann2024posteriorsamplingbasedgradient}
Paul Hagemann, Johannes Hertrich, Fabian Altekr{\"u}ger, Robert Beinert, Jannis
  Chemseddine, and Gabriele Steidl.
\newblock {Posterior Sampling Based on Gradient Flows of the MMD with Negative
  Distance Kernel}.
\newblock In \emph{The Twelfth International Conference on Learning
  Representations}, 2024.

\bibitem[Hertrich et~al.(2024)Hertrich, Wald, Altekr{\"u}ger, and
  Hagemann]{rieszkernels}
Johannes Hertrich, Christian Wald, Fabian Altekr{\"u}ger, and Paul Hagemann.
\newblock {Generative Sliced MMD Flows with Riesz Kernels}.
\newblock In \emph{The Twelfth International Conference on Learning
  Representations}, 2024.

\bibitem[Huix et~al.(2024)Huix, Korba, Durmus, and
  Moulines]{huix2024theoretical}
Tom Huix, Anna Korba, Alain Durmus, and Eric Moulines.
\newblock Theoretical {G}uarantees for {V}ariational {I}nference with
  {F}ixed-{V}ariance {M}ixture of {G}aussians.
\newblock \emph{International Conference on Machine Learning}, 2024.

\bibitem[Ji et~al.(2021)Ji, Yang, and Liang]{bileveloptimization}
Kaiyi Ji, Junjie Yang, and Yingbin Liang.
\newblock {Bilevel Optimization: Convergence Analysis and Enhanced Design}.
\newblock In Marina Meila and Tong Zhang, editors, \emph{Proceedings of the
  38th International Conference on Machine Learning}, volume 139 of
  \emph{Proceedings of Machine Learning Research}, pages 4882--4892. PMLR,
  18--24 Jul 2021.

\bibitem[Kobyzev et~al.(2021)Kobyzev, Prince, and Brubaker]{normalizingflows}
Ivan Kobyzev, Simon~J.D. Prince, and Marcus~A. Brubaker.
\newblock {Normalizing Flows: An Introduction and Review of Current Methods}.
\newblock \emph{IEEE Transactions on Pattern Analysis and Machine
  Intelligence}, 43\penalty0 (11):\penalty0 3964–3979, November 2021.
\newblock ISSN 1939-3539.
\newblock \doi{10.1109/tpami.2020.2992934}.

\bibitem[Lambert et~al.(2022)Lambert, Chewi, Bach, Bonnabel, and
  Rigollet]{lambert2022variational}
Marc Lambert, Sinho Chewi, Francis Bach, Silv{\`e}re Bonnabel, and Philippe
  Rigollet.
\newblock Variational inference via {W}asserstein gradient flows.
\newblock \emph{Advances in Neural Information Processing Systems},
  35:\penalty0 14434--14447, 2022.

\bibitem[Lanzetti et~al.(2025)Lanzetti, Bolognani, and
  D{\"o}rfler]{lanzetti2022first}
Nicolas Lanzetti, Saverio Bolognani, and Florian D{\"o}rfler.
\newblock {First-Order Conditions for Optimization in the Wasserstein Space}.
\newblock \emph{SIAM Journal on Mathematics of Data Science}, 7\penalty0
  (1):\penalty0 274--300, 2025.

\bibitem[Lee(2003)]{lee2003smoothmanifolds}
John~M. Lee.
\newblock \emph{{Introduction to Smooth Manifolds}}.
\newblock Graduate Texts in Mathematics. Springer New York, NY, 1 edition,
  March 2003.
\newblock ISBN 978-0-387-21752-9.
\newblock \doi{10.1007/978-0-387-21752-9}.
\newblock Published: 09 March 2013 (eBook).

\bibitem[Lee(2019)]{lee2019riemannianmanifolds}
John~M. Lee.
\newblock \emph{{Introduction to Riemannian Manifolds}}.
\newblock Graduate Texts in Mathematics. Springer Cham, 2 edition, January
  2019.
\newblock ISBN 978-3-319-91754-2.
\newblock \doi{10.1007/978-3-319-91755-9}.
\newblock Published: 14 January 2019 (Hardcover), Published: 05 August 2021
  (Softcover).

\bibitem[Lee et~al.(2019)Lee, Lee, Kim, Kosiorek, Choi, and
  Teh]{settransformer}
Juho Lee, Yoonho Lee, Jungtaek Kim, Adam Kosiorek, Seungjin Choi, and Yee~Whye
  Teh.
\newblock Set transformer: A framework for attention-based
  permutation-invariant neural networks.
\newblock In \emph{International conference on machine learning}, pages
  3744--3753. PMLR, 2019.

\bibitem[Lessel and
  Schick(2020)]{lessel2020differentiablemapswassersteinspaces}
Bernadette Lessel and Thomas Schick.
\newblock {Differentiable maps between Wasserstein spaces}.
\newblock \emph{arXiv preprint arXiv:2010.02131}, 2020.

\bibitem[Ling et~al.(2024)Ling, Li, Feng, Zhang, Zhou, Qiu, and
  Liao]{ling2024deepequilibriummodelsequivalent}
Zenan Ling, Longbo Li, Zhanbo Feng, Yixuan Zhang, Feng Zhou, Robert~C Qiu, and
  Zhenyu Liao.
\newblock {Deep Equilibrium Models are Almost Equivalent to Not-so-deep
  Explicit Models for High-dimensional Gaussian Mixtures}.
\newblock \emph{arXiv preprint arXiv:2402.02697}, 2024.

\bibitem[Liutkus et~al.(2019)Liutkus, Simsekli, Majewski, Durmus, and
  St{\"o}ter]{liutkus2019sliced}
Antoine Liutkus, Umut Simsekli, Szymon Majewski, Alain Durmus, and
  Fabian-Robert St{\"o}ter.
\newblock {Sliced-Wasserstein flows: Nonparametric generative modeling via
  optimal transport and diffusions}.
\newblock In \emph{International Conference on Machine Learning}, pages
  4104--4113. PMLR, 2019.

\bibitem[Mao et~al.(2023)Mao, Shi, Wang, and Li]{objectdetectionautonomous}
Jiageng Mao, Shaoshuai Shi, Xiaogang Wang, and Hongsheng Li.
\newblock {3D object detection for autonomous driving: A comprehensive survey}.
\newblock \emph{International Journal of Computer Vision}, 131\penalty0
  (8):\penalty0 1909--1963, 2023.

\bibitem[Marion et~al.(2024)Marion, Korba, Bartlett, Blondel, De~Bortoli,
  Doucet, Llinares-L{\'o}pez, Paquette, and Berthet]{implicit}
Pierre Marion, Anna Korba, Peter Bartlett, Mathieu Blondel, Valentin
  De~Bortoli, Arnaud Doucet, Felipe Llinares-L{\'o}pez, Courtney Paquette, and
  Quentin Berthet.
\newblock {Implicit Diffusion: Efficient Optimization through Stochastic
  Sampling}.
\newblock \emph{arXiv preprint arXiv:2402.05468}, 2024.

\bibitem[Matsubara(2024)]{matsubara2024wasserstein}
Takuo Matsubara.
\newblock Wasserstein gradient boosting: A framework for distribution-valued
  supervised learning.
\newblock In \emph{Advances in Neural Information Processing Systems}, 2024.

\bibitem[McInnes et~al.(2018)McInnes, Healy, and
  Melville]{mcinnes2020umapuniformmanifoldapproximation}
Leland McInnes, John Healy, and James Melville.
\newblock {UMAP: Uniform Manifold Approximation and Projection for Dimension
  Reduction}.
\newblock \emph{arXiv preprint arXiv:1802.03426}, 2018.

\bibitem[Neumayer et~al.(2024)Neumayer, Stein, and
  Steidl]{neumayer2024wasserstein}
Sebastian Neumayer, Viktor Stein, and Gabriele Steidl.
\newblock {Wasserstein Gradient Flows for Moreau Envelopes of f-Divergences in
  Reproducing Kernel Hilbert Spaces}.
\newblock \emph{arXiv preprint arXiv:2402.04613}, 2024.

\bibitem[Otto(2001)]{otto2001geometry}
Felix Otto.
\newblock The geometry of dissipative evolution equations: The porous medium
  equation.
\newblock \emph{Communications in Partial Differential Equations}, 26\penalty0
  (1-2):\penalty0 101--174, 2001.

\bibitem[Peyr{\'e} et~al.(2019)Peyr{\'e}, Cuturi,
  et~al.]{peyre2019computational}
Gabriel Peyr{\'e}, Marco Cuturi, et~al.
\newblock Computational optimal transport: With applications to data science.
\newblock \emph{Foundations and Trends{\textregistered} in Machine Learning},
  11\penalty0 (5-6):\penalty0 355--607, 2019.

\bibitem[Qi et~al.(2017{\natexlab{a}})Qi, Su, Mo, and Guibas]{qi2017pointnet}
Charles~R Qi, Hao Su, Kaichun Mo, and Leonidas~J Guibas.
\newblock Pointnet: Deep learning on point sets for 3d classification and
  segmentation.
\newblock In \emph{Proceedings of the IEEE conference on computer vision and
  pattern recognition}, pages 652--660, 2017{\natexlab{a}}.

\bibitem[Qi et~al.(2017{\natexlab{b}})Qi, Yi, Su, and Guibas]{pointnet++}
Charles~Ruizhongtai Qi, Li~Yi, Hao Su, and Leonidas~J Guibas.
\newblock {Pointnet++: Deep hierarchical feature learning on point sets in a
  metric space}.
\newblock \emph{Advances in neural information processing systems}, 30,
  2017{\natexlab{b}}.

\bibitem[Ramzi et~al.(2023)Ramzi, Ablin, Peyr{\'e}, and
  Moreau]{ramzi2023testliketrainimplicit}
Zaccharie Ramzi, Pierre Ablin, Gabriel Peyr{\'e}, and Thomas Moreau.
\newblock {Test like you Train in Implicit Deep Learning}.
\newblock \emph{arXiv preprint arXiv:2305.15042}, 2023.

\bibitem[Santambrogio(2015)]{santambrogio2015optimal}
Filippo Santambrogio.
\newblock Optimal transport for applied mathematicians.
\newblock \emph{Birk{\"a}user, NY}, 55\penalty0 (58-63):\penalty0 94, 2015.

\bibitem[Santambrogio(2017)]{santambrogio2017euclidean}
Filippo Santambrogio.
\newblock $\{$Euclidean, metric, and Wasserstein$\}$ gradient flows: an
  overview.
\newblock \emph{Bulletin of Mathematical Sciences}, 7:\penalty0 87--154, 2017.

\bibitem[Sejdinovic et~al.(2013)Sejdinovic, Sriperumbudur, Gretton, and
  Fukumizu]{sejdinovic2013equivalence}
Dino Sejdinovic, Bharath Sriperumbudur, Arthur Gretton, and Kenji Fukumizu.
\newblock {Equivalence of distance-based and RKHS-based statistics in
  hypothesis testing}.
\newblock \emph{The annals of statistics}, pages 2263--2291, 2013.

\bibitem[Thornton and
  Cuturi(2023)]{thornton2023rethinkinginitializationsinkhornalgorithm}
James Thornton and Marco Cuturi.
\newblock {Rethinking initialization of the Sinkhorn algorithm}.
\newblock In \emph{International Conference on Artificial Intelligence and
  Statistics}, pages 8682--8698. PMLR, 2023.

\bibitem[van~der Maaten and Hinton(2008)]{tsne}
Laurens van~der Maaten and Geoffrey Hinton.
\newblock {Visualizing Data using t-SNE}.
\newblock \emph{Journal of Machine Learning Research}, 9\penalty0
  (86):\penalty0 2579--2605, 2008.

\bibitem[Vaswani et~al.(2017)Vaswani, Shazeer, Parmar, Uszkoreit, Jones, Gomez,
  Kaiser, and Polosukhin]{transformer}
Ashish Vaswani, Noam Shazeer, Niki Parmar, Jakob Uszkoreit, Llion Jones,
  Aidan~N Gomez, \L~ukasz Kaiser, and Illia Polosukhin.
\newblock {Attention is All you Need}.
\newblock In I.~Guyon, U.~Von Luxburg, S.~Bengio, H.~Wallach, R.~Fergus,
  S.~Vishwanathan, and R.~Garnett, editors, \emph{Advances in Neural
  Information Processing Systems}, volume~30. Curran Associates, Inc., 2017.

\bibitem[Vempala and Wibisono(2019)]{vempala2019rapid}
Santosh Vempala and Andre Wibisono.
\newblock Rapid convergence of the unadjusted {L}angevin algorithm:
  {I}soperimetry suffices.
\newblock \emph{Advances in neural information processing systems}, 32, 2019.

\bibitem[Villani(2008)]{villani2021topics}
C{\'e}dric Villani.
\newblock \emph{Topics in optimal transportation}, volume~58.
\newblock American Mathematical Soc., 2008.

\bibitem[Wen et~al.(2021)Wen, Xiang, Han, Cao, Wan, Zheng, and
  Liu]{wen2021pmpnetpointcloudcompletion}
Xin Wen, Peng Xiang, Zhizhong Han, Yan-Pei Cao, Pengfei Wan, Wen Zheng, and
  Yu-Shen Liu.
\newblock Pmp-net: Point cloud completion by learning multi-step point moving
  paths.
\newblock In \emph{Proceedings of the IEEE/CVF conference on computer vision
  and pattern recognition}, pages 7443--7452, 2021.

\bibitem[Wen et~al.(2022)Wen, Xiang, Han, Cao, Wan, Zheng, and
  Liu]{wen2022pmpnetpointcloudcompletion}
Xin Wen, Peng Xiang, Zhizhong Han, Yan-Pei Cao, Pengfei Wan, Wen Zheng, and
  Yu-Shen Liu.
\newblock {PMP-Net++: Point Cloud Completion by Transformer-Enhanced Multi-step
  Point Moving Paths}.
\newblock \emph{IEEE Transactions on Pattern Analysis and Machine
  Intelligence}, 45\penalty0 (1):\penalty0 852--867, 2022.

\bibitem[Wibisono(2018)]{wibisono2018sampling}
Andre Wibisono.
\newblock Sampling as optimization in the space of measures: The {L}angevin
  dynamics as a composite optimization problem.
\newblock In \emph{Conference on Learning Theory}, pages 2093--3027. PMLR,
  2018.

\bibitem[Winston and
  Kolter(2020)]{winston2021monotoneoperatorequilibriumnetworks}
Ezra Winston and J~Zico Kolter.
\newblock Monotone operator equilibrium networks.
\newblock \emph{Advances in neural information processing systems},
  33:\penalty0 10718--10728, 2020.

\bibitem[Ye et~al.(2024)Ye, Peyr{\'e}, Cremers, and
  Ablin]{ye2024enhancinghypergradientsestimationstudy}
Zhenzhang Ye, Gabriel Peyr{\'e}, Daniel Cremers, and Pierre Ablin.
\newblock {Enhancing Hypergradients Estimation: A Study of Preconditioning and
  Reparameterization}.
\newblock In \emph{International Conference on Artificial Intelligence and
  Statistics}, pages 955--963. PMLR, 2024.

\bibitem[Yuan et~al.(2018)Yuan, Khot, Held, Mertz, and
  Hebert]{yuan2019pcnpointcompletionnetwork}
Wentao Yuan, Tejas Khot, David Held, Christoph Mertz, and Martial Hebert.
\newblock {PCN: Point completion network}.
\newblock In \emph{2018 international conference on 3D vision (3DV)}, pages
  728--737. IEEE, 2018.

\bibitem[Zaheer et~al.(2017)Zaheer, Kottur, Ravanbakhsh, Poczos, Salakhutdinov,
  and Smola]{deepsets}
Manzil Zaheer, Satwik Kottur, Siamak Ravanbakhsh, Barnabas Poczos, Russ~R
  Salakhutdinov, and Alexander~J Smola.
\newblock Deep sets.
\newblock \emph{Advances in neural information processing systems}, 30, 2017.

\bibitem[Zaitsev and Polyanin(2002)]{zaitsev2002handbook}
Valentin~F Zaitsev and Andrei~D Polyanin.
\newblock \emph{Handbook of exact solutions for ordinary differential
  equations}.
\newblock Chapman and Hall/CRC, 2002.

\bibitem[Zhao et~al.(2021)Zhao, Jiang, Jia, Torr, and Koltun]{pointtransformer}
Hengshuang Zhao, Li~Jiang, Jiaya Jia, Philip~HS Torr, and Vladlen Koltun.
\newblock Point transformer.
\newblock In \emph{Proceedings of the IEEE/CVF international conference on
  computer vision}, pages 16259--16268, 2021.

\bibitem[Zhu et~al.(2024{\natexlab{a}})Zhu, Wang, Huang, Ye, Ouyang, and
  He]{robotics}
Haoyi Zhu, Yating Wang, Di~Huang, Weicai Ye, Wanli Ouyang, and Tong He.
\newblock {Point Cloud Matters: Rethinking the Impact of Different Observation
  Spaces on Robot Learning}.
\newblock \emph{arXiv preprint arXiv:2402.02500}, 2024{\natexlab{a}}.

\bibitem[Zhu et~al.(2024{\natexlab{b}})Zhu, Wang, Zhang, Zhao, and
  Qian]{zhu2024neural}
Huminhao Zhu, Fangyikang Wang, Chao Zhang, Hanbin Zhao, and Hui Qian.
\newblock {Neural Sinkhorn Gradient Flow}.
\newblock \emph{arXiv preprint arXiv:2401.14069}, 2024{\natexlab{b}}.

\bibitem[Ziesche and Rozo(2023)]{ziesche2024wasserstein}
Hanna Ziesche and Leonel Rozo.
\newblock {Wasserstein gradient flows for optimizing Gaussian mixture
  policies}.
\newblock \emph{Advances in Neural Information Processing Systems}, 36, 2023.

\end{thebibliography}
